\documentclass{article}

\usepackage{microtype}
\usepackage{graphicx}
\usepackage{subfigure}
\usepackage{booktabs} 

\usepackage{hyperref}

\usepackage[accepted]{icml2023}
\usepackage{alias}

\usepackage{amsmath}
\usepackage{amssymb}
\usepackage{mathtools}
\usepackage{amsthm}

\usepackage[capitalize,noabbrev]{cleveref}

\theoremstyle{plain}
\newtheorem{theorem}{Theorem}[section]
\newtheorem{proposition}[theorem]{Proposition}
\newtheorem{lemma}[theorem]{Lemma}
\newtheorem{corollary}[theorem]{Corollary}
\theoremstyle{definition}

\theoremstyle{remark}

\icmltitlerunning{Exponential Smoothing for Off-Policy Learning}

\begin{document}

\twocolumn[
\icmltitle{Exponential Smoothing for Off-Policy Learning}



\icmlsetsymbol{equal}{*}

\begin{icmlauthorlist}
\icmlauthor{Imad Aouali}{yyy,comp}
\icmlauthor{Victor-Emmanuel Brunel}{comp}
\icmlauthor{David Rohde}{yyy}
\icmlauthor{Anna Korba}{comp}
\end{icmlauthorlist}

\icmlaffiliation{yyy}{CREST, ENSAE, IP Paris, France}
\icmlaffiliation{comp}{Criteo AI Lab, Paris, France}

\icmlcorrespondingauthor{Imad Aouali}{i.aouali@criteo.com}

\icmlkeywords{Machine Learning, ICML}

\vskip 0.3in
]



\printAffiliationsAndNotice{}  

\begin{abstract}
Off-policy learning (OPL) aims at finding improved policies from logged bandit data, often by minimizing the inverse propensity scoring (IPS) estimator of the risk. In this work, we investigate a smooth regularization for IPS, for which we derive a two-sided PAC-Bayes generalization bound. The bound is tractable, scalable, interpretable and provides learning certificates. In particular, it is also valid for standard IPS without making the assumption that the importance weights are bounded. We demonstrate the relevance of our approach and its favorable performance through a set of learning tasks. Since our bound holds for standard IPS, we are able to provide insight into when regularizing IPS is useful. Namely, we identify cases where regularization might not be needed. This goes against the belief that, in practice, clipped IPS often enjoys favorable performance than standard IPS in OPL.
\end{abstract}

\section{Introduction}
\label{sec:introduction}

An off-policy contextual bandit \citep{dudik2011doubly} is a ubiquitous framework to optimize decision-making using offline data. In practice, logged data reflecting the preferences of the agent in an online setting is available \citep{bottou2013counterfactual}. In each round, the agent observes a \emph{context}, takes an \emph{action}, and receives a \emph{reward} that depends on the observed context and the taken action. Off-policy evaluation (OPE) \citep{dudik2011doubly} aims at evaluating a policy offline by designing an estimator of its expected reward using logged data. The estimator is often based on the importance sampling trick and it is generally referred to as inverse propensity scoring (IPS) \citep{horvitz1952generalization}. Off-policy learning (OPL) leverages the latter estimator to learn an improved policy \citep{swaminathan2015batch}. 

The literature on OPL has focused so far on using learning principles derived from generalization bounds. First, \citet{swaminathan2015batch} used sample variance penalization (SVP) that favors policies with high estimated reward and low empirical variance. Recently, \citet{london2019bayesian} derived a novel scalable learning principle that favors policies with high estimated reward and whose parameter is not far from that of the logging policy in terms of $L_2$ distance. While derived from generalization bounds, these learning principles do not give any guarantees on the expected performance of the learned policy. Also, they require additional care to tune their hyper-parameters. Thus, motivated by the results in \citet{sakhi2022pac}, we derive tractable generalization bounds that we optimize directly. 

The paper is organized as follows. 
In \cref{sec:setting}, we introduce the necessary background. In \cref{sec:ope}, we explain the shortcomings of the widely used \emph{hard clipping} of IPS and present a smoother correction, 
called exponential smoothing. In \cref{sec:opl}, we focus on OPL and leverage PAC-Bayes theory to derive a \emph{two-sided} generalization bound for our estimator. In contrast with prior works 
\citep{swaminathan2015batch,london2019bayesian,sakhi2022pac}, our bound is also valid for standard IPS without clipping, and this is without assuming that the importance weights are bounded. We also discuss our results in detail in \cref{sec:discussion}. In particular, we give insights into the sample complexity of our learning procedure, an important question not addressed in prior OPL works. Finally, we show in \cref{sec:experiments} that our approach enjoys favorable performance. A detailed comparative review of the literature is provided in \cref{sec:related_work}. The proofs are deferred to \cref{proofs:ope,proofs:opl}. Refer to \cref{app:all_experiments} to reproduce our experiments.

\section{Background}
\label{sec:setting}

Consider an agent interacting with a \emph{contextual bandit} environment over $n$ rounds. In round $t \in [n]$, the agent observes a \emph{context} $x_t \sim \nu$, where $\nu$ is a distribution whose support $\mathcal{X}$ is a compact subset of $\mathbb{R}^d$. Then the agent takes an \emph{action} $a_t \in \cA = [K]$. Finally, the agent receives a stochastic cost $c_t \in [-1, 0]$ that depends on both $x_t$ and $a_t$. That is $c_t \sim p(\cdot | x_t, a_t)$ where $p(\cdot | x, a)$ is the \emph{cost distribution} of action $a$ in context $x$. We let $c(x, a) = \E{c \sim p(\cdot | x, a)}{c}$ be the \emph{cost function} that outputs the expected cost of action $a$ in context $x$. Here we use a negative cost since it is seen as the negative value of the reward, that is for any $(x, a) \in \cX \times \cA\,,$ $c(x, a) = -r(x, a)$ where $r : \cX \times \cA \rightarrow [0, 1]$ is the \emph{reward function} that outputs the expected reward of $a$ in context $x$.

The agent is represented by a stochastic policy $\pi$. Given a context $x \in \cX\,,$ $\pi(\cdot | x)$ is a probability distribution over $\cA$. Our goal is to find a policy $\pi \in \Pi$ among a set of policies $\Pi$ that minimizes the risk defined as
\begin{align}\label{eq:policy_value}
    R(\pi) &= \E{(x, a, c) \sim \mu_\pi}{c} = \E{x \sim \nu, a \sim \pi(\cdot | x)}{c(x, a)}\,,
\end{align}
where $\mu_\pi$ is the joint distribution of $(x, a, c)$; $\mu_\pi(x, a, c) = \nu(x)\pi(a | x)p(c | x, a)$. We assume access to logged data $\mathcal{D}_n = ( x_i, a_i, c_i)_{i \in [n]},$ where $(x_i, a_i, c_i) \sim \mu_{\pi_0}$ are i.i.d. and $\pi_0$ is a \emph{known logging policy}. Given a policy $\pi \in \Pi$, OPE consists in building an estimator for its risk $R(\pi)$ using $\mathcal{D}_n$ such as $\hat{R}_n(\pi) \approx R(\pi)$. After that, OPL is used to find a policy $\hat{\pi}_n \in \Pi$ such that $R(\hat{\pi}_n) \approx \min_{\pi \in \Pi} R(\pi)$. 

In this work, we focus on inverse propensity scoring (IPS) \citep{horvitz1952generalization,dudik2012sample}. Given a policy $\pi \in \Pi$, IPS estimates the risk $R(\pi)$ by re-weighting the samples using the ratio between $\pi$ and $\pi_0$ such as
\begin{align}\label{eq:ips_policy_value}
    \hat{R}_n^{\textsc{ips}}(\pi) &= \frac{1}{n} \sum_{i=1}^n c_i w_{\pi}(a_i | x_i)\,,
\end{align}
where for any $(x, a) \in \cX \times \cA\,, w_{\pi}(a | x) =  \pi(a | x)/\pi_0(a | x)$ are the \textit{importance weights}. The variance of $\hat{R}_n^{\textsc{ips}}(\pi)$ scales linearly with the importance weights \citep{swaminathan2017off} which can be large. Thus other OPE methods that do not rely on the importance weights or partially use them were proposed and they can be categorized into two families, direct method (DM) \citep{jeunen2021pessimistic} and doubly robust (DR) \citep{dudik2011doubly}. The reader may refer to \cref{sec:related_work_OPE} for more details about these methods.

Let $\hat{R}_n$ be an estimator of the risk $R$. For instance, $\hat{R}_n$ can be $\hat{R}^{\textsc{ips}}_n$ in \eqref{eq:ips_policy_value}. The goal in OPL is to minimize the risk $R$. But since we cannot access it, we only search for $\hat{\pi}_n = \argmin_{\pi \in \Pi} \hat{R}_n(\pi) + \operatorname{pen}(\pi)$ hoping that $R(\hat{\pi}_n) \approx \min_{\pi \in \Pi}R(\pi)$. Here $\operatorname{pen}(\cdot)$ is a penalization term obtained using generalization bounds of the following form. Let $\delta \in (0, 1)$, then we have with probability at least $1-\delta$ that
\begin{align}\label{eq:opl_one_sided}
    &R(\pi) \leq \hat{R}_n(\pi) + g(\delta, \Pi, \pi, \pi_0, n)\,, & \forall \pi \in \Pi\,,
\end{align}
for some function $g$. Improving upon $\pi_0$, that is when $R(\pi) - R(\pi_0)< 0$, is guaranteed with high probability when $\hat{R}_n(\pi) + g(\delta, \Pi, \pi, \pi_0, n) - R(\pi_0) < 0$. Thus we minimize $\hat{R}_n(\pi) + g(\delta, \Pi, \pi, \pi_0, n) - R(\pi_0)$ \emph{in the hope} that the minimum is smaller than $0$. Since $R(\pi_0)$ is fixed, the final objective reads 
\begin{align}\label{eq:objective}
 \hat{\pi}_n =  \argmin_{\pi \in \Pi} \hat{R}_n(\pi) + g(\delta, \Pi, \pi, \pi_0, n)\,.
\end{align}
This motivated the concept of \emph{counterfactual risk minimization (CRM)} in \citet{swaminathan2015batch,london2019bayesian,sakhi2022pac}. However, all these works only derived one-sided inequalities similar to \eqref{eq:opl_one_sided}. In contrast, we derive \emph{two-sided} inequalities of the form 
\begin{align}\label{eq:opl_two_sided}
    &|R(\pi) - \hat{R}_n(\pi) |\leq g(\delta, \Pi, \pi, \pi_0, n)\,, & \forall \pi \in \Pi\,.
\end{align}
This is because \eqref{eq:opl_two_sided} can attest to the quality of the estimator $\hat{R}_n$. A one-sided one fails at this. To see why, note that we have with probability $1$ that $R(\pi) \leq \hat{R}^{\textsc{poor}}_n(\pi)$ with $g(\delta, \Pi, \pi, \pi_0, n)=0$, considering a poor estimator of the risk, $\hat{R}^{\textsc{poor}}_n(\pi)=0$ for any $\pi \in \Pi$. This holds since by definition $R(\pi) \in [-1, 0]$ while $\hat{R}^{\textsc{poor}}_n(\pi)=0$ for any $\pi \in \Pi$. While this one-sided inequality holds for  $\hat{R}^{\textsc{poor}}_n$, this estimator is not informative at all about $R$, so minimizing it is not relevant. This is why we need to control the quality of the upper bound on $R$, and this is achieved by two-sided inequalities similar to \eqref{eq:opl_two_sided}. Also, \eqref{eq:opl_two_sided} leads to oracle inequalities of the form $R(\hat{\pi}_n) \leq R(\pi_*) + 2g(\delta, \Pi, \pi_*, \pi_0, n)$, where $\hat{\pi}_n$ is the learned policy in \eqref{eq:objective} and $\pi_* = \argmin_{\pi \in \Pi} R(\pi)$ is the optimal policy. This allows us to quantify the number of samples $n$ needed so that the risk of the learned policy $R(\hat{\pi}_n)$ is close to the optimal one $R(\pi_*)$.

Moreover, in many prior works, the objective in \eqref{eq:objective} is not optimized directly. Instead, the function $g$ is used to motivate a heuristic-based learning principle. Here we review these principles briefly. But the reader may refer to \cref{sec:opl_learning_principles} for more detail. First, \citet{swaminathan2015batch} minimized the estimated risk while penalizing its empirical variance. This was inspired by a function $g$ that contains a variance term; discarding more complicated terms like the covering number of the space of policies $\Pi$. Similarly, \citet{london2019bayesian} parameterize policies by a mean parameter and propose to penalize the estimated risk by the $L_2$ distance between the mean of the logging and the learning policies; discarding all the other terms from their bound. 
In contrast, we follow the theoretically grounded approach that consists in directly optimizing the objective in \eqref{eq:objective} as it is. It may also be relevant to note that some works \citep{metelli2021subgaussian} derived \emph{evaluation} bounds and used them in OPL. In evaluation, we \emph{fix a policy $\pi \in \Pi$}, and show that
\begin{align*}
    \mathbb{P}(|R(\pi) - \hat{R}_n(\pi) |\leq f(\delta, \pi, \pi_0, n)) \geq 1-\delta\,,
\end{align*}
for some function $f$ that does not necessarily depend on the space of policies $\Pi$. In contrast, the generalization bound in \eqref{eq:opl_two_sided} holds simultaneously for any policy $\pi \in \Pi$, and it is the one that should be used in OPL. That said, in this work, we derive a \emph{two-sided} generalization bound that holds \emph{simultaneously for any policy $\pi \in \Pi$} as in \eqref{eq:opl_two_sided}.

\section{Exponential Smoothing}
\label{sec:ope} 
The estimator $\hat{R}_n^{\textsc{ips}}(\pi)$ in \eqref{eq:ips_policy_value} is unbiased when $\pi_0(a|x)=0$ implies that $\pi(a|x)=0$ for any $(x,a)\in \cX \times \cA$. But its variance can be large as it grows linearly with the importance weights $w_\pi(a | x)$. Thus they are often clipped \citep{swaminathan2015batch} such as on the following estimators
\begin{align}\label{eq:clip_ips_policy_value}
 \texttt{IPS-min} \quad  \tilde{R}_n^{\textsc m}(\pi)&= \frac{1}{n} \sum_{i=1}^n c_i \min\big(w_{\pi}(a_i | x_i), M\big)\,,\nonumber\\
    \texttt{IPS-max} \quad \hat{R}_n^\tau(\pi)&= \frac{1}{n} \sum_{i=1}^n c_i \frac{\pi(a_i | x_i)}{\max(\pi_0(a_i | x_i), \tau)}\,.
\end{align}
Here \texttt{IPS-min} clips the weights while \texttt{IPS-max} only clips $\pi_0$ in the denominator since $\pi$ is always smaller than 1. For instance, $M \in \real^+$ in $\tilde{R}_n^{\textsc m}(\pi)$ trades the bias and variance of the estimator. When $M$ is large, the bias of $\tilde{R}_n^{\textsc m}(\pi)$ is small but its variance may be large. On the other hand, the variance goes to $0$ when $M \approx 0$ since in that case $\tilde{R}_n^{\textsc m}(\pi) \approx 0$ for any $\pi \in \Pi$. Similarly, $\tau \in [0, 1]$ trades the bias and variance of $\hat{R}_n^\tau(\pi)$ and can be seen as $\tau \approx \frac{1}{M}$.

This \emph{hard} clipping has some limitations. First, $\min(\cdot, M)$ leads to non-differentiable objectives that may require additional care in optimization \citep{papini2019optimistic}. Also, $\min(\cdot, M)$ is constant on $[M, \infty)$ leading to objectives with zero gradients for any policy $\pi$ that satisfies $w_\pi(a_i | x_i) > M$ for any $i \in [n]$. More importantly, hard clipping is sensitive to the choice of the clipping threshold $M$. In practice, tuning $M$ is challenging and may cause the learned policy to match the logging policy, leading to minimal improvements. To see this, consider the following illustrative example.

For simplicity, suppose that the problem is non-contextual, in which case the reward function $r$ only depends on the actions $a \in \cA$. It follows that policies do not depend on $x \in \cX$; they are now probability distributions $\pi(\cdot)$ over $\cA$. Also, assume that $\cA = [100]$ and that the reward received after taking action $a \in [100]$ is binary. That is, $r \sim {\rm Bern}(r(a))$ where $r(a) = 0.1 - 10^{-3}(a-1)$ is the expected reward of action $a$, and for any $p \in [0, 1],$ ${\rm Bern}(p)$ is the Bernoulli distribution with parameter $p$. This means that the best action is $1$ and the worst is $100$. Finally, the logging policy $\pi_0(\cdot)$ is $\epsilon$-greedy centered at action $50$. That is $\pi_0(50) = 1-\epsilon$, and for any $a \neq 50$, $\pi_0(a) = \frac{\epsilon}{99}$, with $\epsilon=0.05$,. 

Now consider $100$ deterministic policies $\pi_a(\cdot)$ for $a \in [100]$ such that $\pi_a(\cdot)$ is the Dirac distribution centered at $a$. In \cref{fig:example}, we plot the estimated reward of the policies $\pi_a$ using either \texttt{IPS} in \eqref{eq:ips_policy_value} or \texttt{IPS-min} in \eqref{eq:clip_ips_policy_value}. We generate $n=50{\rm k}$ samples and set $M=100 = \mathcal{O}(\sqrt{n})$ as suggested by \citet{ionides2008truncated}. \emph{With this choice of $M$}, \texttt{IPS-min} underestimates the reward of all policies $\pi_a$ for $a \neq 100$ since their weights $\pi_a/\pi_0$ are either $0$ or $ 99/ \epsilon>M$. The estimated reward of \texttt{IPS-min} is maximized in $\pi_{50} \approx \pi_0$ only. Thus, if we optimize $\tilde{R}_n^{\sc M}(\cdot)$ over Dirac policies, we will converge to the logging policy despite its bad performance.

Although the other variant of hard clipping, \texttt{IPS-max} in \eqref{eq:clip_ips_policy_value}, is differentiable, it is still sensitive to $\tau$ and may induce high bias similar to \cref{fig:example}. This is due to some loss of information related to the preferences of the logging policy. Indeed, for two actions $a$ and $a^\prime$ such that $\pi_0(a \mid x_i) \ll \pi_0(a^\prime \mid x_i) < \tau$ for an observed context $x_i$, the propensity scores $\pi_0(a \mid x_i)$ and $\pi_0(a^\prime \mid x_i)$ will be clipped to the same value $\tau$. Thus the information that, for context $x_i$, action $a^\prime$ is preferred by the logging policy than action $a$ will be lost.

\begin{figure}[ht]
\begin{center}
\vskip -0.1in
\centerline{\includegraphics[width=0.7\columnwidth]{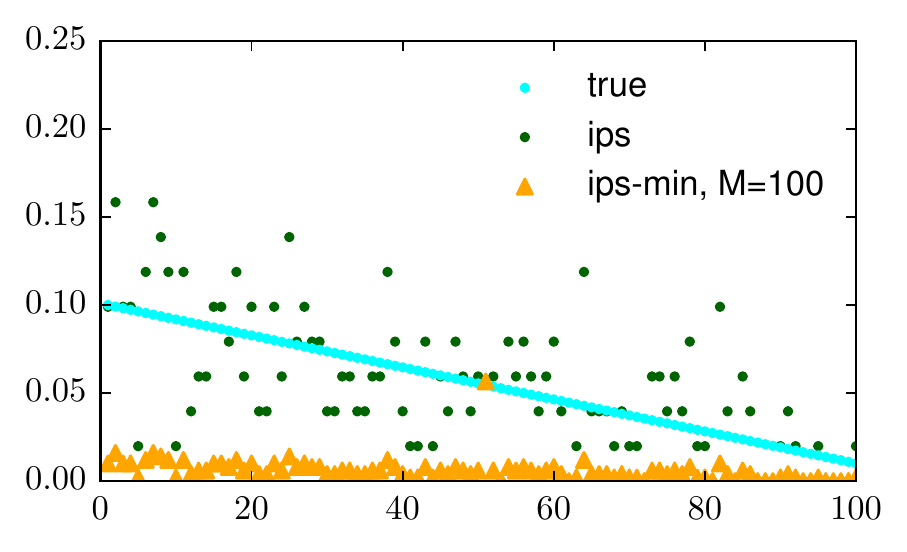}}
\vskip -0.2in
\caption{Effect of hard clipping on the estimation quality. The $x$-axis corresponds to actions $a \in [100]$. The $y$-axis is the estimated reward of each of the $100$ policies $\pi_a$ using either \texttt{IPS} or \texttt{IPS-min}. The cyan line is the true reward for each policy $\pi_a$.}
\label{fig:example}
\end{center}
\vskip -0.3in
\end{figure}

To mitigate this, we propose the following \emph{exponential smoothing} correction for IPS. Our estimators are defined as
\begin{align}\label{eq:exp_ips_policy_value}
\hspace{-0.3cm} \texttt{IPS-}\alpha:   \, \,    \hat{R}_n^\alpha(\pi)
  = \frac{1}{n} \sum_{i=1}^n c_i \hat{w}^\alpha_{\pi}(a_i | x_i)\,, \, \, \alpha \in [0, 1]\,, \nonumber\\
\hspace{-0.3cm} \texttt{IPS-}\beta: \, \,    \tilde{R}_n^\beta(\pi) 
= \frac{1}{n} \sum_{i=1}^n c_i \tilde{w}^\beta_{\pi}(a_i | x_i)\,, \, \, \beta \in [0, 1]\,,
\end{align}
where $\hat{w}^\alpha_{\pi}(a | x) = \frac{\pi(a | x)}{\pi_0(a | x)^\alpha}$ and $\tilde{w}^\beta_{\pi}(a | x) = \frac{\pi(a | x)^\beta}{\pi_0(a | x)^\beta}$. Here standard IPS is recovered for $\alpha=1$ and $\beta=1$. These estimators are differentiable in $\pi$ and do not suffer from stationary points in optimization as they are not constant in $\pi$ when $\beta \neq 0$ and $\alpha \neq 0$. Also, in contrast with \texttt{IPS-max} in \eqref{eq:clip_ips_policy_value}, $\hat{R}_n^\alpha(\pi)$ preserves the preferences of the logging policy. Precisely, for two actions $a$ and $a^\prime$ such that $\pi_0(a \mid x_i) < \pi_0(a^\prime \mid x_i)$ for an observed context $x_i$, we still have $\pi_0(a \mid x_i)^\alpha < \pi_0(a^\prime \mid x_i)^\alpha$ and the information that action $a^\prime$ is preferred by the logging policy than action $a$ is preserved.

While a similar correction to \texttt{IPS-}$\beta$ was proposed in \citet{korba2022adaptive}, its use in off-policy contextual bandits is novel. Also, \citet{su2020doubly,metelli2021subgaussian} regularized the importance weights $w$ as $ \frac{\lambda_1 w}{\lambda_1 + w^2}\,, \lambda_1>0$ and   $ \frac{w}{1-\lambda_2 + \lambda_2 w }\,, \lambda_2 \in [0, 1]$, respectively. Thus, the expression of both corrections is very different from ours. More importantly, these corrections entail different properties than ours. Roughly speaking, our correction allows us to \emph{simultaneously} \textbf{(1)} control a tuning parameter $\alpha \in [0, 1]$ that is in a bounded domain $[0, 1]$, \textbf{(2)} without constraining the resulting importance weights to be bounded, \textbf{(3)} and to obtain PAC-Bayes generalization guarantees as the correction $\frac{\pi}{\pi_0^\alpha}$ is linear in $\pi$; a technical requirement of our analysis. In contrast, \citet{metelli2021subgaussian,su2020doubly} do not provide generalization guarantees; they focus on OPE and only propose heuristics for OPL. Those heuristics are not based on theory, in contrast with ours which is directly derived from our generalization bound. Also, our approach has favorable empirical performance (\cref{app:other_corrections}).

Although \citet[Lemma 1]{korba2022adaptive} show that smoothing the importance weights similarly to \texttt{IPS-}$\beta$ in \eqref{eq:exp_ips_policy_value} reduces the variance, it might still be unclear how $\alpha$ and $\beta$ trade the bias and variance of our estimators in off-policy contextual bandits. To see this, let $\alpha \in [0, 1]$, then we have
\begin{align}\label{eq:alpha_bias_variance} 
    |\mathbb{B}(\hat{R}_n^\alpha(\pi))|  & \leq \E{x \sim \nu, a \sim \pi(\cdot | x)}{1 - \pi_0(a | x)^{1-\alpha}} \,,\\
\mathbb{V}\left[\hat{R}_n^\alpha(\pi)\right] & \leq \frac{1}{n} \mathbb{E}_{x \sim \nu, a \sim \pi(\cdot | x)}\big[ \frac{\pi(a | x)}{\pi_0(a | x)^{2\alpha-1}} \big] \nonumber\,,
\end{align}
with $\mathbb{B}(\hat{R}_n^\alpha(\pi)) = \mathbb{E}[\hat{R}_n^\alpha(\pi)] - R(\pi)$  and $\mathbb{V}[\hat{R}_n^\alpha(\pi)] = \mathbb{E}[(\hat{R}_n^\alpha(\pi)-\mathbb{E}[\hat{R}_n^\alpha(\pi)])^2]$ are respectively the bias and the variance of $\hat{R}_n^\alpha(\pi)$. The bound of the bias in \eqref{eq:alpha_bias_variance} is minimized in $\alpha=1$ (standard IPS); in which case it is equal to 0 (standard IPS is unbiased). In contrast, the bound of the variance is minimized in $\alpha =0$. Thus if the variance is small or $n$ is large enough such that $\mathbb{E}[ \pi(a | x)/\pi_0(a | x)^{2\alpha-1} ] /n \rightarrow 0$, then we set $\alpha  \rightarrow 1$. Otherwise, we set $\alpha  \rightarrow 0$. This shows that $\alpha$ trades the bias and variance of $\hat{R}_n^\alpha$. More details and a similar discussion for $ \tilde{R}_n^\beta(\pi)$ are deferred to \cref{proofs:ope}.

\section{PAC-Bayes Analysis for Off-Policy Learning}\label{sec:opl}

We now derive generalization bounds for our estimator. We opt for the PAC-Bayes framework for the following reasons. First, it is known to provide some of the tightest generalization bounds in challenging scenarios \citep{farid2021generalization}, for aggregated and randomized predictors \citep{alquier2021user}. Second, the bounds have a Kullback–Leibler (KL) divergence \citep{van2014renyi} term $D_{\mathrm{KL}}(\mathbb{Q} \| \mathbb{P})$ that depends on a \emph{fixed prior} $\mathbb{P}$ and a \emph{learning posterior} $\mathbb{Q}$ (see \cref{subsec:pac_bayes_framekwork} for a brief introduction). This quantity can be seen as a complexity measure, similarly to the covering number \citep{maurer2009empirical}. The difference is that complexity measures are uniform on the space of policies while the KL term in PAC-Bayes depends on the prior $\mathbb{P}$ and the posterior $\mathbb{Q}$. This allows getting sharper bounds when the former is well chosen. Third, the PAC-Bayes perspective fits very well with OPL. In fact, a policy $\pi$ can be written as an aggregation of predictors under some distribution $\mathbb{Q}$. Thus the prior $\mathbb{P}$ can be associated with the logging policy $\pi_0$ that we want to improve upon while the posterior $\mathbb{Q}$ is related to the learning policy $\pi$. Fourth, \citet{london2019bayesian} showed that PAC-Bayes can lead to tractable and scalable objectives, an important consideration in practice.

\subsection{Elements of PAC-Bayes }\label{subsec:pac_bayes_framekwork}

Let $\mathcal{Z} = \mathcal{X} \times \mathcal{Y}$ be an instance space: e.g.,  $\mathcal{X}$ and $\mathcal{Y}$ are the input and output space in supervised learning. Let $\mathcal{H} = \set{h : \cX \rightarrow \mathcal{Y}}$ denote a hypothesis space of mappings from $\mathcal{X}$ to $\mathcal{Y}$ (predictors). Also, let $L : \mathcal{H} \times \mathcal{Z} \rightarrow \real$ be a loss function and assume access to data $\cD_n = (z_i)_{i \in [n]}$ drawn from an unknown distribution $\mathbb{D}$. Let $R(h)= \E{z \sim \mathbb{D}}{L(h, z)}$ be the risk of $h \in \mathcal{H}$ while $\hat{R}_n(h)= \frac{1}{n} \sum_{i=1}^n L(h, z_i)$ is its empirical counterpart. Then the main focus in PAC-Bayes is to study the generalization capabilities of random hypothesis $\mathbb{Q}$ on $\mathcal{H}$ by controlling the gap between the expected risk under $\mathbb{Q}$, $\E{h \sim \mathbb{Q}}{R(h)}$ and the expected empirical risk under $\mathbb{Q}$, $\mathbb{E}_{h \sim \mathbb{Q}}[\hat{R}_n(h)]$. For example, assume that $L(h, z) \in [0,1]$ for any $(h, z) \in \mathcal{H} \times \mathcal{Z}$, let $\mathbb{P}$ be a \emph{fixed prior} distribution on $\mathcal{H}$ and let $\delta \in (0, 1)$. Then with probability at least $1-\delta$ over $\cD_n \sim \mathbb{D}^n$, the following inequality holds \emph{simultaneously for any posterior} distribution $\mathbb{Q}$ on $\mathcal{H}$
\begin{align*}
  \E{h \sim \mathbb{Q}}{R(h)} \leq \mathbb{E}_{h \sim \mathbb{Q}}[\hat{R}_n(h)] + \sqrt{ \frac{D_{\mathrm{KL}}(\mathbb{Q} \| \mathbb{P})+\log \frac{2\sqrt{n}}{\delta}}{2n}}.
\end{align*}
This was originally proposed by \citet{mcallester1998some}, and the reader may refer to \citet{alquier2021user,guedj2019primer} for more elaborate introductions of PAC-Bayes theory.

\subsection{PAC-Bayes for Off-Policy Contextual Bandits}
Let $\mathcal{H}=\{h: \cX \rightarrow \cA\}$ be a hypothesis space of mappings from $\cX$ (contexts) to $\cA$ (actions). Given a policy $\pi$ and a context $x \in \cX$, the action distribution $\pi(\cdot|x)$ is induced by a distribution $\mathbb{Q}$ over $\mathcal{H}$ \citep{london2019bayesian} such as
\begin{align}\label{eq:pac_policies}
   &\pi(a |x) = \pi_{\mathbb{Q}}(a | x) = \E{h \sim \mathbb{Q}}{\mathbb{I}_{\{h(x)=a\}}}\,.
\end{align}
This is not an assumption since any policy $\pi$ has this form when $\mathcal{H}$ is rich enough \citep[Theorem 2]{sakhi2022pac}. From \eqref{eq:pac_policies}, we observe that policies can be seen as an aggregation $\E{h \sim \mathbb{Q}}{\cdot}$ (under some distribution $\mathbb{Q}$ on the pre-defined hypothesis space $\mathcal{H}$) of deterministic decision rules $\mathbb{I}_{\{h(x)=a\}}$. This allows formulating OPL as a PAC-Bayes problem. Before showing how this is achieved, we start by providing two practical policies of such form.

\textbf{Example 1 (softmax and mixed-logit policies):} we define the space $\mathcal{H} = \set{h_{\theta, \gamma} \, ; \theta \in \real^{dK}, \gamma \in \real^K}$ of mappings $h_{\theta, \gamma}(x) = \argmax_{a \in \cA} \phi(x)^\top \theta_a + \gamma_a$. Here $\phi(x)$ outputs a $d$-dimensional representation of $x$, and $\gamma_a$ is a standard Gumbel perturbation, $\gamma_a \sim {\rm G}(0, 1)$ for any $a \in \cA$. Then
\begin{align}\label{eq:softmax_pac_bayes}
    \pi^{\textsc{sof}}_{\theta}(a | x) &= \frac{\exp(\phi(x)^\top \theta_a)}{\sum_{a^\prime \in \cA}\exp(\phi(x)^\top  \theta_{a^\prime})}\,,\nonumber\\
    &\stackrel{(i)}{=}  \E{\gamma \sim {\rm G}(0, 1)^K}{\mathbb{I}_{\{ h_{\theta, \gamma}(x) = a \}}}\,,
\end{align}
where $(i)$ follows from the Gumbel-Max trick (GMT) \citep{luce2012individual,maddison2014sampling}. Thus a \texttt{softmax} policy $\pi^{\textsc{sof}}_{\theta}$ can be written as in \eqref{eq:pac_policies}. Now we also consider random parameters $\theta \sim \cN(\mu, \sigma^2 I_{dK})$ with $\mu \in \real^{dK}$ and $\sigma>0$. Then, let $\mathbb{Q}= \cN(\mu, \sigma^2I_{dK}) \times {\rm G}(0, 1)^K$, it follows that $\pi_{\mathbb{Q}} = \pi^{\textsc{mixL}}_{\mu, \sigma}$ is a mixed-logit policy and it reads
\begin{align}\label{eq:logit_pac_bayes}
 \pi^{\textsc{mixL}}_{\mu, \sigma}&(a | x) =  \E{\theta \sim \cN(\mu, \sigma^2 I_d)}{\frac{\exp(\phi(x)^\top \theta_a)}{\sum_{a^\prime \in \cA}\exp(\phi(x)^\top  \theta_{a^\prime})}}\,,\nonumber\\
    &= \E{\theta \sim \cN(\mu, \sigma^2 I_d)\,, \gamma \sim {\rm G}(0, 1)^K}{\mathbb{I}_{\{ h_{\theta, \gamma}(x) = a \}}}.
\end{align}

\textbf{Example 2 (Gaussian policies):} \citet{sakhi2022pac} removed the Gumbel noise $\gamma$ in \eqref{eq:logit_pac_bayes} and consequently defined the hypothesis space as $\mathcal{H} = \set{h_{\theta} \, ; \theta \in \real^{dK}}$ of mappings $h_{\theta}(x) = \argmax_{a \in \cA} \phi(x)^\top \theta_a$ for any $x \in \cX$. Then, let $\mathbb{Q} = \cN(\mu, \sigma^2 I_{dK})$, it follows that $ \pi_{\mathbb{Q}} = \pi^{\textsc{gaus}}_{\mu, \sigma}$ reads 
\begin{align}\label{eq:gaussian_pac_bayes}
   \pi^{\textsc{gaus}}_{\mu, \sigma}(a | x) = \E{\theta \sim \cN(\mu, \sigma^2 I_d)}{\mathbb{I}_{\{ h_{\theta}(x) = a \}}}\,.
\end{align}
To see why removing the Gumbel noise can be beneficial, the reader may refer to \cref{app:policies}. After motivating the definition of policies in \eqref{eq:pac_policies}, we are in a position to relate our estimators to the general PAC-Bayes framework in \cref{subsec:pac_bayes_framekwork}. One technical requirement of our proof is that the estimator should be linear in $\pi$. Thus we focus on $\hat{R}_n^\alpha(\cdot)$ since $\tilde{R}_n^\beta(\pi)$ is non-linear in $\pi$. Let $h \in \mathcal{H}$, $x \in \cX$, $a \in \cA$ and $c \in [-1, 0]$, we define the loss $L_\alpha$ as 
\begin{align}\label{eq:our_loss}
   & L_\alpha(h, x, a, c) = \frac{\mathbb{I}_{\{h(x)=a\}}}{\pi_0(a | x)^\alpha}c\,.
\end{align}
Using the definition in \eqref{eq:pac_policies} and the linearity of the expectation, we have that $\hat{R}^\alpha_n(\cdot)$ in \eqref{eq:exp_ips_policy_value} can be written as 
\begin{align*}
    \hat{R}^\alpha_n(\pi_{\mathbb{Q}}) = \E{h \sim \mathbb{Q}}{\frac{1}{n} \sum_{i=1}^n L_\alpha(h, x_i, a_i, c_i)}\,.
\end{align*}
Moreover, the expectation of $\hat{R}_n(\pi_{\mathbb{Q}})$ reads
\begin{align*}
    R^\alpha(\pi_{\mathbb{Q}}) = \mathbb{E}_{h \sim \mathbb{Q}} \E{(x, a, c) \sim \mu_{\pi_0}}{L_\alpha(h, x, a, c)}\,.
\end{align*}
Finally, the main quantity of interest, the risk $R(\pi_{\mathbb{Q}})\,,$ can be expressed in terms of the loss with $\alpha=1\,,$ $L_1\,,$ as
\begin{align*}
    R(\pi_{\mathbb{Q}}) = \mathbb{E}_{h \sim \mathbb{Q}} \E{(x, a, c) \sim \mu_{\pi_0}}{L_1(h, x, a, c)}\,.
\end{align*}
Since $\hat{R}^\alpha_n(\pi_{\mathbb{Q}})$ is an unbiased estimator of $ R^\alpha(\pi_{\mathbb{Q}})$, PAC-Bayes can be used to bound $R^\alpha(\pi_{\mathbb{Q}}) - \hat{R}^\alpha_n(\pi_{\mathbb{Q}})$. This will allow bounding our quantity of interest $R(\pi_{\mathbb{Q}}) - \hat{R}^\alpha_n(\pi_{\mathbb{Q}})$.

\subsection{Main Result}\label{subsec:main_result}
To ease the exposition, we assume that the costs are deterministic. Then, in logged data $\cD_n$, $c_i = c(x_i, a_i)$ for any $i \in [n]$. Note that the same result holds for stochastic costs. We discuss our result and sketch its proof in \cref{sec:discussion}. The complete proof can be found in \cref{proof:main_thm_proof}.

\begin{theorem}\label{thm:main_result} Let $\lambda>0$,  $n \ge 1$, $\delta \in (0, 1)$, $\alpha \in [0, 1]$, and let $\mathbb{P}$ be a fixed prior on $\mathcal{H}$, then with probability at least $1-\delta$ over draws $\cD_n \sim \mu_{\pi_0}^n$, the following holds simultaneously for any posterior $\mathbb{Q}$ on $\mathcal{H}$ 
\begin{align*}
  |R(\pi_{\mathbb{Q}}) -\hat{R}_n^\alpha(\pi_{\mathbb{Q}})| \leq \sqrt{ \frac{{\textsc{kl}}_{1}(\pi_{\mathbb{Q}})}{2n} } + B_n^\alpha(\pi_{\mathbb{Q}})  +
\frac{{\textsc{kl}}_{2}(\pi_{\mathbb{Q}})}{n \lambda } \\ + \frac{\lambda}{2}\bar{V}_n^\alpha(\pi_{\mathbb{Q}})\,.
\end{align*}
where ${\textsc{kl}}_{1}(\pi_{\mathbb{Q}})  =D_{\mathrm{KL}}(\mathbb{Q} \| \mathbb{P})+\ln \frac{4\sqrt{n}}{\delta}\,,$ and
\begin{align*}
    &{\textsc{kl}}_{2}(\pi_{\mathbb{Q}})  =  D_{\mathrm{KL}}(\mathbb{Q} \| \mathbb{P})+\ln \frac{4}{\delta}\,,\\
   & B_n^\alpha(\pi_{\mathbb{Q}}) = 1 - \frac{1}{n}\sum_{i=1}^{n} \E{a \sim \pi_{\mathbb{Q}}(\cdot | x_i)}{\pi_0^{1-\alpha}(a | x_i)}\,,\\
    &\bar{V}_n^\alpha(\pi_{\mathbb{Q}}) = \frac{1}{n}\sum_{i=1}^n  \E{a \sim \pi_{0}(\cdot | x_i)}{\frac{\pi_{\mathbb{Q}}(a | x_i)}{\pi_0(a | x_i)^{2\alpha }}} + \frac{\pi_{\mathbb{Q}}(a_i | x_i)c_i^2}{\pi_0(a_i | x_i)^{2\alpha}}.
\end{align*}
\end{theorem}
We start by clarifying that the prior $\mathbb{P}$ can be any fixed distribution on $\mathcal{H}$. If we have access to $\mathbb{P}_0$ on $\mathcal{H}$ such that $\pi_0 = \pi_{\mathbb{P}_0}$, then it is natural to set $\mathbb{P} =  \mathbb{P}_0$. But this is just a choice and one may use priors that do not depend on $\pi_0$. Now we explain the main terms in our bound. First, the terms $ {\textsc{kl}}_{1}(\pi_{\mathbb{Q}})$ and  ${\textsc{kl}}_{2}(\pi_{\mathbb{Q}})$ contain the divergence $D_{\mathrm{KL}}(\mathbb{Q} \| \mathbb{P})$ which penalizes posteriors $\mathbb{Q}$ that differ a lot from the prior $\mathbb{P}$. Moreover, $B_n^\alpha(\pi_{\mathbb{Q}})$ is the bias conditioned on the contexts $(x_i)_{i \in [n]}\,$; $B_n^\alpha(\pi_{\mathbb{Q}})=0$ when $\alpha=1$ and $B_n^\alpha(\pi_{\mathbb{Q}})>0$ otherwise. Also, the first term in $\bar{V}_n^\alpha(\pi_{\mathbb{Q}})$ \emph{resembles} the theoretical second moment of the regularized importance weights $\frac{\pi}{\pi^\alpha_0}$ (without the cost) when they are seen as random variables. Similarly, the second term in $\bar{V}_n^\alpha(\pi_{\mathbb{Q}})$ \emph{resembles} the empirical second moment of $\frac{\pi}{\pi^\alpha_0}c$ (with the cost). Finally, if $\bar{V}_n^\alpha(\pi_{\mathbb{Q}})$ is bounded, then we can set $\lambda = 1/\sqrt{n}$, in which case our bound scales as $\mathcal{O}(1/\sqrt{n} + B_n^\alpha(\pi_{\mathbb{Q}}))$. In practice, we set $\alpha \approx 1$ leading to $B_n^\alpha(\pi_{\mathbb{Q}})) \approx 0$ and the bound would scale as $\mathcal{O}(1/\sqrt{n})$.

This bound motivates the idea that we only need to control the second moments $\bar{V}_n^\alpha(\pi_{\mathbb{Q}})$ to get generalization guarantees for $\hat{R}_n^\alpha(\cdot)$.  In particular, one of the main strengths of our result is that it holds for standard IPS with $\alpha=1$ under the assumption that $\bar{V}_n^1(\pi_{\mathbb{Q}})$ is bounded. This assumption is less restrictive than assuming that the importance weight as a random variable, $\pi_{\mathbb{Q}}(a | x)/\pi_0(a|x)$, is bounded, a required assumption for traditional concentration bounds. In contrast, $\bar{V}_n^\alpha(\pi_{\mathbb{Q}})$ only involves the \emph{expectations} of the random variables $\pi_{\mathbb{Q}}(a | x_i) / \pi_0(a | x_i)^{2\alpha }$, and ratios of $\pi_0$ evaluated at observed contexts and actions and  $(x_i, a_i)_{i \in [n]}$, that have non-zero probabilities under $\pi_0$ by definition.

Our result holds for fixed $\lambda>0$ and $\alpha \in [0, 1]$. In \cref{app:thm_extension}, we extend this to any potentially data-dependent $\lambda \in (0, 1)$ and $\alpha \in (0, 1]$. The assumption that $c \in [-1, 0]$ can be relaxed to $c \in [-B, 0]$ up to additional factors $B^2$ and $B$ in $\bar{V}_n^\alpha(\pi_{\mathbb{Q}})$ and  ${\textsc{kl}}_{1}(\pi_{\mathbb{Q}})$, respectively. Finally, our bound is suitable for stochastic first-order optimization \citep{robbins1951stochastic} since data-dependent quantities are not inside a square root. This is important for scalability.

\subsection{Adaptive and Data-Driven Tuning of $\alpha$}\label{subsec:data_dep_alpha}

\cref{thm:main_result} assumes that $\alpha$ is fixed (although we extend it for data-dependent $\alpha$ in \cref{app:thm_extension}). However, providing a procedure to tune $\alpha$ in an adaptive and data-dependent fashion is important in practice. Thus we propose to set 
\begin{align}\label{eq:data_dependent_alpha}
    \alpha_* = \argmin_{\alpha \in [0, 1]} B^\alpha_n(\pi_{\mathbb{Q}})  + \sqrt{\frac{2 {\textsc{kl}}_{2}(\pi_{\mathbb{Q}})\bar{V}^\alpha_n(\pi_{\mathbb{Q}}) }{n}}\,,
\end{align}
where all the terms are defined in \cref{thm:main_result}. Roughly speaking, $\alpha_*$ establishes a bias-variance trade-off; it minimizes the sum of the bias term $B^\alpha_n(\pi_{\mathbb{Q}})$ and the square root of the second moment term $\bar{V}^\alpha_n(\pi_{\mathbb{Q}})$, weighted by $\sqrt{\frac{2 {\textsc{kl}}_{2}(\pi_{\mathbb{Q}})}{n}}$. Here \eqref{eq:data_dependent_alpha} is obtained by minimizing the bound in \cref{thm:main_result} with respect to both $\alpha$ and $\lambda$ as follows. First, we minimize the bound in \cref{thm:main_result} with respect to $\lambda$; the minimizer is $\lambda_* = \sqrt{\frac{2 {\textsc{kl}}_{2}(\pi_{\mathbb{Q}}) }{n\bar{V}^\alpha_n(\pi_{\mathbb{Q}})}}$. Then, the bound in \cref{thm:main_result} evaluated at $\lambda = \lambda_*$ becomes
\begin{align}\label{eq:data_dependent_alpha_2}
    \sqrt{\frac{ {\textsc{kl}}_{1}(\pi_{\mathbb{Q}})}{2n}} +  B^\alpha_n(\pi_{\mathbb{Q}})  + \sqrt{\frac{2 {\textsc{kl}}_{2}(\pi_{\mathbb{Q}})\bar{V}^\alpha_n(\pi_{\mathbb{Q}}) }{n}}\,.
\end{align}
Finally, $\alpha_*$ is defined as the minimizer of \eqref{eq:data_dependent_alpha_2} with respect to $\alpha \in [0, 1]$, and $\sqrt{\frac{ {\textsc{kl}}_{1}(\pi_{\mathbb{Q}})}{2n}}$ does not appear in \eqref{eq:data_dependent_alpha} as it does not depend on $\alpha$. Note that $\alpha_*$ depends on both logged data $\cD_n$ and the learning policy $\pi_{\mathbb{Q}}$. Thus it is adaptive; its value changes in each iteration during optimization.

\section{Discussion}\label{sec:discussion}
We start by interpreting and comparing our results to related work. Then, we present the technical challenges in \cref{subsec:tech_challenges}. After that, we sketch our proof in \cref{subsec:sketch}.

\subsection{Interpretation and Comparison to Related Work}\label{subsec:interpretation}

\cref{thm:main_result} gives insight into the number of samples needed so that the performance of $\hat{\pi}_n$ is close to that of the optimal policy $\pi_*$. To simplify the problem, we consider the Gaussian policies in \eqref{eq:gaussian_pac_bayes} and assume that there exists $\mathbb{Q}_* = \cN(\mu_*, I_{dK})$ with $\mu_* \in \real^{dK}$ such that the optimal policy is $\pi_* = \pi_{\mathbb{Q}_*}$. Also, we let the prior $\mathbb{P} =  \cN(\mu_0, I_{dK})$ and assume that $\pi_0$ is uniform. This is possible since as we said before, the prior $\mathbb{P}$ does not have to depend on the logging policy $\pi_0$. Then we have that $D_{\mathrm{KL}}(\mathbb{Q}_* \| \mathbb{P}) = \norm{\mu_* - \mu_0}^2 / 2$, $ B_n^\alpha(\pi_{\mathbb{Q}_*}) = 1 - 1/K^{1-\alpha}$ and $\bar{V}_n^\alpha(\pi_{\mathbb{Q}_*}) \leq 2K^{2 \alpha}$. The last inequality is not tight but it allows getting an easy-to-interpret term that does not depend on $n$. Now let $\epsilon> 2(1 - K^{\alpha-1})$ for $\alpha \in [1- \log 2 / \log K, 1]$. This condition on $\alpha$ ensures that $\epsilon \in [0, 1]$ and it is mild as $\alpha$ is often close to 1. Then, it holds with high probability that
\begin{align*}
 n \, \widetilde{>} \Big(\frac{ \norm{\mu_* - \mu_0}^2  + K^{2 \alpha}}{\epsilon -  2 (1 - K^{\alpha-1})}\Big)^2  \implies R(\hat{\pi}_n) \leq R(\pi_{\mathbb{Q}_*}) + \epsilon \,,
\end{align*}
where we omit constant and logarithmic terms in $\widetilde{>}$. This gives an intuition on the sample complexity for our procedure. In particular, fewer samples are needed in four cases. The first is when $\epsilon$ is large, which means that we afford to learn a policy whose performance is far from the optimal one. The second is when the prior $\mathbb{P}$ is close to $\mathbb{Q}_*$, that is when $\norm{\mu_* - \mu_0}$ is small. This highlights that the choice of the prior $\mathbb{P}$ is important. The third is when the second-moment term $K^{2 \alpha}$ is small. The fourth is when the bias $B_n^\alpha(\pi_{\mathbb{Q}_*})$ is small. In particular, when $\alpha=1$, the bias is 0. In contrast, the second-moment term is minimized in $\alpha=0$. This is where the choice of $\alpha$ matters. The proofs of these claims and more detail can be found in \cref{proof:practice_theory}.

Prior works \citep{swaminathan2015batch,london2019bayesian,sakhi2022pac} do not provide such insight for two reasons. They only derived one-sided inequalities and thus they cannot relate the risk of the learned policy with the optimal one as we discussed in the last three paragraphs of \cref{sec:setting}. Also, their bounds do not contain a bias term and as a result, they are minimized in $\tau=1$. In contrast, ours have a bias term and this allows seeing the effect of $\alpha$.

Our paper derives a \emph{tractable generalization bound} for an estimator other than clipped IPS in \eqref{eq:clip_ips_policy_value}, which also holds for the standard IPS in \eqref{eq:ips_policy_value}. The bounds in \citet{swaminathan2015batch,london2019bayesian,sakhi2022pac} have a multiplicative dependency on the clipping threshold ($M$ or $1/\tau$ in \eqref{eq:clip_ips_policy_value}). Standard IPS is recovered when $M \rightarrow \infty$ (or $\tau = 0$) in which case their bounds explode. We successfully avoid any similar dependency on $\alpha$. Moreover, \citet{swaminathan2015batch,london2019bayesian} only used their generalization bounds to inspire learning principles. Although we directly optimize our theoretical bound (\cref{thm:main_result})
in our experiments, our analysis also inspires a learning principle where we simultaneously penalize the $L_2$ distance, the variance and the bias. That is, we find $\mu \in \real^{dK}$ that minimizes
\begin{align}\label{eq:learning_principle}
\hspace{-0.2cm} \hat{R}^\alpha_n(\pi_{\mu}) + \lambda_1 \norm{\mu - \mu_0}^2  + \lambda_2  \bar{V}_n^\alpha(\pi_{\mu}) + \lambda_3 B_n^\alpha(\pi_{\mu})\,.
\end{align}
Here $\lambda_1, \lambda_2$ and $\lambda_3$ are tunable hyper-parameters, $\pi_{\mu}$ can be the Gaussian policy in \eqref{eq:gaussian_pac_bayes}, $\pi_{\mu} = \pi^{\textsc{gaus}}_{\mu, 1}$, with a fixed $\sigma=1$, and $\mu_0$ is the mean of the prior $\mathbb{P} = \cN(\mu_0, I_{dK})$. Existing works either penalize the $L_2$ distance or the variance. For completeness, we also show that this learning principle should be preferred over existing ones in \cref{app:add_discussion}. 

\subsection{Technical Challenges}\label{subsec:tech_challenges}

\citet{london2019bayesian,sakhi2022pac} derived PAC-Bayes generalization bounds for the estimator \texttt{IPS-max} in \eqref{eq:clip_ips_policy_value}. Extending their analyses to our case is not straightforward. First, their estimator \texttt{IPS-max} is lower bounded by $-1/\tau$, and thus they relied on traditional techniques for $[0,1]$-losses \citep{alquier2021user}. In contrast, our loss in \eqref{eq:our_loss} is not lower bounded, and controlling it without assuming that the importance weights are bounded is challenging.

Moreover, their bounds have a multiplicative dependency on $1/\tau$, hence they explode as $\tau \rightarrow 0$. This makes them vacuous for small values of $\tau$ and inapplicable to the standard IPS estimator in \eqref{eq:ips_policy_value} recovered for $\tau =0$. 
In contrast, our bound does not have a similar dependency on $\alpha$ and it is also valid for standard IPS recovered for $\alpha=1$. Moreover, we derive two-sided inequalities rather than one-sided ones for the important reasons that we priorly discussed.  
This requires carefully controlling in \emph{closed-form} the absolute value of the bias. Prior works only used that the bias is negative which was enough to obtain one-sided inequalities. 

Explaining other challenges requires stating a result that inspired our analysis: \citet{kuzborskij2019efron} derived PAC-Bayes generalization bounds for unbounded losses by only controlling their second moments. Recently, \citet{haddouche2022pac} proposed a similar result using Ville's inequality \citep{bercu2008exponential}. Adapting their theorem to our problem is given \cref{prop:direct_application}. We slightly adapt their proof to get a \emph{two-sided} inequality for a \emph{negative} loss. The proof is deferred to \cref{proof:direct_application}. 

\begin{proposition}\label{prop:direct_application}  Let $\lambda>0$, $n \geq 1$, $\delta \in(0,1)$, $\alpha \in [0, 1]$ and let $\mathbb{P}$ be a fixed prior on $\mathcal{H}$, then with probability at least $1-\delta$ over draws $\cD_n \sim \mu_{\pi_0}^n$, the following holds simultaneously for all posteriors, $\mathbb{Q}$, on $\mathcal{H}$
\begin{align}\label{eq:direct_application}
    &|R^\alpha(\pi_{\mathbb{Q}}) -\hat{R}_n^\alpha(\pi_{\mathbb{Q}})|   \leq \frac{D_{\mathrm{KL}}(\mathbb{Q} \| \mathbb{P})+\log \frac{2}{\delta}}{\lambda n} \\ &+\frac{\lambda}{2 n} \sum_{i=1}^n \frac{\pi_{\mathbb{Q}}(a_i | x_i)}{\pi_0^{2\alpha}(a_i | x_i)} c_i^2 +\frac{\lambda}{2}\mathbb{E}_{(x, a, c) \sim \mu_{\pi_0}}\left[\frac{\pi_{\mathbb{Q}}(a | x)}{\pi_0^{2\alpha}(a | x)} c^2\right]\nonumber\,,
\end{align}
\end{proposition}

There are two main issues with \cref{prop:direct_application}. First, the term $\mathbb{E}_{(x, a, c) \sim \mu_{\pi_0}}\big[\frac{\pi_{\mathbb{Q}}(a | x)}{\pi_0^{2\alpha}(a | x)} c^2\big]$ in \eqref{eq:direct_application} is intractable. One could bound $c^2$ by $1$, but the resulting term will still be intractable due to the expectation over the unknown distribution of contexts $\nu$. Second, we need an upper bound of $|R(\pi_{\mathbb{Q}}) - \hat{R}_n^\alpha(\pi_{\mathbb{Q}})|$ while \cref{prop:direct_application} only provides one for $|R^\alpha(\pi_{\mathbb{Q}})-\hat{R}_n^\alpha(\pi_{\mathbb{Q}})|$. Thus it remains to quantify the approximation error $|R(\pi_{\mathbb{Q}})-R^\alpha(\pi_{\mathbb{Q}})|$. This will also require computing an expectation over $x \sim \nu$, which is intractable.

\subsection{Sketch of Proof for \Cref{thm:main_result}}\label{subsec:sketch}

We conclude by showing how the technical challenges above were solved. First, We decompose $R(\pi_{\mathbb{Q}})-\hat{R}_n^\alpha(\pi_{\mathbb{Q}})$ as
\begin{align*}
    R(\pi_{\mathbb{Q}})-\hat{R}_n^\alpha &(\pi_{\mathbb{Q}}) = I_1 + I_2 + I_3\,, \qquad \text{where}
\end{align*}
  \begin{align*}
    I_1 &= R(\pi_{\mathbb{Q}}) - \frac{1}{n}\sum_{i=1}^n R(\pi_{\mathbb{Q}} | x_i)\,,\\
     I_2 &=  \frac{1}{n} \sum_{i=1}^n R(\pi_{\mathbb{Q}} | x_i) - \frac{1}{n}\sum_{i=1}^n R^\alpha(\pi_{\mathbb{Q}} | x_i)\,,\\
     I_3 &= \frac{1}{n}\sum_{i=1}^n R^\alpha(\pi_{\mathbb{Q}} | x_i) - \hat{R}_n^\alpha(\pi_{\mathbb{Q}})\,,\\
    R(\pi_{\mathbb{Q}} | x_i) & = \E{a \sim \pi_{\mathbb{Q}}(\cdot | x_i)}{c(x_i, a)}\,,\\
    R^\alpha(\pi_{\mathbb{Q}} | x_i) &= \mathbb{E}_{a \sim \pi_0(\cdot | x_i)}\Big[\frac{\pi_{\mathbb{Q}}(a | x_i)}{\pi_0(a | x_i)^\alpha}c(x_i, a)\Big]\,.
\end{align*}
$I_1$ is the estimation error of the empirical mean of the risk using $n$ i.i.d. contexts $(x_i)_{i \in [n]}$. This term is introduced to avoid the intractable expectation over $x \sim \nu$. Moreover, $I_2$ is the bias term conditioned on the contexts $(x_i)_{i \in [n]}$ and we bound it in closed-form. Finally, $I_3$ is the estimation error of the risk conditioned on the contexts $(x_i)_{i \in [n]}$. Again, this conditioning allows us to avoid the intractable expectation over $x \sim \nu$ and to consequently bound $|I_3|$ by tractable terms. First, \citet[Theorem~3.3]{alquier2021user} yields that with probability at least $1-\frac{\delta}{2}$, it holds for any $\mathbb{Q}$ on $\mathcal{H}$ that
\begin{align*}
    |I_1| \leq \sqrt{ \frac{D_{\mathrm{KL}}(\mathbb{Q} \| \mathbb{P})+\log \frac{4\sqrt{n}}{\delta}}{2n}} \,.
\end{align*}
Also, $|I_2|$ is bounded similarly to \eqref{eq:alpha_bias_variance} as 
\begin{align*}
   |I_2| \leq \frac{1}{n} \sum_{i=1}^{n} \E{a \sim \pi_{\mathbb{Q}}(\cdot | x_i)}{1 - \pi_0^{1-\alpha}(a | x_i)}\,.
\end{align*}
Bounding $|I_3|$ is achieved by expressing it using martingale difference sequences $(f_i(a_i, h))_{i \in [n]}$ that we construct as follows. Let $(\mathcal{F}_i)_{i \in \{0\} \cup [n]}$ be a filtration adapted to $(S_i)_{i \in [n]}$ where $S_i = (a_\ell)_{\ell \in [i]}$ for any $i \in [n]$, we define
\begin{align*}
f_i\left(a_i, h\right) = \E{a \sim \pi_0(\cdot | x_i)}{\frac{\mathbb{I}_{\{h(x_i)=a\}}c(x_i, a)}{\pi_0(a | x_i)^\alpha} } \\- \frac{\mathbb{I}_{\{h(x_i)=a_i\}}c_i}{\pi_0(a_i | x_i)^\alpha}\,.
\end{align*}
Then we show that for any $h \in \mathcal{H}$,  $(f_i(a_i, h))_{i \in [n]}$ is a martingale difference sequence. After that, we apply \citet[Theorem 5]{haddouche2022pac} and obtain that with probability at least $1-\delta/2$, it holds for any $\mathbb{Q}$ on $\mathcal{H}$ that
\begin{align*}
    \left|\E{h \sim \mathbb{Q}}{M_n(h)}\right| &\leq \frac{D_{\mathrm{KL}}(\mathbb{Q} \| \mathbb{P})+\log \frac{4}{\delta}}{\lambda} +\frac{\lambda}{2} \E{h \sim \mathbb{Q}}{\bar{V}_n(h)},\nonumber
\end{align*}
where $M_n(h)=\sum_{i=1}^n f_i\left(a_i, h\right)$ and $\bar{V}_n(h)=\sum_{i=1}^n f_i\left(a_i, h\right)^2 + \mathbb{E}\left[f_i\left(a_i, h\right)^2 | \mathcal{F}_{i-1}\right]$. Then notice that $\E{h \sim \mathbb{Q}}{M_n(h)}$ can be expressed in terms of $I_3$ as 
\begin{align*}
    \E{h \sim \mathbb{Q}}{M_n(h)} &= \sum_{i=1}^n R^\alpha(\pi_{\mathbb{Q}} | x_i) - n\hat{R}_n^\alpha(\pi_{\mathbb{Q}}) = n I_3\,,
\end{align*}
Moreover, $\E{h \sim \mathbb{Q}}{\bar{V}_n(h)}$ is bounded by
\begin{align*}
    \sum_{i=1}^n  \E{a \sim \pi_0(\cdot | x_i)}{\frac{\pi_{\mathbb{Q}}(a | x_i)}{\pi_0(a | x_i)^{2\alpha}}}+ \frac{\pi_{\mathbb{Q}}(a_i | x_i)}{\pi_0(a_i | x_i)^{2\alpha}}c_i^2\,.
\end{align*}
Thus with probability at least $1-\frac{\delta}{2}$, it holds for any $\mathbb{Q}$ that
\begin{align*}
   |I_3 | &\leq \frac{D_{\mathrm{KL}}(\mathbb{Q} \| \mathbb{P})+\log\frac{4}{\delta}}{n\lambda}+ \frac{\lambda}{2n} \sum_{i=1}^n\frac{\pi_{\mathbb{Q}}(a_i | x_i)}{\pi_0(a_i | x_i)^{2\alpha}}c_i^2\\
   &  + \frac{\lambda}{2n}\sum_{i=1}^n  \E{a \sim \pi_0(\cdot | x_i)}{\frac{\pi_{\mathbb{Q}}(a | x_i)}{\pi_0(a | x_i)^{2\alpha}}}\,.
\end{align*}
Our result is obtained by bounding $|I_1| + |I_2| + |I_3|$. One shortcoming of our analysis is that $\bar{V}_n^\alpha(\pi_{\mathbb{Q}})$ is not exactly and only resembles the sum of the theoretical and empirical second moments of our estimator. Precisely, the terms $\pi_{\mathbb{Q}}/\pi_0^{2\alpha}$ should be $\pi_{\mathbb{Q}}^2/\pi_0^{2\alpha}$. This problem arises due to our definition of the martingale difference sequences $(f_i(a_i, h))_{i \in [n]}$ in \eqref{eq:our_loss}. Precisely, in our proof, we compute the square $f_i(a_i, h)^2$. However, the square of an indicator function is the indicator function itself. Thus applying the expectation afterwards, $\E{h \sim \mathbb{Q}}{f_i(a_i, h)^2}$, leads to $\pi_{\mathbb{Q}}$ appearing instead of $\pi_{\mathbb{Q}}^2$. This issue is inherent in the PAC-Bayes formulation and seminal works \citep{london2019bayesian,sakhi2022pac} would suffer the same issue. Solving this would be beneficial and we leave it to future work.

\begin{figure*}[t!]
  \centering  \includegraphics[width=\linewidth]{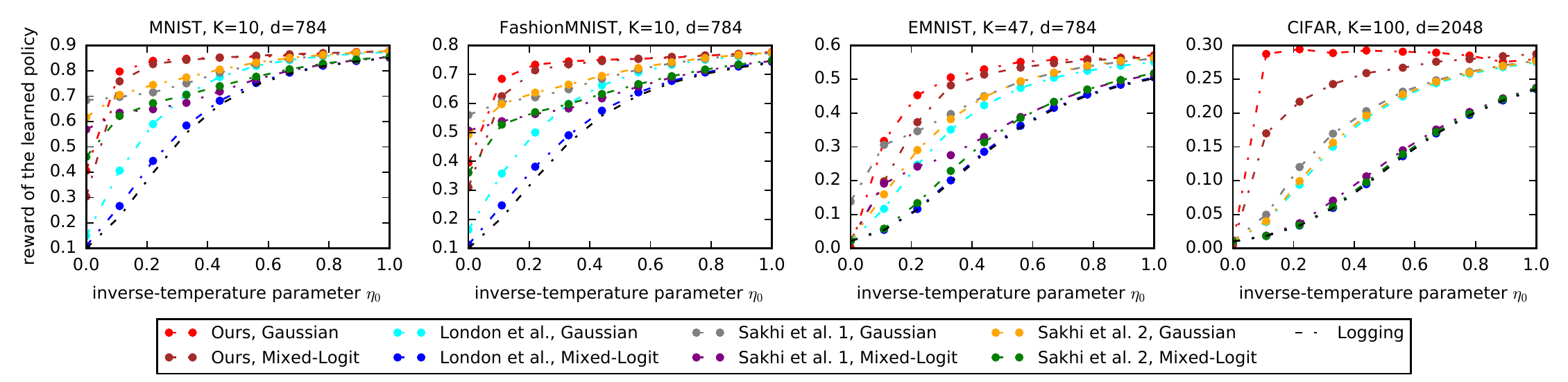}
  \vspace{-0.8cm}
  \caption{The reward of the learned policy using one of the baselines with varying quality of the logging policy $\eta_0 \in [0, 1]$.} 
  \label{fig:main_exp_results}
\end{figure*}

\section{Experiments}\label{sec:experiments}

We briefly present our experiments. More details and discussions can be found in \Cref{app:all_experiments}. We consider the standard supervised-to-bandit conversion \citep{agarwal2014taming} where we transform a supervised training set $\mathcal{S}^{\textsc{tr}}_{n}$ to a logged bandit data $\cD_n$ as described in \cref{alg:supervised_to_bandit} in \cref{app:setup}. Here the action space $\cA$ is the label set and the context space $\cX$ is the input space. Then, $\cD_n$ is used to train our policies. After that, we evaluate the reward of the learned policies on the supervised test set $\mathcal{S}^{\textsc{ts}}_{n_{\textsc{ts}}}$ as described in \cref{alg:supervised_to_bandit_test} in \cref{app:setup}. Roughly speaking, the resulting reward quantifies the ability of the learned policy to predict the true labels of the inputs in the test set. This is our performance metric; the higher the better. We use 4 image classification datasets \texttt{MNIST} \citep{lecun1998gradient}, \texttt{FashionMNIST} \citep{xiao2017fashion},  \texttt{EMNIST} \citep{cohen2017emnist} and \texttt{CIFAR100} \citep{krizhevsky2009learning}.

The logging policy is defined as $\pi_0 = \pi_{\eta_0 \cdot \mu_0}^{\textsc{sof}}$ in \eqref{eq:softmax_pac_bayes}, where $\mu_0 = (\mu_{0,a})_{a \in \cA} \in \real^{dK}$ and $\eta_0 \in [0, 1]$ is the inverse-temperature parameter. The higher $\eta_0$, the better the performance of $\pi_0$. When $\eta_0=0$, $\pi_0$ is uniform. The parameters $\mu_0$ are learned using 5\% of the training set $\mathcal{S}^{\textsc{tr}}_{n}$. In our experiments, we consider both, Gaussian and mixed-logit policies, in \eqref{eq:logit_pac_bayes} and \eqref{eq:gaussian_pac_bayes}, for which we set the prior as $\mathbb{P} = \cN(\eta_0 \mu_0, I_{dK})$ and $\mathbb{P} = \cN(\eta_0 \mu_0, I_{dK}) \times {\rm G}(0, 1)^K$, respectively. Given that $\mu_0$ are learnt on $5\%$ of $\mathcal{S}^{\textsc{tr}}_{n}$, we train our policies on the remaining $95\%$ portion of $\mathcal{S}^{\textsc{tr}}_{n}$ to match our theory that requires the prior to not depend on training data. The policies are trained using Adam \citep{kingma2014adam} with a learning rate of $0.1$ for $20$ epochs.

We compare our bound to those in \citet{london2019bayesian,sakhi2022pac}; discarding the intractable bound in \citet{swaminathan2015batch} as it requires computing a covering number. Here we do not include the learning principles in \citet{swaminathan2015batch,london2019bayesian} since we directly optimize our bounds. But we make such a comparison in \cref{app:add_discussion} for completeness, showing the favorable performance of our bound and the newly proposed learning principle in \eqref{eq:learning_principle}. Also, we do not compare to \citet{su2020doubly,metelli2021subgaussian} since they do not provide generalization guarantees; they focus on OPE and only propose a heuristic for OPL. However, we still show the favorable performance of our approach in OPL compared to \citet{su2020doubly,metelli2021subgaussian} in \cref{app:other_corrections} for completeness.

Prior methods are not named. Thus we refer to them as \textbf{(Author, Policy)} where \textbf{Author} $\in$ \{\textbf{Ours, London et al., Sakhi et al. 1, Sakhi et al. 2\} and Policy $\in$ \{Gaussian, Mixed-Logit}\}. Here \textbf{Ours}, \textbf{London et al.}, \textbf{Sakhi et al. 1} and \textbf{Sakhi et al. 2} correspond to \cref{thm:main_result},  \citet[Theorem 1]{london2019bayesian}, \citet[Proposition 1]{sakhi2022pac}, and \citet[Proposition 3]{sakhi2022pac}, respectively. Since we have two classes of policies, each bound leads to two baselines. For example, \citet[Theorem 1]{london2019bayesian} leads to\textbf{ (London et al., Gaussian)} and \textbf{(London et al., Mixed-Logit)}. More details are provided in \cref{app:baselines}.

In \cref{fig:main_exp_results}, we report the reward of the learned policies. Here we fix $\tau= 1/\sqrt[\leftroot{-2}\uproot{2}4]{n} \approx 0.06$ and $\alpha = 1-1/\sqrt[\leftroot{-2}\uproot{2}4]{n} \approx 0.94$ so that when $n$ is large enough, both $\hat{R}_n^\tau(\pi)$ and $ \hat{R}_n^\alpha(\pi)$ approach $\hat{R}_n^{\textsc{ips}}(\pi)$ \citep{ionides2008truncated}. This is because standard IPS should be preferred when $n \rightarrow \infty$. To have a fair comparison, we fixed $\alpha$ instead of tuning it in an adaptive fashion as described in \cref{subsec:data_dep_alpha}. However, we also provide the results with an adaptive $\alpha$ in \cref{fig:varying_params}. Let us start with interpreting \cref{fig:main_exp_results} (with fixed $\alpha$ and $\tau$). Overall, our method outperforms all the baselines. We also observe that Gaussian policies behave better than mixed-logit policies. However, this is less significant for our method where the performances of both Gaussian and mixed-logit policies are comparable. Moreover, our method reaches the maximum reward even when the logging policy has an average performance. In contrast, the baselines only reach their best reward when the logging policy is well-performing ($\eta_0 \approx 1$), in which case minor to no improvements are made. Finally, the baselines induce a better reward when the logging policy is uniform ($\eta_0 = 0$). But our method has a better reward when $\eta_0>0$, which is more common in practice. 

Our choice of $\tau$ and $\alpha$ does not affect the above conclusions. In \cref{fig:varying_params} (left-hand side), we compare our method with the best baseline, \textbf{(Sakhi et al. 2)} with Gaussian policies, for $20$ evenly spaced values of $\tau \in (0, 1)$ and $\alpha \in (0, 1)$. We also include the results using the adaptive tuning procedure of $\alpha$ described in \cref{subsec:data_dep_alpha} (green curve). This procedure is reliable since the performance with an adaptive $\alpha$ (green curve) is comparable with the best possible choice of $\alpha$. Also, our method consistently outperforms the best baseline \textbf{(Sakhi et al. 2)} with the best value of $\tau$ when the logging policy is not uniform $(\eta_0 >0)$. Also, there is no very bad choice of $\alpha$, in contrast with $\tau = 10^{-5}$ (dark blue plot) which led to minimal improvement upon all logging policies. This might be due to the $1/\tau$ dependency in existing bounds. 

To see the effect of $\alpha$, we consider the following experiment. We split the logging policies into two groups. The first is called \emph{modest logging} which corresponds to logging policies $\pi_0$ whose $\eta_0$ is between $0$ and $0.5$. This group includes the uniform policy and other average-performing policies. The second is called \emph{good logging} and it includes the logging policies whose $\eta_0$ is between $0.5$ and $1$. Then, for each $\alpha$, we compute the average reward of the learned policy, with that value of $\alpha$, across these two groups. This leads to the two red and green curves in \cref{fig:varying_params} (right-hand side). Overall, we observe that $\alpha \approx 0.7$ leads to the best performance across the modest logging group. Thus when the performance of the logging policy is bad or average, which is common in practice, regularization can be critical. In contrast, when the performance of the logging policy is already good and $n$ is large enough, regularization might not be needed and $\alpha \approx 1$ would also lead to good performance. This is one of the main strengths of our bound; it holds for the standard IPS recovered with $\alpha=1$. This result goes against the belief that clipped IPS should always be preferred to standard IPS. Here, our bound applied to standard IPS outperformed clipping by a large margin when the logging policy is relatively well-performing. Similar results for the other datasets are deferred to \cref{app:add_results}.

\begin{figure}[ht]
\begin{center}
\vskip -0.1in
\centerline{\includegraphics[width=\columnwidth]{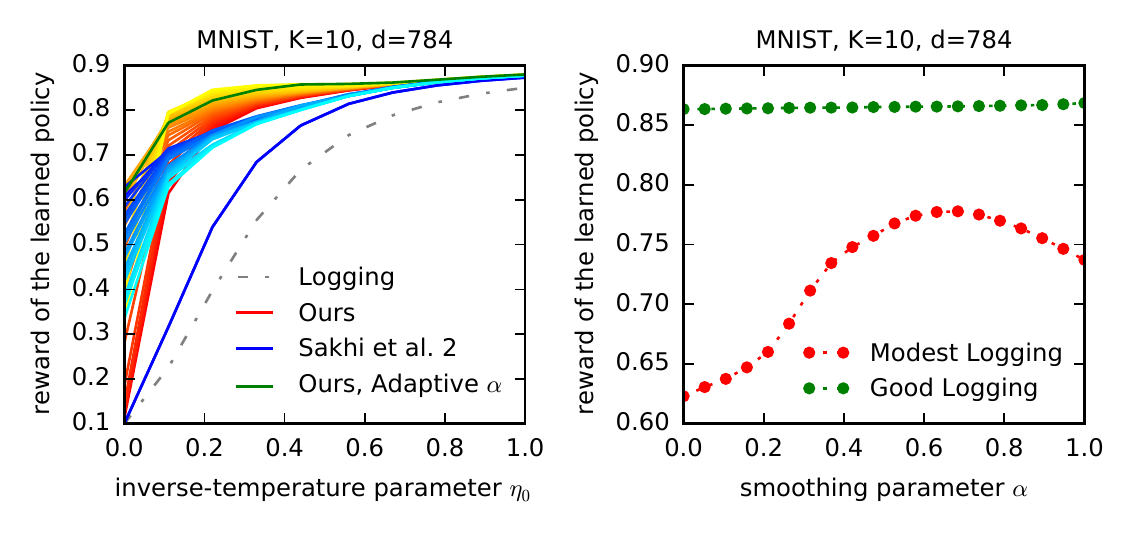}}
\vskip -0.2in
\caption{On the left-hand side is the reward of the learned policy with varying $\tau \in (0, 1)$, $\alpha \in (0, 1)$ and $\eta_0 \in [0, 1]$, and for an adaptive $\alpha$ using the procedure in \cref{subsec:data_dep_alpha} (green curve). The blue-to-cyan and red-to-yellow colors correspond to varying values of $\tau$ and $\alpha$, respectively. The lighter the color, the higher the value of $\tau$ or $\alpha$. The green curve corresponds to the reward of the learned policy with an adaptive and data-dependent $\alpha$ (\cref{subsec:data_dep_alpha}). On the right-hand side is the \emph{average} reward of the learned policies using our method across the modest and good logging groups, $\eta_0 \in [0, 0.5]$ (red) and $\eta_0 \in [0.5, 1]$ (green), respectively.}
\label{fig:varying_params}
\end{center}
\vskip -0.3in
\end{figure}

\section{Conclusion}\label{sec:conclusion}

In this paper, we investigated a smooth regularization of IPS in the context of OPL. First, we highlighted the pitfalls of hard clipping and advocated for a soft regularization alternative, called exponential smoothing. Moreover, we addressed some fundamental theoretical limitations of existing OPL approaches. Those limitations include the use of one-sided inequalities instead of two-sided ones, the use of learning principles and the use of evaluation bounds in OPL. Building upon this, we successfully derived a \emph{tractable two-sided} PAC-Bayes \emph{generalization bound} for our estimator, which \emph{we directly optimize}. We demonstrated, both theoretically through our bias-variance trade-off analysis in \eqref{eq:alpha_bias_variance} and our bound in \cref{thm:main_result}, and empirically, that this smooth regularization may be critical in some situations. In contrast with all prior works, our bound also applies to the standard IPS. This allowed us to also show that in some other cases, slight to no correction of IPS is needed in OPL.

\newpage

\bibliography{example_paper}
\bibliographystyle{icml2023}

\newpage
\appendix
\onecolumn

\section*{Organization of the Supplementary Material}
The supplementary material is organized as follows.
\begin{itemize}[topsep=0pt,itemsep=0pt]
    \item In \textbf{\cref{sec:related_work}}, we  give a detailed comparative
review of the literature of OPE and OPL.
\item In \textbf{\cref{proofs:ope}}, we provide some results and proofs for the bias and variance trade-off for our estimators. 
\item In \textbf{\cref{proofs:opl}}, we prove \cref{thm:main_result}. We also provide the proofs for all the claims made in \cref{sec:opl}.
\item In \textbf{\cref{app:all_experiments}}, we present in detail our experimental setup for reproducibility. This appendix also includes additional experiments.
\end{itemize}

\section{Related Work}\label{sec:related_work}

A contextual bandit \citep{lattimore19bandit} is a popular and practical framework for online learning to act under uncertainty \citep{li10contextual,chu11contextual}. In practice, the action space is large and short-term gains are important. Thus the agent should be \emph{risk-averse} which goes against the core principle of online algorithms that seek to explore the action space for the sake of long-term gains \citep{auer02finitetime,thompson33likelihood,russo18tutorial}. Although some practical algorithms have been proposed to efficiently explore the action space of a contextual bandit \citep{zong2016cascading,hong2022deep,zhu2022contextual,aoualimixed}. A clear need remains for an offline procedure that allows optimizing decision-making using offline data. Fortunately, we have access to logged data about previous interactions. The agent can leverage such data to learn an improved policy \emph{offline} \citep{swaminathan2015batch,london2019bayesian,sakhi2022pac} and consequently enhance the performance of the current system. In this work, we are concerned with this offline, or \emph{off-policy}, formulation of contextual bandits \citep{dudik2011doubly,dudik2012sample,dudik14doubly,wang2017optimal,farajtabar2018more}. Before learning an improved policy, an important intermediary step is to estimate the performance of policies using logged data, as if they were evaluated online. This task is referred to as \emph{off-policy evaluation (OPE)} \citep{dudik2011doubly}. After that, the resulting estimator is optimized to approximate the optimal policy, and this is referred to as \emph{off-policy learning (OPL)} \citep{swaminathan2015batch}. Next, we review both OPE and OPL approaches.

\subsection{Off-Policy Evaluation}
\label{sec:related_work_OPE}

Off-policy evaluation in contextual bandits has seen a lot of interest these recent years \citep{dudik2011doubly,dudik2012sample,dudik14doubly,wang2017optimal,farajtabar2018more,su2019cab,su2020doubly,kallus2021optimal,metelli2021subgaussian,kuzborskij2021confident,saito2022off,sakhi2020blob,jeunen2021pessimistic}. We can distinguish between three main families of approaches in the literature. First, \emph{direct method (DM)} \citep{jeunen2021pessimistic} learns a model that approximates the expected cost and then uses it to estimate the performance of evaluated policies. Unfortunately, DM can suffer from modeling bias and misspecification. Thus DM is often designed for specific use cases, in particular large-scale recommender systems \citep{sakhi2020blob,jeunen2021pessimistic,aouali2021combining,aouali2022probabilistic}. Second, \emph{inverse propensity scoring (IPS)} \citep{horvitz1952generalization,dudik2012sample} estimates the cost of the evaluated policies by removing the preference bias of the logging policy in logged data. Under the assumption that the evaluation policy is absolutely continuous with respect to the logging policy, IPS is unbiased, but it can suffer high variance. Note that it can also be highly biased when such an assumption is violated \citep{sachdeva2020off}. The variance issue is acknowledged and some fixes were proposed. For instance, clipping the importance weights \citep{ionides2008truncated,swaminathan2015batch}, self normalizing them \citep{swaminathan2015self}, etc. (see \citet{gilotte2018offline} for a survey). Third, \emph{doubly robust (DR)}  \citep{robins1995semiparametric,bang2005doubly,dudik2011doubly,dudik14doubly,farajtabar2018more} is a combination of DM and IPS. Here a model of the expected cost is used as a control variate for IPS to reduce the variance. Finally, the accuracy of an estimator $\hat{R}_{n}(\pi)$ in OPE is assessed using the mean squared error (MSE) defined as 
\begin{align*}
    {\rm MSE}(\hat{R}_n(\pi)) 
    &= \mathbb{E}\big[\big( \hat{R}_n(\pi) - R(\pi) \big)^2\big]\,\\
    &=  \mathbb{B}(\hat{R}_n(\pi))^2 + \mathbb{V}\big[\hat{R}_n(\pi)\big]\,,
\end{align*}
where $\mathbb{B}(\hat{R}_n(\pi)) = \mathbb{E}_{\cD_n}[\hat{R}_n(\pi)] - R(\pi)$  and $\mathbb{V}[\hat{R}_n(\pi)] = \mathbb{E}_{\cD_n}[(\hat{R}_n(\pi)-\mathbb{E}_{\cD_n}[\hat{R}_n(\pi)])^2]$ are respectively the bias and the variance of the estimator. It may be relevant to note that \citet{metelli2021subgaussian} argued that high-probability concentration rates should be preferred over the MSE to evaluate OPE estimators as they provide non-asymptotic guarantees. In this work, we highlighted the effect of $\alpha$ and $\beta$ in OPE following the common methodology of using the MSE as a performance metric. However, we also derived two-sided high-probability generalization bounds that attest to the quality of our estimator.

\subsection{Off-Policy Learning}
\label{sec:opl_learning_principles}

Previous works focused on deriving learning principles inspired by generalization bounds. First, \citet{swaminathan2015batch} derived a generalization bound for the \texttt{IPS-min} estimator in \eqref{eq:clip_ips_policy_value} of the form 
\begin{align}\label{eq:svp_bound}
    R(\pi) \leq \tilde{R}^M_n(\pi) + \mathcal{O}\left(\sqrt{\frac{\hat{V}_n(\pi) \mathcal{C}_n(\Pi, \delta)}{n}}+M\frac{ \mathcal{C}_n(\Pi, \delta)}{n}\right)\,,
\end{align}
where $\mathcal{C}_n(\Pi, \delta)$ is the complexity measure of the class of learning policies $\Pi$ while $\hat{V}_n(\pi)$ is the empirical variance of the estimator on the logged data $\cD_n$. The term $\mathcal{C}_n(\Pi, \delta)$ is not necessarily tractable. Thus the generalization bound above was only used to inspire the following learning principle
\begin{align}\label{eq:svp_principle}
\min_{\mu} \tilde{R}^M_n(\pi_\mu) + \lambda\sqrt{\frac{\hat{V}_n(\pi_\mu)}{n}}\,,
\end{align}
where $\lambda$ is a tunable hyper-parameter. This learning principle favors policies that simultaneously enjoy low estimated cost and empirical variance. \citet{faury2020distributionally} generalized their work using distributional robustness optimization while \citet{zenati2020counterfactual} generalized it to continuous action spaces. The latter also proposed a soft clipping scheme but they derived a generalization bound similar to the one in \citet{swaminathan2015batch}. Hence they also used the learning principle in \eqref{eq:svp_principle}. Our paper improves upon these work in different ways. First, \eqref{eq:svp_bound} has a multiplicative dependency on $M$. Therefore, it is not applicable to standard IPS recovered for $M \rightarrow \infty$. In contrast, our bound in \cref{thm:main_result} does not have a similar dependency on $\alpha$ and thus it also provides generalization guarantees for standard IPS without assuming that the importance weights are bounded. Second, the complexity measure $\mathcal{C}_n(\Pi, \delta)$ is often hard to compute while our bound is tractable and the KL terms can be computed or bounded in closed-form for Gaussian and mixed-logit policies. Third, our bound is differentiable and scalable while the learning principle in \eqref{eq:svp_principle} requires additional care in optimization \citep{swaminathan2015batch}. Fourth, it is challenging to tune $\lambda$ in \eqref{eq:svp_principle} using a procedure that is aligned with online metrics. Finally, we follow the theoretically grounded approach where we optimize our bound directly instead of using a learning principle. This direct optimization of the bound does not require any additional hyper-parameters tuning.

Recently, \citet{london2019bayesian} elegantly made the connection between PAC-Bayes theory and OPL. As a result, they derived a generalization bound for \texttt{IPS-max} in \eqref{eq:clip_ips_policy_value} which roughly has the following form 
    \begin{align}\label{eq:l2_bound}
    R(\pi_\mu) \leq \hat{R}^\tau_n(\pi_\mu)  + \mathcal{O}\left(\sqrt{\frac{(\hat{R}^\tau_n(\pi_\mu) + \frac{1}{\tau}) \norm{\mu -\mu_0}^2}{\tau n}}+\frac{ \norm{\mu -\mu_0}^2}{\tau n}\right)\,.
\end{align}
Again, this bound was used to inspire a novel learning principle in the form 
\begin{align}\label{eq:l2_principle}
\min_{\mu} \hat{R}^\tau_n(\pi_\mu) + \lambda \norm{\mu - \mu_0}^2\,,
\end{align}
where $\lambda$ is a tunable hyper-parameter and $\mu_0 \in \real^{dK}$ is the parameter of the logging policy. This principle favors policies with low estimated cost and whose parameter is not far from that of the logging policy in terms of $L_2$ distance. While the bound of \citet{london2019bayesian} is tractable, it still has a multiplicative dependency on $1/\tau$. This makes it inapplicable to standard IPS recovered for $\tau=0$. It is also not suitable for stochastic first-order optimization \citep{robbins1951stochastic} since the data-dependent quantity $ \hat{R}^\tau_n(\pi_\mu)$ is inside a square root. Moreover, optimizing directly their bound leads to minimal improvements over the logging policy in practice. In their work, they used the learning principle in \eqref{eq:l2_principle} instead which suffers the same issues that we discussed before for \citet{swaminathan2015batch}, except that it is scalable. Recently, \citet{sakhi2022pac} derived novel generalization bounds for a doubly robust version of the \texttt{IPS-max} estimator in \eqref{eq:clip_ips_policy_value}. \citet{sakhi2022pac} optimized the theoretical bound directly instead of using some form of learning principle and they showed favorable performance over existing methods. Unfortunately, their bounds have the same multiplicative dependency on $1/\tau$ which makes them vacuous for small values of $\tau$ and inapplicable to standard IPS. Moreover, we derive two-sided generalization bounds while all these works only derived one-sided generalization bounds. Unfortunately, the latter does not provide any guarantees on the expected performance of the learned policy. Also, we propose a different estimator to the clipped IPS traditionally used for OPL and we demonstrate empirically that it has better performance.

\section{Bias and Variance Trade-Off}\label{proofs:ope}
In this section, we provide additional results on how  $\beta$ and $\alpha$ control the bias and variance of $\tilde{R}_n^\beta(\cdot)$ and $\hat{R}_n^\alpha(\cdot)$, respectively. Precisely, in \cref{prop:beta_bias_variance,prop:alpha_bias_variance} we upper bound the absolute bias and variance of $\hat{R}_n^\alpha(\cdot)$ and $\tilde{R}_n^\beta(\cdot)$, respectively.

\subsection{Bias and variance of $\texttt{IPS-}\alpha$}\label{app:alpha_bias_variance}

The following proposition states the bias-variance trade-off for $\hat{R}_n^\alpha(\cdot)$. 

\begin{proposition}[Bias and variance of $\texttt{IPS-}\alpha$]\label{prop:alpha_bias_variance} Let $\alpha \in [0, 1]$, the following holds for any evaluation policy $\pi \in \Pi$ that is absolutely continuous with respect to $\pi_0$
\begin{align*}
    |\mathbb{B}(\hat{R}_n^\alpha(\pi))|  & \leq \E{x \sim \nu, a \sim \pi(\cdot | x)}{1 - \pi_0(a | x)^{1-\alpha}} \,,\nonumber\\
\mathbb{V}\left[\hat{R}_n^\alpha(\pi)\right] & \leq \frac{1}{n} \mathbb{E}_{x \sim \nu, a \sim \pi(\cdot | x)}\big[ \frac{\pi(a | x)}{\pi_0(a | x)^{2\alpha-1}} \big] \,.
\end{align*}
\end{proposition}

\begin{proof}
We first bound the bias as 
\begin{align*}
 \mathbb{B}(\hat{R}_n^\alpha(\pi))  & = \E{}{\hat{R}_n^\alpha(\pi)} - R(\pi)\,,\\
 &= \frac{1}{n} \sum_{i=1}^n \E{x_i \sim \nu, a_i \sim \pi_0(\cdot | x_i), c_i \sim p(\cdot | x_i, a_i)}{c_i \frac{\pi(a_i | x_i)}{\pi_0(a_i | x_i)^\alpha}} - R(\pi)\,,\\
    & \stackrel{(i)}{=} \E{(x, a, c) \sim \mu_{\pi_0}}{c \frac{\pi(a | x)}{\pi_0(a | x)^\alpha}} - R(\pi)\,,\\
    & = \E{x \sim \nu}{\sum_{a \in \cA}c(x, a) \frac{\pi(a | x)}{\pi_0(a | x)^{\alpha-1}}} - \E{x \sim \nu}{\sum_{a \in \cA}c(x, a)\pi(a | x)}\,,\\
    & = \E{x \sim \nu}{\sum_{a \in \cA}c(x, a) \pi(a | x) ( \pi_0(a | x)^{1-\alpha}-1)}\,,\\
    & = \E{x \sim \nu, a \sim \pi(\cdot | x)}{c(x, a) ( \pi_0(a | x)^{1-\alpha}-1)}\,,\\
\end{align*}
where $(i)$ follows from the i.i.d. assumption. Since $\pi_0(a | x)^{1-\alpha} \leq 1$ for any $x \in \cX$ and $a \in \cA$, we have that
\begin{align*}
 |\mathbb{B}(\hat{R}_n^\alpha(\pi))|     & \leq \E{x \sim \nu, a \sim \pi(\cdot | x)}{ |c(x, a)||\pi_0(a | x)^{1-\alpha}-1|}\,,\\
&\leq \E{x \sim \nu, a \sim \pi(\cdot | x)}{ 1-\pi_0(a | x)^{1-\alpha}}\,.
\end{align*}

The variance is bounded as 
\begin{align*}
\mathbb{V}\left[\tilde{R}_n^\beta(\pi)\right] &= \frac{1}{n^2} \sum_{i=1}^n \mathbb{V}_{x_i \sim \nu, a_i \sim \pi_0(\cdot | x_i), c_i \sim p(\cdot | x_i, a_i)}\Big[c_i \frac{\pi(a_i | x_i)}{\pi_0(a_i | x_i)^\alpha}\Big] \,,\\
   &= \frac{1}{n} \mathbb{V}_{(x, a, c) \sim \mu_{\pi_0}}\Big[c \frac{\pi(a | x)}{\pi_0(a | x)^\alpha}\Big] \,,\\
   &\leq\frac{1}{n} \mathbb{E}_{(x, a, c) \sim \mu_{\pi_0}}\Big[c^2 \frac{\pi(a | x)^2}{\pi_0(a | x)^{2\alpha}}\Big] \,,\\
   &\leq\frac{1}{n} \mathbb{E}_{x \sim \nu, a \sim \pi_0(\cdot | x)}\Big[\frac{\pi(a | x)^2}{\pi_0(a | x)^{2\alpha}}\Big] \,,\\
    &=\frac{1}{n} \mathbb{E}_{x \sim \nu}\Big[ \sum_{a \in \cA} \frac{\pi(a | x)^{2}}{\pi_0(a | x)^{2\alpha-1}}\Big] \,,\\
    &=\frac{1}{n} \mathbb{E}_{x \sim \nu, a \sim \pi(\cdot | x)}\Big[ \frac{\pi(a | x)}{\pi_0(a | x)^{2\alpha-1}}\Big] \,.
\end{align*}
\end{proof}

\subsection{Bias and variance of $\texttt{IPS-}\beta$}\label{app:beta_bias_variance}

The following proposition states the bias-variance trade-off for $\tilde{R}_n^\beta(\cdot)$. 

\begin{proposition}[Bias and variance of $\texttt{IPS-}\beta$]\label{prop:beta_bias_variance} Let $\beta \in [0, 1]$, the following holds for any evaluation policy $\pi \in \Pi$ that is absolutely continuous with respect to $\pi_0$
\begin{align*}
    | \mathbb{B}(\tilde{R}_n^\beta(\pi))| & \leq \mathbb{E}_{x \sim \nu, a \sim \pi(\cdot | x)}\big[\big|\big(\frac{\pi(a | x)}{\pi_0(a | x)}\big)^{\beta-1}-1\big|\big] \,,\\
\mathbb{V}\left[\tilde{R}_n^\beta(\pi)\right] & \leq \frac{1}{n} \mathbb{E}_{x \sim \nu, a \sim \pi(\cdot | x)}\big[\big(\frac{\pi(a | x)}{\pi_0(a | x)}\big)^{2\beta-1}\big] \,.
\end{align*}
\end{proposition}

\begin{proof}
We first bound the bias as 
\begin{align*}
 \mathbb{B}(\tilde{R}_n^\beta(\pi))  & = \E{}{\tilde{R}_n^\beta(\pi)} - R(\pi)\,\\ &= \frac{1}{n} \sum_{i=1}^n \E{x_i \sim \nu, a_i \sim \pi_0(\cdot | x_i), c_i \sim p(\cdot | x_i, a_i)}{c_i \Big(\frac{\pi(a_i | x_i)}{\pi_0(a_i | x_i)}\Big)^\beta} - R(\pi)\,,\\
    & \stackrel{(i)}{=} \E{(x, a, c) \sim \mu_{\pi_0}}{c \Big(\frac{\pi(a | x)}{\pi_0(a | x)}\Big)^\beta} - R(\pi)\,,\\
    & = \E{x \sim \nu, a \sim \pi_0(\cdot | x)}{c(x, a) \Big(\frac{\pi(a | x)}{\pi_0(a | x)}\Big)^\beta} - \E{x \sim \nu, a \sim \pi(\cdot | x)}{c(x, a)}\,,\\
    & = \E{x \sim \nu}{\sum_{a \in \cA}c(x, a) \frac{\pi(a | x)^\beta}{\pi_0(a | x)^{\beta-1}}} - \E{x \sim \nu}{\sum_{a \in \cA}c(x, a)\pi(a | x)}\,,\\
    & = \E{x \sim \nu}{\sum_{a \in \cA}c(x, a) \pi(a | x) \Big(\Big(\frac{\pi_0(a | x)}{\pi(a | x)}\Big)^{1-\beta}-1\Big)}\,,
\end{align*}
where $(i)$ follows from the i.i.d. assumption. It follows that 
\begin{align*}
 |\mathbb{B}(\tilde{R}_n^\beta(\pi))| & \leq \E{x \sim \nu}{\sum_{a \in \cA}|c(x, a)| \pi(a | x) \Big|\Big(\frac{\pi_0(a | x)}{\pi(a | x)}\Big)^{1-\beta}-1\Big|}\,,\\ 
    & \leq \E{x \sim \nu}{\sum_{a \in \cA} \pi(a | x) \Big|\Big(\frac{\pi_0(a | x)}{\pi(a | x)}\Big)^{1-\beta}-1\Big|}\,,\\
    &= \E{x \sim \nu, a \sim \pi(\cdot | x)}{ \Big|\Big(\frac{\pi_0(a | x)}{\pi(a | x)}\Big)^{1-\beta}-1\Big|}\,.
\end{align*}

The variance is bounded as 
\begin{align*}
\mathbb{V}\left[\tilde{R}_n^\beta(\pi)\right] &= \frac{1}{n^2} \sum_{i=1}^n \mathbb{V}_{x_i \sim \nu, a_i \sim \pi_0(\cdot | x_i), c_i \sim p(\cdot | x_i, a_i)}\Big[c_i \Big(\frac{\pi(a_i | x_i)}{\pi_0(a_i | x_i)}\Big)^\beta\Big]\,,\\
&= \frac{1}{n} \mathbb{V}_{(x, a, c) \sim \mu_{\pi_0}}\Big[c \Big(\frac{\pi(a | x)}{\pi_0(a | x)}\Big)^\beta\Big] \,,\\
   &\leq\frac{1}{n} \mathbb{E}_{(x, a, c) \sim \mu_{\pi_0}}\Big[c^2 \Big(\frac{\pi(a | x)}{\pi_0(a | x)}\Big)^{2\beta}\Big] \,,\\
   &\leq\frac{1}{n} \mathbb{E}_{x \sim \nu, a \sim \pi_0(\cdot | x)}\Big[\Big(\frac{\pi(a | x)}{\pi_0(a | x)}\Big)^{2\beta}\Big] \,,\\
    &=\frac{1}{n} \mathbb{E}_{x \sim \nu}\Big[ \sum_{a \in \cA} \frac{\pi(a | x)^{2\beta}}{\pi_0(a | x)^{2\beta-1}}\Big]\,,\\ 
    &=\frac{1}{n} \mathbb{E}_{x \sim \nu, a \sim \pi(\cdot | x)}\Big[ \frac{\pi(a | x)^{2\beta-1}}{\pi_0(a | x)^{2\beta-1}}\Big] \,.
\end{align*}
\end{proof}

\subsection{Discussion}

Here we show using \cref{prop:alpha_bias_variance,prop:beta_bias_variance} how $\alpha$ 
and $\beta$ trade the bias and variance of $ \hat{R}_n^\alpha(\pi)$ and $ \tilde{R}_n^\beta(\pi)$, respectively. Let us start with $ \hat{R}_n^\alpha(\pi)$, from \cref{prop:alpha_bias_variance}, the bound on the bias is minimized in $\alpha=1$; in which case it is equal to 0. In contrast, the bound on the variance is minimized in $\alpha =0$; in which case the variance is bounded by $1/n$. Let $\alpha_*$ be the minimizer of the corresponding bound of the MSE
\begin{align*}
    \alpha_* = {\rm argmin}_{\alpha \in [0, 1]} & \E{x \sim \nu, a \sim \pi(\cdot | x)}{1 - \pi_0(a | x)^{1-\alpha}}^2 + \mathbb{E}_{x \sim \nu, a \sim \pi(\cdot | x)}\big[ \frac{\pi(a | x)}{\pi_0(a | x)^{2\alpha-1}} \big] / n \,.
\end{align*}
We observe that when the variance is small or $n$ is large enough such that $\mathbb{E}_{x \sim \nu, a \sim \pi(\cdot | x)}\big[ \frac{\pi(a | x)}{\pi_0(a | x)^{2\alpha-1}} \big] /n \rightarrow 0$, then we have that $\alpha_* \rightarrow 1$. Thus it is better to use the standard IPS in this case. Otherwise, we have $\alpha_* \rightarrow 0$ and this is when regularization helps; basically when we have few samples or when the evaluation policy induces high variance. This demonstrates that the choice of $\alpha$ matters as it trades the bias and variance of $\hat{R}_n^\alpha$. 

Similarly, from \cref{prop:beta_bias_variance}, we define $\beta_*$ as 
\begin{align*}
    \beta_* = {\rm argmin}_{\beta \in [0, 1]} & \mathbb{E}_{x \sim \nu, a \sim \pi(\cdot | x)}\big[\big|\big(\frac{\pi(a | x)}{\pi_0(a | x)}\big)^{\beta-1}-1\big|\big] + \frac{1}{n} \mathbb{E}_{x \sim \nu, a \sim \pi(\cdot | x)}\big[\big(\frac{\pi(a | x)}{\pi_0(a | x)}\big)^{2\beta-1}\big] \,.
\end{align*}
Again, we observe that if $\mathbb{E}_{x \sim \nu, a \sim \pi(\cdot | x)}\big[\big(\frac{\pi(a | x)}{\pi_0(a | x)}\big)^{2\beta-1}\big]/n \rightarrow 0$, then $\beta_* \rightarrow 1$; in which case it is better to use standard IPS. Otherwise, we have $\beta_* \rightarrow 0$ to regularize the importance weights.

\section{Proofs for Off-Policy Learning}\label{proofs:opl}
In this section, we provide the complete proofs for our OPL results in \cref{sec:opl}. We start with proving \cref{thm:main_result} in \cref{proof:main_thm_proof}. We then state the extension of \cref{thm:main_result} along with its proof in \cref{app:thm_extension}. After that, in \cref{proof:direct_application}, we provide the proof for \cref{prop:direct_application}.  Finally, in \cref{proof:practice_theory}, we discuss in detail and prove our claims regarding the number of samples needed so that the performance of the learned policy is close to that of the optimal policy.

\subsection{Proof of \cref{thm:main_result}}\label{proof:main_thm_proof}
In this section, we prove \cref{thm:main_result}.
\begin{proof}
    
First, we decompose the difference $R(\pi_{\mathbb{Q}})-\hat{R}_n^\alpha(\pi_{\mathbb{Q}})$ as
\begin{align*}
    R(\pi_{\mathbb{Q}})-\hat{R}_n^\alpha(\pi_{\mathbb{Q}}) =  \underbrace{R(\pi_{\mathbb{Q}}) - \frac{1}{n}\sum_{i=1}^n R(\pi_{\mathbb{Q}} | x_i)}_{I_1} + \underbrace{\frac{1}{n} \sum_{i=1}^n R(\pi_{\mathbb{Q}} | x_i) - \frac{1}{n}\sum_{i=1}^n R^\alpha(\pi_{\mathbb{Q}} | x_i)}_{I_2} + \underbrace{\frac{1}{n}\sum_{i=1}^n R^\alpha(\pi_{\mathbb{Q}} | x_i) - \hat{R}_n^\alpha(\pi_{\mathbb{Q}})}_{I_3}\,,
\end{align*}
where 
\begin{align*}
    & R(\pi_{\mathbb{Q}}) = \E{x \sim \nu\,, a \sim \pi_{\mathbb{Q}}(\cdot | x)}{c(x, a)}\,,\\
    & R(\pi_{\mathbb{Q}} | x_i)    = \E{a \sim \pi_{\mathbb{Q}}(\cdot | x_i)}{c(x_i, a)}\,,\\ 
   & R^\alpha(\pi_{\mathbb{Q}} | x_i) = \E{a \sim \pi_0(\cdot | x_i)}{\frac{\pi_{\mathbb{Q}}(a | x_i)}{\pi_0(a | x_i)^\alpha}c(x_i, a)}\,, \\
   & \hat{R}_n^\alpha(\pi) = \frac{1}{n} \sum_{i=1}^n \frac{\pi(a_i | x_i)}{\pi_0(a_i | x_i)^\alpha}c_i \,.
\end{align*}
Our goal is to bound $|R(\pi_{\mathbb{Q}})-\hat{R}_n^\alpha(\pi_{\mathbb{Q}}) |$ and thus we need to bound $|I_1| + |I_2| + |I_3| $. We start with $|I_1|$, \citet[Theorem 3.3]{alquier2021user} yields that following inequality holds with probability at least $1 -\delta/2$ for any distribution $\mathbb{Q}$ on $\mathcal{H}$ 
\begin{align}\label{eq:app_i1}
    |I_1| \leq \sqrt{ \frac{D_{\mathrm{KL}}(\mathbb{Q} \| \mathbb{P})+\log \frac{4\sqrt{n}}{\delta}}{2n}} \,.
\end{align}
Moreover, $|I_2|$ can be bounded by decomposing it as 
\begin{align*}
    |I_2|&=\left|\frac{1}{n} \sum_{i=1}^{n}  \E{a \sim \pi_{\mathbb{Q}}(\cdot | x_i)}{c(x_i, a)}-\frac{1}{n} \sum_{i=1}^{n}\mathbb{E}_{a \sim \pi_0\left(\cdot | x_i\right)}\left[ \frac{\pi_{\mathbb{Q}}(a | x_i)}{\pi_0^\alpha(a | x_i)}c(x_i, a) \right]\right| \\
    &=\left|\frac{1}{n} \sum_{i=1}^{n} \sum_{a \in \cA} \pi_{\mathbb{Q}}(a | x_i) c(x_i, a) -\pi_0(a | x_i)\frac{\pi_{\mathbb{Q}}(a | x_i)}{\pi_0^\alpha(a | x_i)}c(x_i, a)\right| \\
    &=\left|\frac{1}{n} \sum_{i=1}^{n} \sum_{a \in \cA} \Big(\pi_{\mathbb{Q}}(a | x_i) - \frac{\pi_{\mathbb{Q}}(a | x_i)}{\pi_0^{\alpha-1}(a | x_i)}\Big)c(x_i, a)\right| \\
    &=\left|\frac{1}{n} \sum_{i=1}^{n} \sum_{a \in \cA} \Big(1 - \pi_0^{1-\alpha}(a | x_i)\Big)\pi_{\mathbb{Q}}(a | x_i)c(x_i, a)\right|\,,\\
    &\leq\frac{1}{n} \sum_{i=1}^{n} \sum_{a \in \cA} \left|1 - \pi_0^{1-\alpha}(a | x_i)\right|\pi_{\mathbb{Q}}(a | x_i)\left|c(x_i, a)\right|\,.
\end{align*}
But $1 - \pi_0^{1-\alpha}(a | x) \geq 0$ and $|c(x, a)| \leq 1$ for any $a \in \cA$ and $x \in \cX$. Thus 
\begin{align}\label{eq:app_i2}
   |I_2| \leq \frac{1}{n} \sum_{i=1}^{n} \E{a \sim \pi_{\mathbb{Q}}(\cdot | x_i)}{1 - \pi_0^{1-\alpha}(a | x_i)}\,.
\end{align}
Finally, we need to bound the main term $|I_3|$. To achieve this, we borrow the following technical lemma from \citet{haddouche2022pac}. It is slightly different from the one in \citet{haddouche2022pac}; their result holds for any $n \geq 1$ while we state a simpler version where $n$ is fixed in advance.

\begin{lemma}\label{lemma:app_maxime}
Let $\mathcal{Z}$ be an instance space and let $S_n=\left(z_i\right)_{i \in [n]}$ be an $n$-sized dataset for some $n \geq 1$. Let $\left(\mathcal{F}_i\right)_{i \in\{0\} \cup [n] }$ be a filtration adapted to $S_n$. Also, let $\mathcal{H}$ be a hypothesis space and $(f_i\left(S_i, h\right))_{i \in [n]}$ be a martingale difference sequence for any $h \in \mathcal{H}$, that is for any $i \in [n]$, and $h \in \mathcal{H}\,,$ we have that $ \mathbb{E}\left[f_i\left(S_i, h\right) | \mathcal{F}_{i-1}\right]=0$. Moreover, for any $h \in \mathcal{H}$, let $M_n(h)=\sum_{i=1}^n f_i\left(S_i, h\right)$. Then for any fixed prior, $\mathbb{P}$, on $\mathcal{H}$, any $\lambda>0$, the following holds with probability $1-\delta$ over the sample $S_n$, simultaneously for any $\mathbb{Q}$, on $\mathcal{H}$
\begin{align*}
    \left|\E{h \sim \mathbb{Q}}{M_n(h)}\right| \leq \frac{D_{\mathrm{KL}}(\mathbb{Q} \| \mathbb{P})+\log (2 / \delta)}{\lambda}+\frac{\lambda}{2}\left(\E{h \sim \mathbb{Q}}{\langle M\rangle_n(h) + [M]_n(h)}\right)\,,
\end{align*}
where $ \langle M\rangle_n(h)=\sum_{i=1}^n \mathbb{E}\left[f_i\left(S_i, h\right)^2 | \mathcal{F}_{i-1}\right]$ and $[M]_n(h)=\sum_{i=1}^n f_i\left(S_i, h\right)^2$.
\end{lemma}

To apply \cref{lemma:app_maxime}, we need to construct an adequate martingale difference sequence $(f_i(S_i, h))_{i \in [n]}$ for $h \in \mathcal{H}$ that allows us to retrieve $|I_3|$. To achieve this, we define $S_n = (a_i)_{i \in [n]}$ as the set of $n$ taken actions. Also, we let $(\mathcal{F}_i)_{i \in \{0\}\cup[n]}$ be a filtration adapted to $S_n$. For $h \in \mathcal{H}$, we define $f_i\left(S_i, h\right)$ as 
\begin{align*}
    f_i\left(S_i, h\right) = f_i\left(a_i, h\right) = \E{a \sim \pi_0(\cdot | x_i)}{\frac{\mathbb{I}_{\{h(x_i)=a\}}}{\pi_0(a | x_i)^\alpha}c(x_i, a)} - \frac{\mathbb{I}_{\{h(x_i)=a_i\}}}{\pi_0(a_i | x_i)^\alpha}c(x_i, a_i)\,.
\end{align*}
We stress that $f_i(S_i, h)$ only depends on the last action in $S_i$, $a_i$, and the predictor $h$. For this reason, we denote it by $f_i(a_i, h)$. The function $f_i$ is indexed by $i$ since it depends on the fixed $i$-th context, $x_i$. The context $x_i$ is fixed and thus randomness only comes from $a_i \sim \pi_0(\cdot | x_i)$. It follows that the expectations are under $a_i \sim \pi_0(\cdot | x_i)$. First, we have that $\mathbb{E}\left[f_i\left(a_i, h\right) | \mathcal{F}_{i-1}\right]= 0$ for any $i \in [n]\,, h \in \mathcal{H}$. This follows from 
\begin{align*}
    \mathbb{E}\left[f_i\left(a_i, h\right) | \mathcal{F}_{i-1}\right]  &= \mathbb{E}_{a_i \sim \pi_0(\cdot | x_i)}\left[f_i\left(a_i, h\right) \Big| a_1, \ldots, a_{i-1}\right] \,,\\
     &=  \mathbb{E}_{a_i \sim \pi_0(\cdot | x_i)}\left[\E{a \sim \pi_0(\cdot | x_i)}{\frac{\mathbb{I}_{\{h(x_i)=a\}}}{\pi_0(a | x_i)^\alpha}c(x_i, a)} - \frac{\mathbb{I}_{\{h(x_i)=a_i\}}}{\pi_0(a_i | x_i)^\alpha}c(x_i, a_i) \Big| a_1, \ldots, a_{i-1}\right]\,,\\
     &\stackrel{(i)}{=}  \E{a \sim \pi_0(\cdot | x_i)}{\frac{\mathbb{I}_{\{h(x_i)=a\}}}{\pi_0(a | x_i)^\alpha}c(x_i, a)} - \mathbb{E}_{a_i \sim \pi_0(\cdot | x_i)}\left[\frac{\mathbb{I}_{\{h(x_i)=a_i\}}}{\pi_0(a_i | x_i)^\alpha}c(x_i, a_i) \Big| a_1, \ldots, a_{i-1}\right]\,.
\end{align*}
In $(i)$ we use the fact that given $x_i$, $ \E{a \sim \pi_0(\cdot | x_i)}{\frac{\mathbb{I}_{\{h(x_i)=a\}}}{\pi_0(a | x_i)^\alpha}c(x_i, a)}$ is deterministic. Now $a_i$ does not depend on $a_1, \ldots, a_{i-1}$ since logged data is i.d.d. Hence 
\begin{align*}
    \mathbb{E}_{a_i \sim \pi_0(\cdot | x_i)}\left[\frac{\mathbb{I}_{\{h(x_i)=a_i\}}}{\pi_0(a_i | x_i)^\alpha}c(x_i, a_i) \Big| a_1, \ldots, a_{i-1}\right] &= \mathbb{E}_{a_i \sim \pi_0(\cdot | x_i)}\left[\frac{\mathbb{I}_{\{h(x_i)=a_i\}}}{\pi_0(a_i | x_i)^\alpha}c(x_i, a_i)\right]\,,\\ &= \mathbb{E}_{a \sim \pi_0(\cdot | x_i)}\left[\frac{\mathbb{I}_{\{h(x_i)=a\}}}{\pi_0(a | x_i)^\alpha}c(x_i, a)\right]\,.
\end{align*}
It follows that 
\begin{align*}
     \mathbb{E}\left[f_i\left(a_i, h\right) | \mathcal{F}_{i-1}\right] 
     &=  \E{a \sim \pi_0(\cdot | x_i)}{\frac{\mathbb{I}_{\{h(x_i)=a\}}}{\pi_0(a | x_i)^\alpha}c(x_i, a)} - \mathbb{E}_{a_i \sim \pi_0(\cdot | x_i)}\left[\frac{\mathbb{I}_{\{h(x_i)=a_i\}}}{\pi_0(a_i | x_i)^\alpha}c(x_i, a_i) \Big| a_1, \ldots, a_{i-1}\right]\,,\\
    &= \E{a \sim \pi_0(\cdot | x_i)}{\frac{\mathbb{I}_{\{h(x_i)=a\}}}{\pi_0(a | x_i)^\alpha}c(x_i, a)} - \mathbb{E}_{a \sim \pi_0(\cdot | x_i)}\left[\frac{\mathbb{I}_{\{h(x_i)=a\}}}{\pi_0(a | x_i)^\alpha}c(x_i, a)\right]\,,\\
    &= 0\,.
\end{align*}
Therefore, for any $h \in \mathcal{H}$, $(f_i(a_i, h))_{i \in [n]}$ is a martingale difference sequence. Hence we apply \cref{lemma:app_maxime} and obtain that the following inequality holds with probability at least $1-\delta/2$ for any $\mathbb{Q}$ on $\mathcal{H}$
\begin{align}\label{eq:app_proof_0}
    \left|\E{h \sim \mathbb{Q}}{M_n(h)}\right| \leq \frac{D_{\mathrm{KL}}(\mathbb{Q} \| \mathbb{P})+\log (4 / \delta)}{\lambda}+\frac{\lambda}{2}\left(\E{h \sim \mathbb{Q}}{\langle M\rangle_n(h) + [M]_n(h)}\right)\,,
\end{align}
where 
\begin{align*}
    M_n(h)&=\sum_{i=1}^n f_i\left(a_i, h\right)\,,\\
     \langle M\rangle_n(h)&=\sum_{i=1}^n \mathbb{E}\left[f_i\left(a_i, h\right)^2 | \mathcal{F}_{i-1}\right]\,,\\
     [M]_n(h)&=\sum_{i=1}^n f_i\left(a_i, h\right)^2
\end{align*}
Now these terms can be decomposed as 
\begin{align*}
    \E{h \sim \mathbb{Q}}{M_n(h)} &= \sum_{i=1}^n\E{h \sim \mathbb{Q}}{f_i\left(a_i, h\right)}\,,\\
    &=  \sum_{i=1}^n \E{h \sim \mathbb{Q}}{\E{a \sim \pi_0(\cdot | x_i)}{\frac{\mathbb{I}_{\{h(x_i)=a\}}}{\pi_0(a | x_i)^\alpha}c(x_i, a)} - \frac{\mathbb{I}_{\{h(x_i)=a_i\}}}{\pi_0(a_i | x_i)^\alpha}c(x_i, a_i)}\,,\\
    &\stackrel{(i)}{=} \sum_{i=1}^n \E{h \sim \mathbb{Q}}{\E{a \sim \pi_0(\cdot | x_i)}{\frac{\mathbb{I}_{\{h(x_i)=a\}}}{\pi_0(a | x_i)^\alpha}c(x_i, a)}} - \E{h \sim \mathbb{Q}}{\frac{\mathbb{I}_{\{h(x_i)=a_i\}}}{\pi_0(a_i | x_i)^\alpha}c(x_i, a_i)}\,,\\
    &\stackrel{(ii)}{=} \sum_{i=1}^n\E{a \sim \pi_0(\cdot | x_i)}{\frac{ \E{h \sim \mathbb{Q}}{\mathbb{I}_{\{h(x_i)=a\}}}}{\pi_0(a | x_i)^\alpha}c(x_i, a)} - \frac{\E{h \sim \mathbb{Q}}{\mathbb{I}_{\{h(x_i)=a_i\}}}}{\pi_0(a_i | x_i)^\alpha}c(x_i, a_i)\,,\\
    &\stackrel{(iii)}{=} \sum_{i=1}^n\E{a \sim \pi_0(\cdot | x_i)}{\frac{ \pi_{\mathbb{Q}}(a | x_i) }{\pi_0(a | x_i)^\alpha}c(x_i, a)} - \sum_{i=1}^n \frac{ \pi_{\mathbb{Q}}(a_i | x_i)}{\pi_0(a_i | x_i)^\alpha}c(x_i, a_i)\,,
\end{align*}
where we use the linearity of the expectation in both $(i) $  and $(ii)$. In $(iii)$, we use our definition of policies in \eqref{eq:pac_policies}. Therefore, we have that 
\begin{align}\label{eq:app_proof_1}
    \E{h \sim \mathbb{Q}}{M_n(h)}  &=\sum_{i=1}^n\E{a \sim \pi_0(\cdot | x_i)}{\frac{ \pi_{\mathbb{Q}}(a | x_i) }{\pi_0(a | x_i)^\alpha}c(x_i, a)} - \sum_{i=1}^n \frac{ \pi_{\mathbb{Q}}(a_i | x_i)}{\pi_0(a_i | x_i)^\alpha}c(x_i, a_i)\,,\nonumber\\
    & \stackrel{(i)}{=} \sum_{i=1}^n R^\alpha(\pi_{\mathbb{Q}} | x_i) - n\hat{R}_n^\alpha(\pi_{\mathbb{Q}})\,,\nonumber\\
    &= nI_3\,,
\end{align}
where we used the fact that $c_i = c(a_i, x_i)$ for any $i \in [n]$ in $(i)$. 

Now we focus on the terms $\langle M\rangle_n(h)$ and $ [M]_n(h)$. First, we have that 
\begin{align}\label{eq:app_proof_2}
    f_i\left(a_i, h\right)^2 &= \Big(\E{a \sim \pi_0(\cdot | x_i)}{\frac{\mathbb{I}_{\{h(x_i)=a\}}}{\pi_0(a | x_i)^\alpha}c(x_i, a)} - \frac{\mathbb{I}_{\{h(x_i)=a_i\}}}{\pi_0(a_i | x_i)^\alpha}c(x_i, a_i)\Big)^2\,,\\
    &= \E{a \sim \pi_0(\cdot | x_i)}{\frac{\mathbb{I}_{\{h(x_i)=a\}}}{\pi_0(a | x_i)^\alpha}c(x_i, a)}^2 + \Big(\frac{\mathbb{I}_{\{h(x_i)=a_i\}}}{\pi_0(a_i | x_i)^\alpha}c(x_i, a_i)\Big)^2 \nonumber \\ & \hspace{2cm} - 2   \E{a \sim \pi_0(\cdot | x_i)}{\frac{\mathbb{I}_{\{h(x_i)=a\}}}{\pi_0(a | x_i)^\alpha}c(x_i, a)} \frac{\mathbb{I}_{\{h(x_i)=a_i\}}}{\pi_0(a_i | x_i)^\alpha}c(x_i, a_i)\,,\nonumber\\
    &= \E{a \sim \pi_0(\cdot | x_i)}{\frac{\mathbb{I}_{\{h(x_i)=a\}}}{\pi_0(a | x_i)^\alpha}c(x_i, a)}^2 + \frac{\mathbb{I}_{\{h(x_i)=a_i\}}}{\pi_0(a_i | x_i)^{2\alpha}}c(x_i, a_i)^2 \nonumber \\ & \hspace{2cm} - 2   \E{a \sim \pi_0(\cdot | x_i)}{\frac{\mathbb{I}_{\{h(x_i)=a\}}}{\pi_0(a | x_i)^\alpha}c(x_i, a)} \frac{\mathbb{I}_{\{h(x_i)=a_i\}}}{\pi_0(a_i | x_i)^\alpha}c(x_i, a_i)\,.\nonumber
\end{align}
Moreover, $f_i\left(a_i, h\right)^2$ does not depend on $a_1, \ldots, a_{i-1}$. Thus
\begin{align*}
    \mathbb{E}\left[f_i\left(a_i, h\right)^2 | \mathcal{F}_{i-1}\right] = \mathbb{E}_{a_i \sim \pi_0(\cdot | x_i)}\left[f_i\left(a_i, h\right)^2 | \mathcal{F}_{i-1}\right] = \mathbb{E}_{a_i \sim \pi_0(\cdot | x_i)}\left[f_i\left(a_i, h\right)^2\right]= \mathbb{E}_{a \sim \pi_0(\cdot | x_i)}\left[f_i\left(a, h\right)^2\right]\,.
\end{align*}
Computing $\mathbb{E}_{a \sim \pi_0(\cdot | x_i)}\left[f_i\left(a, h\right)^2\right]$ using the decomposition in \eqref{eq:app_proof_2} yields
\begin{align}\label{eq:app_proof_3}
    \mathbb{E}\left[f_i\left(a_i, h\right)^2 | \mathcal{F}_{i-1}\right] &= \mathbb{E}_{a \sim \pi_0(\cdot | x_i)}\left[f_i\left(a, h\right)^2\right] \,,\nonumber\\
    &= -  \E{a \sim \pi_0(\cdot | x_i)}{\frac{\mathbb{I}_{\{h(x_i)=a\}}}{\pi_0(a | x_i)^\alpha}c(x_i, a)}^2
+ \E{a \sim \pi_0(\cdot | x_i)}{\frac{\mathbb{I}_{\{h(x_i)=a\}}}{\pi_0(a | x_i)^{2\alpha}}c(x_i, a)^2}
\end{align}
Combining \eqref{eq:app_proof_2} and \eqref{eq:app_proof_3} leads to
\begin{align}\label{eq:app_proof_4}
    \mathbb{E}\left[f_i\left(a_i, h\right)^2 | \mathcal{F}_{i-1}\right] + f_i\left(a_i, h\right)^2 &= \E{a \sim \pi_0(\cdot | x_i)}{\frac{\mathbb{I}_{\{h(x_i)=a\}}}{\pi_0(a | x_i)^{2\alpha}}c(x_i, a)^2}+ \frac{\mathbb{I}_{\{h(x_i)=a_i\}}}{\pi_0(a_i | x_i)^{2\alpha}}c(x_i, a_i)^2 \nonumber\\ &\hspace{2cm} - 2   \E{a \sim \pi_0(\cdot | x_i)}{\frac{\mathbb{I}_{\{h(x_i)=a\}}}{\pi_0(a | x_i)^\alpha}c(x_i, a)} \frac{\mathbb{I}_{\{h(x_i)=a_i\}}}{\pi_0(a_i | x_i)^\alpha}c(x_i, a_i)\,,\nonumber\\
    & \stackrel{(i)}{\leq} \E{a \sim \pi_0(\cdot | x_i)}{\frac{\mathbb{I}_{\{h(x_i)=a\}}}{\pi_0(a | x_i)^{2\alpha}}c(x_i, a)^2}+ \frac{\mathbb{I}_{\{h(x_i)=a_i\}}}{\pi_0(a_i | x_i)^{2\alpha}}c(x_i, a_i)^2\,.
\end{align}
The inequality in $(i)$ holds because $- 2   \E{a \sim \pi_0(\cdot | x_i)}{\frac{\mathbb{I}_{\{h(x_i)=a\}}}{\pi_0(a | x_i)^\alpha}c(x_i, a)} \frac{\mathbb{I}_{\{h(x_i)=a_i\}}}{\pi_0(a_i | x_i)^\alpha}c(x_i, a_i) \leq 0$. Therefore, we have that
\begin{align*}
    \langle M\rangle_n(h) + [M]_n(h) \leq \sum_{i=1}^n  \E{a \sim \pi_0(\cdot | x_i)}{\frac{\mathbb{I}_{\{h(x_i)=a\}}}{\pi_0(a | x_i)^{2\alpha}}c(x_i, a)^2}+ \frac{\mathbb{I}_{\{h(x_i)=a_i\}}}{\pi_0(a_i | x_i)^{2\alpha}}c(x_i, a_i)^2\,.
\end{align*}
Finally, by using the linearity of the expectation and the definition of policies in \eqref{eq:pac_policies}, we get that 
\begin{align}\label{eq:app_proof_5}
    \E{h \sim \mathbb{Q}}{\langle M\rangle_n(h) + [M]_n(h)} & \leq \sum_{i=1}^n  \E{a \sim \pi_0(\cdot | x_i)}{\frac{\E{h \sim \mathbb{Q}}{\mathbb{I}_{\{h(x_i)=a\}}}}{\pi_0(a | x_i)^{2\alpha}}c(x_i, a)^2}+ \frac{\E{h \sim \mathbb{Q}}{\mathbb{I}_{\{h(x_i)=a_i\}}}}{\pi_0(a_i | x_i)^{2\alpha}}c(x_i, a_i)^2\,,\nonumber\\
    & = \sum_{i=1}^n  \E{a \sim \pi_0(\cdot | x_i)}{\frac{\pi_{\mathbb{Q}}(a | x_i)}{\pi_0(a | x_i)^{2\alpha}}c(x_i, a)^2}+ \frac{\pi_{\mathbb{Q}}(a_i | x_i)}{\pi_0(a_i | x_i)^{2\alpha}}c(x_i, a_i)^2\,.
\end{align}
Combining \eqref{eq:app_proof_0} and \eqref{eq:app_proof_5} yields
\begin{align}\label{eq:app_proof_6}
  n |I_3| &= | \sum_{i=1}^n R^\alpha(\pi_{\mathbb{Q}} | x_i) - n\hat{R}_n^\alpha(\pi_{\mathbb{Q}})  |\, \nonumber\\
   &\leq \frac{D_{\mathrm{KL}}(\mathbb{Q} \| \mathbb{P})+\log (4 / \delta)}{\lambda} + \frac{\lambda}{2}\sum_{i=1}^n  \E{a \sim \pi_0(\cdot | x_i)}{\frac{\pi_{\mathbb{Q}}(a | x_i)}{\pi_0(a | x_i)^{2\alpha}}c(x_i, a)^2}+ \frac{\pi_{\mathbb{Q}}(a_i | x_i)}{\pi_0(a_i | x_i)^{2\alpha}}c(x_i, a_i)^2\,.
\end{align}
This means that the following inequality holds with probability at least $1-\delta/2$ for any distribution $\mathbb{Q}$ on $\mathcal{H}$
\begin{align}
   \left|I_3  \right| \leq \frac{D_{\mathrm{KL}}(\mathbb{Q} \| \mathbb{P})+\log (4 / \delta)}{n\lambda} + \frac{\lambda}{2n}\sum_{i=1}^n  \E{a \sim \pi_0(\cdot | x_i)}{\frac{\pi_{\mathbb{Q}}(a | x_i)}{\pi_0(a | x_i)^{2\alpha}}c(x_i, a)^2}+ \frac{\lambda}{2n} \sum_{i=1}^n\frac{\pi_{\mathbb{Q}}(a_i | x_i)}{\pi_0(a_i | x_i)^{2\alpha}}c(x_i, a_i)^2\,.
\end{align}
However we know that $c(x, a)^2 \leq 1$ for any $x\in \cX$ and $a \in \cA$ and that $c(x_i, a_i) = c_i$ for any $i \in [n]$. Thus the following inequality holds with probability at least $1-\delta/2$ for any distribution $\mathbb{Q}$ on $\mathcal{H}$
\begin{align}\label{eq:app_i3}
   \left|I_3  \right| \leq \frac{D_{\mathrm{KL}}(\mathbb{Q} \| \mathbb{P})+\log (4 / \delta)}{n\lambda} + \frac{\lambda}{2n}\sum_{i=1}^n  \E{a \sim \pi_0(\cdot | x_i)}{\frac{\pi_{\mathbb{Q}}(a | x_i)}{\pi_0(a | x_i)^{2\alpha}}}+ \frac{\lambda}{2n} \sum_{i=1}^n\frac{\pi_{\mathbb{Q}}(a_i | x_i)}{\pi_0(a_i | x_i)^{2\alpha}}c_i^2\,.
\end{align}
The union bound of \eqref{eq:app_i1} and \eqref{eq:app_i3} combined with the deterministic result in \eqref{eq:app_i2} yields that the following inequality holds with probability at least $1-\delta$ for any distribution $\mathbb{Q}$ on $\mathcal{H}$
\begin{align}\label{eq:app_main_inequality}
    |R(\pi_{\mathbb{Q}}) -\hat{R}_n^\alpha(\pi_{\mathbb{Q}})|  \leq  \sqrt{ \frac{D_{\mathrm{KL}}(\mathbb{Q} \| \mathbb{P})+\log \frac{4\sqrt{n}}{\delta}}{2n}} + \frac{1}{n} \sum_{i=1}^{n} \E{a \sim \pi_{\mathbb{Q}}(\cdot | x_i)}{1 - \pi_0^{1-\alpha}(a | x_i)} + \frac{D_{\mathrm{KL}}(\mathbb{Q} \| \mathbb{P})+\log (4 / \delta)}{n\lambda}\nonumber\\
    + \frac{\lambda}{2n}\sum_{i=1}^n  \E{a \sim \pi_0(\cdot | x_i)}{\frac{\pi_{\mathbb{Q}}(a | x_i)}{\pi_0(a | x_i)^{2\alpha}}}+ \frac{\lambda}{2n} \sum_{i=1}^n\frac{\pi_{\mathbb{Q}}(a_i | x_i)}{\pi_0(a_i | x_i)^{2\alpha}}c_i^2\,.
\end{align}
\end{proof}

\subsection{Extensions of \cref{thm:main_result}}\label{app:thm_extension}

\begin{proposition}[Extension of \cref{thm:main_result} to hold simultaneously for any $\lambda \in (0, 1)$]\label{prop:thm_any_lambda} Let $n \ge 1$, $\delta \in [0, 1]$, $\alpha \in [0, 1]$, and let $\mathbb{P}$ be a fixed prior on $\mathcal{H}$, then with probability at least $1-\delta$ over draws $\cD_n \sim \mu_{\pi_0}^n$, the following holds simultaneously for any posterior $\mathbb{Q}$ on $\mathcal{H}$, and for any $\lambda \in (0, 1)$ that 
\begin{align*}
    |R(\pi_{\mathbb{Q}}) -\hat{R}_n^\alpha(\pi_{\mathbb{Q}})| \leq \sqrt{ \frac{{\textsc{kl}^\prime}_{1}(\pi_{\mathbb{Q}}, \lambda)}{2n} } + B_n^\alpha(\pi_{\mathbb{Q}})  +
\frac{{\textsc{kl}^\prime}_{2}(\pi_{\mathbb{Q}}, \lambda)}{n \lambda } + \frac{\lambda}{2}\bar{V}_n^\alpha(\pi_{\mathbb{Q}})\,.
\end{align*}
where 
\begin{align*}
    & {\textsc{kl}^\prime}_{1}(\pi_{\mathbb{Q}}, \lambda)  =D_{\mathrm{KL}}(\mathbb{Q} \| \mathbb{P})+\log \frac{8\sqrt{n}}{\delta \lambda} \,,\\
    &{\textsc{kl}^\prime}_{2}(\pi_{\mathbb{Q}}, \lambda)  =  2\big( D_{\mathrm{KL}}(\mathbb{Q} \| \mathbb{P})+\log \frac{8 }{\delta \lambda}\big)\,,\\
    & B_n^\alpha(\pi_{\mathbb{Q}}) = 1 - \frac{1}{n}\sum_{i=1}^{n} \E{a \sim \pi_{\mathbb{Q}}(\cdot | x_i)}{\pi_0^{1-\alpha}(a | x_i)}\,,\\
    &\bar{V}_n^\alpha(\pi_{\mathbb{Q}}) = \frac{1}{n}\sum_{i=1}^n  \E{a \sim \pi_0(\cdot | x_i)}{\frac{\pi_{\mathbb{Q}}(a | x_i)}{\pi_0(a | x_i)^{2\alpha}}} + \frac{\pi_{\mathbb{Q}}(a_i | x_i)}{\pi_0(a_i | x_i)^{2\alpha}}c_i^2.
\end{align*}
\end{proposition}

\begin{proof}
Let $\delta \in (0, 1)$. For any $i \geq 1$, we define $\lambda_i=2^{-i}$ and let $\delta_i = \delta \lambda_i$. Then \cref{thm:main_result} yields that for any $i \geq 1$, the following inequality holds with probability at least $1 - \delta_i$ for any $\mathbb{Q}$ on $\mathcal{H}$
\begin{align*}
    |R(\pi_{\mathbb{Q}}) -\hat{R}_n^\alpha(\pi_{\mathbb{Q}})| \leq \sqrt{ \frac{D_{\mathrm{KL}}(\mathbb{Q} \| \mathbb{P})+\log \frac{4\sqrt{n}}{\delta_i}}{2n} } + B_n^\alpha(\pi_{\mathbb{Q}})  +
\frac{D_{\mathrm{KL}}(\mathbb{Q} \| \mathbb{P})+\log \frac{4}{\delta_i}}{n \lambda_i } + \frac{\lambda_i}{2}\bar{V}_n^\alpha(\pi_{\mathbb{Q}})\,.
\end{align*}
Now notice that $\sum_{i =1}^\infty \lambda_i = 1$, and hence $\sum_{i =1}^\infty \delta_i = \delta$. Therefore, the union bound of the above inequalities over $i \geq 1$ yields that with probability at least $1-\delta$, the following inequality holds with probability at least $1 - \delta$ for any $\mathbb{Q}$ on $\mathcal{H}$ and for any $i \geq 1$
\begin{align}\label{eq:for_any_i}
    |R(\pi_{\mathbb{Q}}) -\hat{R}_n^\alpha(\pi_{\mathbb{Q}})| \leq \sqrt{ \frac{D_{\mathrm{KL}}(\mathbb{Q} \| \mathbb{P})+\log \frac{4\sqrt{n}}{\delta_i}}{2n} } + B_n^\alpha(\pi_{\mathbb{Q}})  +
\frac{D_{\mathrm{KL}}(\mathbb{Q} \| \mathbb{P})+\log \frac{4}{\delta_i}}{n \lambda_i } + \frac{\lambda_i}{2}\bar{V}_n^\alpha(\pi_{\mathbb{Q}})\,.
\end{align}
Let $\lceil \cdot \rceil$ denote the ceiling function, then we have that for any $\lambda \in (0, 1)$, there exists $j = \lceil \frac{-\log \lambda}{\log 2} \rceil \geq 1$ such that $\lambda/2 \leq \lambda_j \leq \lambda$. Since \eqref{eq:for_any_i} holds for any $i \geq 1$, it holds in particular for $j$. In addition to this, we have that $\frac{1}{\lambda_j} \leq \frac{2}{\lambda}$, that $\lambda_j \leq \lambda$ and that $\frac{1}{\delta_j} = \frac{1}{\lambda_j \delta} \leq \frac{2}{\delta \lambda}$. This yields that the following inequality holds with probability at least $1 - \delta$ for any $\mathbb{Q}$ on $\mathcal{H}$ and for any $\lambda \in (0, 1)$
\begin{align}\label{eq:any_lambda}
    |R(\pi_{\mathbb{Q}}) -\hat{R}_n^\alpha(\pi_{\mathbb{Q}})| \leq \sqrt{ \frac{D_{\mathrm{KL}}(\mathbb{Q} \| \mathbb{P})+\log \frac{8\sqrt{n}}{\delta \lambda}}{2n} } + B_n^\alpha(\pi_{\mathbb{Q}})  +
2 \frac{D_{\mathrm{KL}}(\mathbb{Q} \| \mathbb{P})+\log \frac{8}{\delta \lambda}}{n \lambda } + \frac{\lambda}{2}\bar{V}_n^\alpha(\pi_{\mathbb{Q}})\,.
\end{align}
The additional $2$ in $2 \frac{D_{\mathrm{KL}}(\mathbb{Q} \| \mathbb{P})+\log \frac{8}{\delta \lambda}}{n \lambda }$ appears since we used that $\frac{1}{\lambda_j} \leq \frac{2}{\lambda}$. Similarly, the additional $\frac{2}{\lambda}$ in the logarithmic terms is due to the fact that $\frac{1}{\delta_j} \leq \frac{2}{\delta \lambda}$. Finally, setting
\begin{align*}
    & {\textsc{kl}^\prime}_{1}(\pi_{\mathbb{Q}}, \lambda)  =D_{\mathrm{KL}}(\mathbb{Q} \| \mathbb{P})+\log \frac{8\sqrt{n}}{\delta \lambda} \,,\\
    &{\textsc{kl}^\prime}_{2}(\pi_{\mathbb{Q}}, \lambda)  =  2\big( D_{\mathrm{KL}}(\mathbb{Q} \| \mathbb{P})+\log \frac{8 }{\delta \lambda}\big)\,,
\end{align*}
concludes the proof.
\end{proof}

Next, we provide a similar proof to extend \cref{thm:main_result} to any $\alpha \in (0, 1]$. While we only provide a one-sided inequality, the same covering technique can be used to obtain the other side of the inequality. 

\begin{proposition}[One-sided extension of \cref{thm:main_result} to hold simultaneously for any $\alpha \in (0,1)\cup \{1\}\,$]\label{prop:thm_alpha} Let $n \ge 1$, $\delta \in [0, 1]$, $\lambda >0 $, and let $\mathbb{P}$ be a fixed prior on $\mathcal{H}$, then with probability at least $1-\delta$ over draws $\cD_n \sim \mu_{\pi_0}^n$, the following holds simultaneously for any posterior $\mathbb{Q}$ on $\mathcal{H}$, and for any $\alpha \in (0, 1]$ that 
\begin{align*}
    R(\pi_{\mathbb{Q}}) \leq \hat{R}_n^\alpha(\pi_{\mathbb{Q}}) + \sqrt{ \frac{{\textsc{kl}^{\prime \prime}}_{1}(\pi_{\mathbb{Q}}, \alpha)}{2n} } + B_n^\alpha(\pi_{\mathbb{Q}})  +
\frac{{\textsc{kl}^{\prime \prime}}_{2}(\pi_{\mathbb{Q}}, \alpha)}{n \lambda } + \frac{\lambda}{2}\bar{V}_n^\alpha(\pi_{\mathbb{Q}})\,.
\end{align*}
where 
\begin{align*}
    & {\textsc{kl}^{\prime \prime}}_{1}(\pi_{\mathbb{Q}}, \alpha)  =D_{\mathrm{KL}}(\mathbb{Q} \| \mathbb{P})+\log \frac{8\sqrt{n}}{\delta \alpha} \,,\\
    &{\textsc{kl}^{\prime \prime}}_{2}(\pi_{\mathbb{Q}}, \alpha)  =  D_{\mathrm{KL}}(\mathbb{Q} \| \mathbb{P})+\log \frac{8 }{\delta \alpha}\,,\\
   & B_n^\alpha(\pi_{\mathbb{Q}}) = 1 - \frac{1}{n}\sum_{i=1}^{n} \E{a \sim \pi_{\mathbb{Q}}(\cdot | x_i)}{\pi_0^{1-\alpha}(a | x_i)}\,,\\
    &\bar{V}_n^\alpha(\pi_{\mathbb{Q}}) = \frac{1}{n}\sum_{i=1}^n  \E{a \sim \pi_0(\cdot | x_i)}{\frac{\pi_{\mathbb{Q}}(a | x_i)}{\pi_0(a | x_i)^{2\alpha}}} + \frac{\pi_{\mathbb{Q}}(a_i | x_i)}{\pi_0(a_i | x_i)^{2\alpha}}c_i^2.
\end{align*}
\end{proposition}
\begin{proof}
    
Let $\delta \in (0, 1)$. For any $i \geq 0$, we define $\alpha_i=2^{-i}$ and let $\delta_i = \delta \alpha_i / 2$. Then \cref{thm:main_result} yields that for any $i \geq 0$, the following inequality holds with probability at least $1 - \delta_i$ for any $\mathbb{Q}$ on $\mathcal{H}$
\begin{align*}
    |R(\pi_{\mathbb{Q}}) -\hat{R}_n^{\alpha_i}(\pi_{\mathbb{Q}})| \leq \sqrt{ \frac{D_{\mathrm{KL}}(\mathbb{Q} \| \mathbb{P})+\log \frac{4\sqrt{n}}{\delta_i}}{2n} } + B_n^{\alpha_i}(\pi_{\mathbb{Q}})  +
\frac{D_{\mathrm{KL}}(\mathbb{Q} \| \mathbb{P})+\log \frac{4}{\delta_i}}{n \lambda } + \frac{\lambda}{2}\bar{V}_n^{\alpha_i}(\pi_{\mathbb{Q}})\,.
\end{align*}
Now notice that $\sum_{i =0}^\infty \alpha_i = 2$, and hence by definition of $\delta_i$, we have $\sum_{i =0}^\infty \delta_i = \delta$. Therefore, the union bound of the above inequalities over $i \geq 0$ yields that with probability at least $1-\delta$, the following inequality holds with probability at least $1 - \delta$ for any $\mathbb{Q}$ on $\mathcal{H}$ and for any $i \geq 0$
\begin{align}\label{eq:for_any_alpha_i}
    |R(\pi_{\mathbb{Q}}) -\hat{R}_n^{\alpha_i}(\pi_{\mathbb{Q}})| \leq \sqrt{ \frac{D_{\mathrm{KL}}(\mathbb{Q} \| \mathbb{P})+\log \frac{4\sqrt{n}}{\delta_i}}{2n} } + B_n^{\alpha_i}(\pi_{\mathbb{Q}})  +
\frac{D_{\mathrm{KL}}(\mathbb{Q} \| \mathbb{P})+\log \frac{4}{\delta_i}}{n \lambda } + \frac{\lambda}{2}\bar{V}_n^{\alpha_i}(\pi_{\mathbb{Q}})\,.
\end{align}
Let $\lfloor \cdot \rfloor$ denote the floor function, then we have that for any $\alpha \in (0, 1]$, there exists $j = \lfloor \frac{-\log \alpha}{\log 2} \rfloor \geq 0$ such that $\alpha \leq \alpha_j \leq 2 \alpha$. Since \eqref{eq:for_any_i} holds for any $i \geq 0$, it holds in particular for $j$. In addition, we have that $B_n^\alpha(\pi_{\mathbb{Q}})$ and $\hat{R}_n^\alpha(\pi_{\mathbb{Q}})$ are decreasing in $\alpha$ while $\bar{V}_n^\alpha(\pi_{\mathbb{Q}})$ is increasing in $\alpha$. Therefore, we have that $\hat{R}_n^{\alpha_j}(\pi_{\mathbb{Q}}) \leq \hat{R}_n^\alpha(\pi_{\mathbb{Q}})\,,$ $B_n^{\alpha_j}(\pi_{\mathbb{Q}}) \leq B_n^\alpha(\pi_{\mathbb{Q}})\,,$ and $\bar{V}_n^{\alpha_j}(\pi_{\mathbb{Q}}) \leq \bar{V}_n^{2\alpha}(\pi_{\mathbb{Q}})$. Moreover, we have that $\frac{1}{\delta_j}  \leq \frac{2}{\delta \alpha}$. This yields that the following inequality holds with probability at least $1 - \delta$ for any $\mathbb{Q}$ on $\mathcal{H}$ and for any $\alpha \in (0, 1]$
\begin{align}\label{eq:any_alpha}
    R(\pi_{\mathbb{Q}})  \leq \hat{R}_n^\alpha(\pi_{\mathbb{Q}}) + \sqrt{ \frac{D_{\mathrm{KL}}(\mathbb{Q} \| \mathbb{P})+\log \frac{8\sqrt{n}}{\delta \alpha}}{2n} } + B_n^\alpha(\pi_{\mathbb{Q}})  +
\frac{D_{\mathrm{KL}}(\mathbb{Q} \| \mathbb{P})+\log \frac{8}{\delta \alpha}}{n \lambda } + \frac{\lambda}{2}\bar{V}_n^{2\alpha}(\pi_{\mathbb{Q}})\,.
\end{align}
Finally, setting
\begin{align*}
    & {\textsc{kl}^{\prime \prime}}_{1}(\pi_{\mathbb{Q}}, \alpha)  =D_{\mathrm{KL}}(\mathbb{Q} \| \mathbb{P})+\log \frac{8\sqrt{n}}{\delta \alpha} \,,\\
    &{\textsc{kl}^{\prime \prime}}_{2}(\pi_{\mathbb{Q}}, \alpha)  =  D_{\mathrm{KL}}(\mathbb{Q} \| \mathbb{P})+\log \frac{8 }{\delta \alpha}\,,
\end{align*}
concludes the proof.
\end{proof}

\subsection{Proof of \cref{prop:direct_application} }\label{proof:direct_application}

\citet[Theorem 7]{haddouche2022pac} provides an application of \cref{lemma:app_maxime} to the general PAC-Bayes learning problems in \cref{subsec:pac_bayes_framekwork}. We cannot apply their theorem directly to get \cref{prop:direct_application} for two reasons. They assume that the loss function is non-negative and they derive a one-sided generalization bound. In our case, the loss function is negative and we want to derive a two-sided generalization bound. Fortunately, we show with a slight modification of their proof that the result can be extended to two-sided inequalities with negative losses. In fact, the only requirement is that the sign of loss is fixed. We show next how this is achieved.
\begin{proof}
First, note that \cref{lemma:app_maxime} does not make any assumption on the sign of the martingale difference sequence $(f_i(S_i, h))_{i \in [n]}$ nor on the sign of the terms that decompose it. Now similarly to the proof in \cref{proof:main_thm_proof}, we define $S_n = (x_i, a_i)_{i \in [n]}$ as the set of $n$ observed contexts and taken actions. Also, we let $(\mathcal{F}_i)_{i \in \{0\}\cup[n]}$ be a filtration adapted to $S_n$. For $h \in \mathcal{H}$, we define $f_i\left(S_i, h\right)$ as 
\begin{align*}
    f_i\left(S_i, h\right) = f_i\left(x_i, a_i, h\right) = f\left(x_i, a_i, h\right) = \E{x \sim \nu, a \sim \pi_0(\cdot | x)}{\frac{\mathbb{I}_{\{h(x)=a\}}}{\pi_0(a | x)^\alpha}c(x, a)} - \frac{\mathbb{I}_{\{h(x_i)=a_i\}}}{\pi_0(a_i | x_i)^\alpha}c(x_i, a_i)\,.
\end{align*}
Here $f_i(S_i, h)$ only depends on the last samples $x_i, a_i$ and the predictor $h$. For this reason, we denote it by $f_i\left(x_i, a_i, h\right)$. Also, the function $f_i$ does not depend on $i$ and this is why we simplify the notation as $f_i\left(x_i, a_i, h\right) = f\left(x_i, a_i, h\right)$. Moreover, the randomness in $f\left(x_i, a_i, h\right)$ is only due $x_i \sim \nu$ and $a_i \sim \pi_0(\cdot | x_i)$; all other terms are deterministic. Thus the expectations are under $x_i \sim \nu, a_i \sim \pi_0(\cdot | x_i)$. Now similarly to the proof in \cref{proof:main_thm_proof}, we have that $\mathbb{E}\left[f\left(x_i, a_i, h\right) | \mathcal{F}_{i-1}\right]= 0$ for any $i \in [n]\,, h \in \mathcal{H}$. Therefore, $(f(x_i, a_i, h))_{i \in [n]}$ is a martingale difference sequence for any $h \in \mathcal{H}$. Thus we apply \cref{lemma:app_maxime} and get that that with probability at least $1-\delta$, the following holds simultaneously for any distribution $\mathbb{Q}$ on $\mathcal{H}$
\begin{align}\label{eq:app_direct_proof_0}
    \left|\E{h \sim \mathbb{Q}}{M_n(h)}\right| \leq \frac{D_{\mathrm{KL}}(\mathbb{Q} \| \mathbb{P})+\log (2 / \delta)}{\lambda}+\frac{\lambda}{2}\left(\E{h \sim \mathbb{Q}}{\langle M\rangle_n(h) + [M]_n(h)}\right)\,,
\end{align} 
where 
\begin{align*}
    &M_n(h)=\sum_{i=1}^n f\left(x_i, a_i, h\right)\,,\\
    &\langle M\rangle_n(h)=\sum_{i=1}^n \mathbb{E}\left[f\left(x_i, a_i, h\right)^2 | \mathcal{F}_{i-1}\right]\,,\\
    &[M]_n(h)=\sum_{i=1}^n f\left(x_i, a_i, h\right)^2\,.
\end{align*}
Now we compute $\E{h \sim \mathbb{Q}}{M_n(h)} $ as
\begin{align}\label{eq:app_direct_proof_1}
    \E{h \sim \mathbb{Q}}{M_n(h)}  &=\sum_{i=1}^n\E{x \sim \nu, a \sim \pi_0(\cdot | x)}{\frac{ \pi_{\mathbb{Q}}(a | x) }{\pi_0(a | x)^\alpha}c(x, a)} - \frac{ \pi_{\mathbb{Q}}(a_i | x_i)}{\pi_0(a_i | x_i)^\alpha}c(x_i, a_i)\,,\nonumber\\
    &=n \E{x \sim \nu, a \sim \pi_0(\cdot | x)}{\frac{ \pi_{\mathbb{Q}}(a | x) }{\pi_0(a | x)^\alpha}c(x, a)} - \sum_{i=1}^n \frac{ \pi_{\mathbb{Q}}(a_i | x_i)}{\pi_0(a_i | x_i)^\alpha}c(x_i, a_i)\,,
\end{align}
where we used the linearity of the expectation $\E{h \sim \mathbb{Q}}{\cdot}$ and the definition of policies in \eqref{eq:pac_policies}. Moreover, similarly to the proof in \cref{proof:main_thm_proof}, we have that
\begin{align}\label{eq:app_direct_proof_4}
  \langle M\rangle_n(h) + [M]_n(h)  & = \sum_{i=1}^n \mathbb{E}\left[f\left(x_i, a_i, h\right)^2 | \mathcal{F}_{i-1}\right] + f\left(x_i, a_i, h\right)^2\nonumber\\
    &=  \sum_{i=1}^n\E{x \sim \nu, a \sim \pi_0(\cdot | x)}{\frac{\mathbb{I}_{\{h(x)=a\}}}{\pi_0(a | x)^{2\alpha}}c(x, a)^2}+ \frac{\mathbb{I}_{\{h(x_i)=a_i\}}}{\pi_0(a_i | x_i)^{2\alpha}}c(x_i, a_i)^2 \nonumber\\
    & \hspace{1.5cm} - 2   \E{x \sim \nu, a \sim \pi_0(\cdot | x)}{\frac{\mathbb{I}_{\{h(x)=a\}}}{\pi_0(a | x)^\alpha}c(x, a)} \frac{\mathbb{I}_{\{h(x_i)=a_i\}}}{\pi_0(a_i | x_i)^\alpha}c(x_i, a_i)\,,\nonumber\\
    & \stackrel{(i)}{\leq} n\E{x \sim \nu, a \sim \pi_0(\cdot | x)}{\frac{\mathbb{I}_{\{h(x)=a\}}}{\pi_0(a | x)^{2\alpha}}c(x, a)^2} +  \sum_{i=1}^n\frac{\mathbb{I}_{\{h(x_i)=a_i\}}}{\pi_0(a_i | x_i)^{2\alpha}}c(x_i, a_i)^2\,,
\end{align}
where $(i)$ holds since $- 2   \E{x \sim \nu, a \sim \pi_0(\cdot | x)}{\frac{\mathbb{I}_{\{h(x)=a\}}}{\pi_0(a | x)^\alpha}c(x, a)} \frac{\mathbb{I}_{\{h(x_i)=a_i\}}}{\pi_0(a_i | x_i)^\alpha}c(x_i, a_i) \leq 0$ for any $i \in [n]$. This is where the non-negative loss assumption is not needed. Our loss $L_\alpha(h, x, a, c) = \frac{\mathbb{I}_{\{h(x)=a\}}}{\pi_0(a | x)^\alpha}c\,$ is negative since $c \in [-1, 0]$. However, we only need the product between the loss and its expectation to be non-negative. This holds in particular when the loss has a fixed sign. In that case, the expectation of the loss and the loss itself will have the same sign and thus their product will be non-negative. In our case, the loss has a fixed negative sign and this is all we needed. Now notice that 
\begin{align*}
    &n\E{x \sim \nu, a \sim \pi_0(\cdot | x)}{\frac{ \pi_{\mathbb{Q}}(a | x) }{\pi_0(a | x)^\alpha}c(x, a)} = nR^\alpha(\pi_{\mathbb{Q}})\,,\\
   & \sum_{i=1}^n \frac{ \pi_{\mathbb{Q}}(a_i | x_i)}{\pi_0(a_i | x_i)^\alpha}c(x_i, a_i) = n \hat{R}_n^\alpha(\pi_{\mathbb{Q}})\,,
\end{align*}
where we used that $c(x_i, a_i)=c_i$ for any $i \in [n]$ in the second equality. Using these two equalities and plugging \eqref{eq:app_direct_proof_1} and \eqref{eq:app_direct_proof_4} in \eqref{eq:app_direct_proof_0} yields that with probability at least $1 -\delta$, the following holds simultaneously for any distribution $\mathbb{Q}$ on $\mathcal{H}$
\begin{align}
    n\left|R^\alpha(\pi_{\mathbb{Q}}) - \hat{R}_n^\alpha(\pi_{\mathbb{Q}})\right|  \leq \frac{D_{\mathrm{KL}}(\mathbb{Q} \| \mathbb{P})+\log (2 / \delta)}{\lambda}+\frac{\lambda}{2} \Big(n\E{x \sim \nu, a \sim \pi_0(\cdot | x)}{\frac{ \pi_{\mathbb{Q}}(a | x)}{\pi_0(a | x)^{2\alpha}}c(x, a)^2}\nonumber\\ +  \sum_{i=1}^n\frac{ \pi_{\mathbb{Q}}(a_i | x_i)}{\pi_0(a_i | x_i)^{2\alpha}}c(x_i, a_i)^2\Big)\,.
\end{align} 
Again we used the linearity of the expectation $\E{h \sim \mathbb{Q}}{\cdot}$ and the definition of policies in \eqref{eq:pac_policies}. Finally, we have that $c(x_i, a_i)=c_i$ for any $i \in [n]$. Thus with probability at least $1 -\delta$ the following inequality holds for any distribution $\mathbb{Q}$ on $\mathcal{H}$
\begin{align}
    \left|R^\alpha(\pi_{\mathbb{Q}}) - \hat{R}_n^\alpha(\pi_{\mathbb{Q}})\right| \leq \frac{D_{\mathrm{KL}}(\mathbb{Q} \| \mathbb{P})+\log (2 / \delta)}{n \lambda} +\frac{\lambda}{2} \E{x \sim \nu, a \sim \pi_0(\cdot | x)}{\frac{\pi_{\mathbb{Q}}(a | x)}{\pi_0(a | x)^{2\alpha}}c(x, a)^2}   \nonumber\\ + \frac{\lambda}{2n}\sum_{i=1}^n\frac{\pi_{\mathbb{Q}}(a_i | x_i)}{\pi_0(a_i | x_i)^{2\alpha}}c_i^2\,.
\end{align} 
This concludes the proof.
\end{proof}

\subsection{Sample Complexity}\label{proof:practice_theory}

\begin{proposition}\label{prop:samples_oracle}
Let $\mathcal{M}_1(\mathcal{H})$ be the set of probability distributions on the hypothesis space $\mathcal{H}$, and let $\lambda>0$,  $n \ge 1$, $\delta \in [0, 1]$, $\alpha \in [0, 1]$, and let $\mathbb{P}$ be a fixed prior on $\mathcal{H}$, then with probability at least $1-\delta$ over draws $\cD_n \sim \mu_{\pi_0}^n$, we have
\begin{align*}
    R( \pi_{\hat{\mathbb{Q}}_n}) \leq R(\pi_{\mathbb{Q}_*})  + 2 \sqrt{ \frac{{\textsc{kl}}_{1}(\pi_{\mathbb{Q}_*})}{2n} } + 2B_n^\alpha(\pi_{\mathbb{Q}_*})  +
2\frac{{\textsc{kl}}_{2}(\pi_{\mathbb{Q}_*})}{n \lambda } + \lambda \bar{V}_n^\alpha(\pi_{\mathbb{Q}_*})\,.
\end{align*}
where $\pi_{\hat{\mathbb{Q}}_n}\,$ is the learned policy with $\, \hat{\mathbb{Q}}_n =  \argmin_{\mathbb{Q} \in \mathcal{M}_1(\mathcal{H})} \hat{R}_n^\alpha(\pi_{\mathbb{Q}}) + \sqrt{ \frac{{\textsc{kl}}_{1}(\pi_{\mathbb{Q}})}{2n} } + B_n^\alpha(\pi_{\mathbb{Q}})  +
\frac{{\textsc{kl}}_{2}(\pi_{\mathbb{Q}})}{n \lambda } + \frac{\lambda}{2}\bar{V}_n^\alpha(\pi_{\mathbb{Q}})\,,$ $\mathbb{Q}_* =  \argmin_{\mathbb{Q} \in \mathcal{M}_1(\mathcal{H})} R(\pi_{\mathbb{Q}})$, and
\begin{align*}
    &{\textsc{kl}}_{1}(\pi_{\mathbb{Q}})  =D_{\mathrm{KL}}(\mathbb{Q} \| \mathbb{P})+\log \frac{4\sqrt{n}}{\delta}\,, \qquad {\textsc{kl}}_{2}(\pi_{\mathbb{Q}})  =  D_{\mathrm{KL}}(\mathbb{Q} \| \mathbb{P})+\log \frac{4}{\delta}\,,\\
   & B_n^\alpha(\pi_{\mathbb{Q}}) = 1 - \frac{1}{n}\sum_{i=1}^{n} \E{a \sim \pi_{\mathbb{Q}}(\cdot | x_i)}{\pi_0^{1-\alpha}(a | x_i)}\,, \quad \bar{V}_n^\alpha(\pi_{\mathbb{Q}}) = \frac{1}{n}\sum_{i=1}^n  \E{a \sim \pi_0(\cdot | x_i)}{\frac{\pi_{\mathbb{Q}}(a | x_i)}{\pi_0(a | x_i)^{2\alpha}}} + \frac{\pi_{\mathbb{Q}}(a_i | x_i)c_i^2}{\pi_0(a_i | x_i)^{2\alpha}}\,.
\end{align*}
\end{proposition}

\begin{proof}
First, \cref{thm:main_result} holds for any potentially data dependent distribution $\mathbb{Q}$ on $\mathcal{H}$. In particular, we have that with probability at least $1-\delta$ the following inequalities hold simultaneously for $\hat{\mathbb{Q}}_n$ and $\mathbb{Q}_*$
\begin{align*}
    & |R(\pi_{\hat{\mathbb{Q}}_n}) -\hat{R}_n^\alpha(\pi_{\hat{\mathbb{Q}}_n})| \leq \sqrt{ \frac{{\textsc{kl}}_{1}(\pi_{\hat{\mathbb{Q}}_n})}{2n} } + B_n^\alpha(\pi_{\hat{\mathbb{Q}}_n})  +
\frac{{\textsc{kl}}_{2}(\pi_{\hat{\mathbb{Q}}_n})}{n \lambda } + \frac{\lambda}{2}\bar{V}_n^\alpha(\pi_{\hat{\mathbb{Q}}_n})\,,\\
&|R(\pi_{\mathbb{Q}_*}) -\hat{R}_n^\alpha(\pi_{\mathbb{Q}_*})| \leq \sqrt{ \frac{{\textsc{kl}}_{1}(\pi_{\mathbb{Q}_*})}{2n} } + B_n^\alpha(\pi_{\mathbb{Q}_*})  +
\frac{{\textsc{kl}}_{2}(\pi_{\mathbb{Q}_*})}{n \lambda } + \frac{\lambda}{2}\bar{V}_n^\alpha(\pi_{\mathbb{Q}_*})\,.
\end{align*}
Taking only one side of these inequalities yields that with probability at least $1-\delta$ the following inequalities hold simultaneously for $\hat{\mathbb{Q}}_n$ and $\mathbb{Q}_*$
\begin{align*}
    &R(\pi_{\hat{\mathbb{Q}}_n}) \leq \underbrace{\hat{R}_n^\alpha(\pi_{\hat{\mathbb{Q}}_n}) + \sqrt{ \frac{{\textsc{kl}}_{1}(\pi_{\hat{\mathbb{Q}}_n})}{2n} } + B_n^\alpha(\pi_{\hat{\mathbb{Q}}_n})  +
\frac{{\textsc{kl}}_{2}(\pi_{\hat{\mathbb{Q}}_n})}{n \lambda } + \frac{\lambda}{2}\bar{V}_n^\alpha(\pi_{\hat{\mathbb{Q}}_n})}_{(I)}\,,\\
& \hat{R}_n^\alpha(\pi_{\mathbb{Q}_*}) \leq R(\pi_{\mathbb{Q}_*}) +  \sqrt{ \frac{{\textsc{kl}}_{1}(\pi_{\mathbb{Q}_*})}{2n} } + B_n^\alpha(\pi_{\mathbb{Q}_*})  +
\frac{{\textsc{kl}}_{2}(\pi_{\mathbb{Q}_*})}{n \lambda } + \frac{\lambda}{2}\bar{V}_n^\alpha(\pi_{\mathbb{Q}_*})\,.
\end{align*}
Now using the definition of $\pi_{\hat{\mathbb{Q}}_n}$, we know that 
\begin{align*}
   I \leq \hat{R}_n^\alpha(\pi_{\mathbb{Q}_*})  +  \sqrt{ \frac{{\textsc{kl}}_{1}(\pi_{\mathbb{Q}_*})}{2n} } + B_n^\alpha(\pi_{\mathbb{Q}_*})  +
\frac{{\textsc{kl}}_{2}(\pi_{\mathbb{Q}_*})}{n \lambda } + \frac{\lambda}{2}\bar{V}_n^\alpha(\pi_{\mathbb{Q}_*})\,.
\end{align*}
This yields that with probability at least $1-\delta$ the following inequalities hold simultaneously for $\hat{\mathbb{Q}}_n$ and $\mathbb{Q}_*$
\begin{align*}
    &R(\pi_{\hat{\mathbb{Q}}_n}) \leq \hat{R}_n^\alpha(\pi_{\mathbb{Q}_*})  +  \sqrt{ \frac{{\textsc{kl}}_{1}(\pi_{\mathbb{Q}_*})}{2n} } + B_n^\alpha(\pi_{\mathbb{Q}_*})  +
\frac{{\textsc{kl}}_{2}(\pi_{\mathbb{Q}_*})}{n \lambda } + \frac{\lambda}{2}\bar{V}_n^\alpha(\pi_{\mathbb{Q}_*})\,,\\
& \hat{R}_n^\alpha(\pi_{\mathbb{Q}_*}) \leq R(\pi_{\mathbb{Q}_*}) +  \sqrt{ \frac{{\textsc{kl}}_{1}(\pi_{\mathbb{Q}_*})}{2n} } + B_n^\alpha(\pi_{\mathbb{Q}_*})  +
\frac{{\textsc{kl}}_{2}(\pi_{\mathbb{Q}_*})}{n \lambda } + \frac{\lambda}{2}\bar{V}_n^\alpha(\pi_{\mathbb{Q}_*})\,.
\end{align*}
Computing the sum of these two inequalities concludes the proof.
\end{proof}

\begin{corollary}[Special case of \cref{prop:samples_oracle}] Let $\mathcal{H} = \set{h_{\theta} \, ; \theta \in \real^{dK}}$ of mappings $h_{\theta}(x) = \argmax_{a \in \cA} \phi(x)^\top \theta_a$ for any $x \in \cX$. Let $n \ge 1$, $\delta \in [0, 1]$, $\alpha \in [0, 1]$, and let $\mathbb{P} = \cN(\mu_0, I_{dK})$ be a fixed prior on $\mathcal{H}$, then with probability at least $1-\delta$ over draws $\cD_n \sim \mu_{\pi_0}^n$, we have that
\begin{align*}
    R( \pi_{\hat{\mathbb{Q}}_n}) \leq R(\pi_{\mathbb{Q}_*})  +  \frac{ \sqrt{\norm{\mu_* - \mu_0}^2 + 2\log\frac{4 \sqrt{n}}{\delta}} }{\sqrt{n}}  + 2(1 - K^{\alpha-1})  +
\frac{ \norm{\mu_* - \mu_0}^2 + 2\log\frac{4}{\delta}}{\sqrt{n}} + \frac{K^{2\alpha-1} + K^{2\alpha}}{\sqrt{n}}\,.
\end{align*}
where $\pi_{\hat{\mathbb{Q}}_n}\,$ is the learned policy with $\, \hat{\mathbb{Q}}_n =  \argmin_{\mathbb{Q} = \cN(\mu, I_{dK}) } \hat{R}_n^\alpha(\pi_{\mathbb{Q}}) + \sqrt{ \frac{{\textsc{kl}}_{1}(\pi_{\mathbb{Q}})}{2n} } + B_n^\alpha(\pi_{\mathbb{Q}})  +
\frac{{\textsc{kl}}_{2}(\pi_{\mathbb{Q}})}{n \lambda } + \frac{\lambda}{2}\bar{V}_n^\alpha(\pi_{\mathbb{Q}})\,,$ $\mathbb{Q}_* =  \argmin_{\mathbb{Q} = \cN(\mu, I_{dK})} R(\pi_{\mathbb{Q}})$.
\end{corollary}

\begin{proof}
    This result follows from the general \cref{prop:samples_oracle} by simply setting $\mathbb{P} = \cN(\mu_0, I_{dK})$ and $\mathbb{Q}_* = \cN(\mu_*, I_{dK})$. First, since the covariance matrices of both distributions are $I_{dK}$, their KL divergence is $D_{\mathrm{KL}}(\mathbb{Q} \| \mathbb{P}) = \norm{\mu_* - \mu_0}^2/2$. Moreover, since the logging policy is uniform then $B_n^\alpha(\pi_{\mathbb{Q}}) = (1-K^{\alpha-1})$ and $ \bar{V}_n^\alpha(\pi_{\mathbb{Q}})
 \leq K^{2\alpha-1} + K^{2\alpha}$.  Using these quantities, setting $\lambda = 1/\sqrt{n}$ and applying \cref{prop:samples_oracle} yields that with probability at least $1-\delta$ over draws $\cD_n \sim \mu_{\pi_0}^n$, we have that
\begin{align*}
    R( \pi_{\hat{\mathbb{Q}}_n}) \leq R(\pi_{\mathbb{Q}_*})  +  \frac{ \sqrt{\norm{\mu_* - \mu_0}^2 + 2\log\frac{4 \sqrt{n}}{\delta}} }{\sqrt{n}}+ 2(1 - K^{\alpha-1}) + \frac{ \norm{\mu_* - \mu_0}^2 + 2\log\frac{4}{\delta}}{\sqrt{n}} + \frac{K^{2\alpha-1} + K^{2\alpha}}{\sqrt{n}}\,.
\end{align*}
This concludes the proof.
\end{proof}

The above corollary allows us to give insights into the sample complexity of our procedure. That is, the number of samples needed so that the performance of the learned policy $\pi_{\hat{\mathbb{Q}}_n}$ is close to that of the optimal one. Let $\epsilon> 2(1 - K^{\alpha-1})$ for $\alpha \in [1- \log 2 / \log K, 1]$. This condition on $\alpha$ ensures that $\epsilon \in [0, 1]$ and it is mild as $\alpha$ is often close to 1. Let $\delta$, then the following implication holds
\begin{align}\label{eq:inequality1}
    \epsilon \geq \frac{ \sqrt{\norm{\mu_* - \mu_0}^2 + 2\log\frac{4 \sqrt{n}}{\delta}} }{\sqrt{n}}+ 2(1 - K^{\alpha-1}) + \frac{ \norm{\mu_* - \mu_0}^2 + 2\log\frac{4}{\delta}}{\sqrt{n}} + \frac{K^{2\alpha-1} + K^{2\alpha}}{\sqrt{n}}  \nonumber \\ \implies \mathbb{P}(R(\pi_{\hat{\mathbb{Q}}_n}) \leq R(\pi_{\mathbb{Q}_*}) + \epsilon) \geq 1-\delta\,.
\end{align}
First, we use that $\sqrt{\norm{\mu_* - \mu_0}^2 + 2\log\frac{4 \sqrt{n}}{\delta}} \leq \norm{\mu_* - \mu_0} + \sqrt{2\log\frac{4 \sqrt{n}}{\delta}}$. Moreover we bound $K^{2\alpha-1} + K^{2\alpha} \leq 2K^{2 \alpha}$. Then the implication in \eqref{eq:inequality1} becomes
\begin{align}\label{eq:inequality2}
   \sqrt{n} \geq \frac{\norm{\mu_* - \mu_0} + \norm{\mu_* - \mu_0}^2 + 2\log\frac{4}{\delta} + \sqrt{2\log\frac{4 \sqrt{n}}{\delta}} + 2K^{2\alpha} }{\epsilon - 2(1 - K^{\alpha-1})}  \implies \mathbb{P}(R(\pi_{\hat{\mathbb{Q}}_n}) \leq R(\pi_{\mathbb{Q}_*}) + \epsilon) \geq 1-\delta\,.
\end{align}
We only provide intuition on the sample complexity and aim at having easy-to-interpret terms. Thus we omit the logarithmic terms in \eqref{eq:inequality2} and assume that $\norm{\mu_* - \mu_0}^2 \geq \norm{\mu_* - \mu_0}$. This leads to the claim made in \cref{subsec:interpretation}. Of course, a more precise sample complexity analysis can be made by studying the function $h(x) = \sqrt{x} -  \sqrt{2\log\frac{4 \sqrt{x}}{\delta}}/(\epsilon - 2(1 - K^{\alpha-1}))$ and finding $x$ such that $f(x) \geq \frac{\norm{\mu_* - \mu_0} + \norm{\mu_* - \mu_0}^2 + 2\log\frac{4}{\delta} + 2K^{2\alpha} }{\epsilon - 2(1 - K^{\alpha-1})}$.
\section{Experiments}\label{app:all_experiments}

\subsection{Setup}\label{app:setup}
We consider the standard supervised-to-bandit conversion \citep{agarwal2014taming}. Precisely, let $\mathcal{S}^{\textsc{tr}}_{n}$ and $\mathcal{S}^{\textsc{ts}}_{n_{\textsc{ts}}}$ be the training and testing set of a classification dataset, respectively. First, we transform the training set $\mathcal{S}^{\textsc{tr}}_{n}$ to a logged bandit data $\cD_n$ as described in \cref{alg:supervised_to_bandit}. The resulting logged data $\cD_n$ is then used to train our policies. After that, the learned policies are tested on $\mathcal{S}^{\textsc{ts}}_{n_{\textsc{ts}}}$ as described in \cref{alg:supervised_to_bandit_test}. We consider that the resulting reward in \cref{alg:supervised_to_bandit_test} is a good proxy for the unknown true reward of the learned policies. This will be our performance metric, the higher the better.

In our experiments, we use the following image classification datasets \texttt{MNIST} \citep{lecun1998gradient}, \texttt{FashionMNIST} \citep{xiao2017fashion},  \texttt{EMNIST} \citep{cohen2017emnist} and \texttt{CIFAR100} \citep{krizhevsky2009learning}. We provide a summary of the statistics of these datasets in \cref{tab:stats}. \cref{alg:supervised_to_bandit} takes as input a logging policy $\pi_0$ which we define as
\begin{align}\label{eq:logging}
    &\pi_0(a | x) = \frac{ \exp(\eta_0 \cdot \phi(x)^\top \mu_{0, a})}{\sum_{a^\prime \in \cA}  \exp(\eta_0 \cdot \phi(x)^\top \mu_{0, a^\prime})}\,, & \forall (x,a) \in \cX \times \cA\,.
\end{align}
Here $\phi(x) \in \real^d$ is the feature transformation function that outputs a $d$-dimensional vector, $\mu_0 = (\mu_{0,a})_{a \in \cA} \in \real^{dK}$ are learnable parameters and $\eta_0$ is an inverse-temperature parameter for the \texttt{softmax} in \eqref{eq:logging}. We explain next how these quantities are derived in detail.

\textbf{The feature transformation function $\phi(x) \in \real^d$:} for all the datasets, except \texttt{CIFAR100}, the feature transformation function $\phi(\cdot)$ is defined as $\phi(x) = \frac{x}{\norm{x}}$ for any $x \in \cX$. That is, we simply normalize the features $x \in \cX$ by their $L_2$ norm $\norm{x}$. In contrast, \texttt{CIFAR100} is a more challenging problem. Thus we use transfer learning to extract features $\phi(x)$ expressive enough so that a linear \texttt{softmax} model would enjoy a reasonable performance. Precisely, we retrieve the last hidden layer of a \texttt{ResNet-50} network, pre-trained on the ImageNet dataset, to output 2048-dimensional features. Finally, the obtained features are normalized as $\frac{x}{\norm{x}}$ and this whole process (\texttt{ResNet-50}  + normalization) corresponds to $\phi(\cdot)$ for \texttt{CIFAR100}.

\textbf{The parameters $\mu_0 = (\mu_{0,a})_{a \in \cA} \in \real^{dK}$:} we learn the parameters $\mu_0$ using 5\% of the training set $\mathcal{S}^{\textsc{tr}}_{n}$. Precisely, we use the cross-entropy loss with an $L_2$ regularization of $10^{-6}$ to prevent the logging policy $\pi_0$ from being degenerate.  This ensures that the learning policies are absolutely continuous with respect to the logging policy $\pi_0$, a condition under which standard IPS is unbiased. In optimization, we use Adam \citep{kingma2014adam} with a learning rate of $0.1$ for $10$ epochs. In all the experiments, we set the prior $\mathbb{P} = \cN(\eta_0 \mu_0, I_{dK})$ for the Gaussian policies in \eqref{eq:gaussian_pac_bayes} and we set it as $\mathbb{P} = \cN(\eta_0 \mu_0, I_{dK}) \times {\rm G}(0, 1)^K$ for the mixed-logit policies in \eqref{eq:logit_pac_bayes}. Our theory requires that the prior does not depend on data. Given that $\mu_0$ is learned on the $5\%$ portion of data, we only train our learning policies on the remaining  $95\%$ portion of the data to match our theoretical requirements.

\textbf{The inverse-temperature parameter $\eta_0 \in \real$:} this controls the performance of the logging policy. A high positive value of $\eta_0$ leads to a well-performing logging policy, while a negative one leads to a low-performing logging policy. When $\eta_0=0$, $\pi_0$ is identical to the uniform policy. In our experiments $\eta_0$ varies between $0$ and $1$.

\begin{algorithm}
\caption{Supervised-to-bandit: creating logged data}
\label{alg:supervised_to_bandit}
\textbf{Input:} training classification set $\mathcal{S}^{\textsc{tr}}_{n}=\{(x_i, y_i)\}_{i=1}^n$, logging policy $\pi_0$.\\
\textbf{Output:} logged bandit data $\mathcal{D}_n=(x_i, a_i, c_i)_{i \in [n]}$.\\
Initialize $\mathcal{D}_n =\{\}$ \\
\For{$i=1, \dots, n$}
{$a_i \sim \pi_0(\cdot | x_i)$\\
$c_i = - \mathbb{I}_{\{a_i = y_i\}}$\\
$\mathcal{D}_n \gets \mathcal{D}_n \cup \{(x_i, a_i, c_i)\}\,.$ 
}
\end{algorithm}

\begin{algorithm}
\caption{Supervised-to-bandit: testing policies}
\label{alg:supervised_to_bandit_test}
\textbf{Input:} image classification dataset $\mathcal{S}^{\textsc{ts}}_{n_{\textsc{ts}}}=\{(x_i, y_i)\}_{i=1}^{n_{\textsc{ts}}}$, learned policy $ \hat{\pi}_n$.\\
\textbf{Output:} reward $r$.\\
\For{$i=1, \dots, n_{\textsc{ts}}$}
{$a_i \sim \hat{\pi}_n(\cdot | x_i)$\\
$r_i = \mathbb{I}_{\{a_i = y_i\}}$}
$r = \frac{1}{n_{\textsc{ts}}} \sum_{i=1}^{n_{\textsc{ts}}} r_i\,.$
\end{algorithm}

\begin{table}[t]
\caption{Statistics of the datasets used in our experiments.}
\label{tab:stats}
\vskip 0.15in
\begin{center}
\begin{small}
\begin{sc}
\begin{tabular}{lcccr}
\toprule
Data set & Nbr. train samples $n$ & Nbr. test samples $n_{\textsc{ts}}$ & Nbr. actions $K$ & Dimension $d$ \\
\bottomrule
\texttt{MNIST}    & 60000 & 10000 & 10& 784 \\
\texttt{FashionMNIST} & 60000 & 10000 & 10& 784\\
\texttt{EMNIST}    & 112800 &18800 & 47& 784 \\
\texttt{CIFAR100}    & 50000 & 10000 & 100&   2048\\
\bottomrule
\end{tabular}
\end{sc}
\end{small}
\end{center}
\vskip -0.1in
\end{table}

Now it remains to explain the learning policies $\pi_{\mathbb{Q}}$ and the corresponding closed-form bounds using either our results or those in existing works \citep{london2019bayesian,sakhi2022pac}. 

\subsection{Policies}\label{app:policies}
Here we present the two families of policies that we use in our experiments, Gaussian and mixed-logit policies.
\subsubsection{Mixed-Logit}
Let $\mathcal{H} = \set{h_{\theta, \gamma} \, ; \theta \in \real^{dK}, \gamma \in \real^K}$ be a hypothesis space of mappings $h_{\theta, \gamma}(x) = \argmax_{a \in \cA} \phi(x)^\top \theta_a + \gamma_a$ for any $x \in \cX$. Here $\phi(x)$ outputs a $d$-dimensional representation of context $x \in \cX$. Now assume that for any $a \in \cA$, $\gamma_a$ is a standard Gumbel perturbation, $\gamma_a \sim {\rm G}(0, 1)$, then we have that 
\begin{align}\label{eq:app_softmax_pac_bayes}
    \pi^{\textsc{sof}}_{\theta}(a | x) &= \frac{\exp(\phi(x)^\top \theta_a)}{\sum_{a^\prime \in \cA}\exp(\phi(x)^\top  \theta_{a^\prime})}\,,\nonumber\\
    &= \E{\gamma \sim {\rm G}(0, 1)^K}{\mathbb{I}_{\{ h_{\theta, \gamma}(x) = a \}}}\,.
\end{align}
In addition, we randomize $\theta$ such as $\theta \sim \cN(\mu, \sigma^2 I_{dK})$ where $\mu \in \real^{dK}$ and $\sigma>0$. It follows that the posterior $\mathbb{Q}$ is a multivariate Gaussian $\cN(\mu, \sigma^2 I_{dK})$ over the parameters $\theta$ with standard Gumbel perturbations $\gamma \sim {\rm G}(0, 1)^K$. We denote such policies by $\pi^{\textsc{mixL}}_{\mu, \sigma}$ and they are defined as
\begin{align}\label{eq:app_logit_pac_bayes}
    \pi^{\textsc{mixL}}_{\mu, \sigma}(a | x) 
    &=  \E{\theta \sim \cN(\mu, \sigma^2 I_{dK})}{\frac{\exp(\phi(x)^\top \theta_a)}{\sum_{a^\prime \in \cA}\exp(\phi(x)^\top  \theta_{a^\prime})}}\,,\nonumber\\
    &=  \E{\theta \sim \cN(\mu, \sigma^2 I_{dK})}{\pi^{\textsc{sof}}_{\theta}(a | x)}\,,\nonumber\\
    &= \E{\theta \sim \cN(\mu, \sigma^2 I_{dK})\,, \gamma \sim {\rm G}(0, 1)^K}{\mathbb{I}_{\{ h_{\theta, \gamma}(x) = a \}}}\,.
\end{align}
To sample from the mixed-logit policies $\pi^{\textsc{mixL}}_{\mu, \sigma}$, we first sample $\theta \sim \cN(\mu, \sigma^2 I_{dK})$ and $\gamma \sim {\rm G}(0, 1)^K$ and then set the sampled action as $a \gets h_{\theta, \gamma}(x)$. Now we also need to compute the gradient of the expectation in \eqref{eq:app_logit_pac_bayes}. This needs additional care since the distribution under which we take the expectation depends on the parameters $\mu, \sigma$. Fortunately, the reparameterization trick can be used in this case. Roughly speaking, it allows us to express a gradient of the expectation in \eqref{eq:app_logit_pac_bayes} as an expectation of a gradient. In our case, we use the \emph{local} reparameterizaton trick \citep{kingma2015variational} which is known for reducing the variance of stochastic gradients. Precisely, we rewrite \eqref{eq:app_logit_pac_bayes} as
\begin{align}
      \pi^{\textsc{mixL}}_{\mu, \sigma}(a | x)  &=  \E{\epsilon \sim \cN(0, \norm{\phi(x)}^2 I_{K})}{\frac{\exp(\phi(x)^\top \mu_a + \sigma \epsilon_a)}{\sum_{a^\prime \in \cA}\exp(\phi(x)^\top  \mu_{a^\prime} + \sigma \epsilon_{a^\prime})}}\,.\nonumber\\
      &=  \E{\epsilon \sim \cN(0, I_{K})}{\frac{\exp(\phi(x)^\top \mu_a + \sigma \epsilon_a)}{\sum_{a^\prime \in \cA}\exp(\phi(x)^\top  \mu_{a^\prime} + \sigma \epsilon_{a^\prime})}}\,,\nonumber
\end{align}
where we used that $\norm{\phi(x)}^2=1$ since features are normalized. It follows that gradients read 
$$\nabla_{\mu, \sigma} \pi^{\textsc{mixL}}_{\mu, \sigma}(a | x) =  \E{\epsilon \sim \cN(0, I_{K})}{\nabla_{\mu, \sigma}\frac{\exp(\phi(x)^\top \mu_a + \sigma \epsilon_a)}{\sum_{a^\prime \in \cA}\exp(\phi(x)^\top  \mu_{a^\prime} + \sigma \epsilon_{a^\prime})}}\,.$$
Moreover, the propensities are approximated as 
\begin{align}
      &\pi^{\textsc{mixL}}_{\mu, \sigma}(a | x)  \approx \frac{1}{S} \sum_{i \in [S]}{\frac{\exp(\phi(x)^\top \mu_a + \sigma \epsilon_{i, a})}{\sum_{a^\prime \in \cA}\exp(\phi(x)^\top  \mu_{a^\prime} + \sigma \epsilon_{i, a^\prime})}}\,, & 
\epsilon_i  \sim \cN(0, I_{K})\,, \forall i \in [S]\,.
\end{align}
In all our experiments, we set $S=32$.

\subsubsection{Gaussian}
We define the hypothesis space $\mathcal{H} = \set{h_{\theta} \, ; \theta \in \real^{dK}}$ of mappings $h_{\theta}(x) = \argmax_{a \in \cA} \phi(x)^\top \theta_a$ for any $x \in \cX$. It follows that the learning policies $ \pi_{\mathbb{Q}} = \pi^{\textsc{gaus}}_{\mu, \sigma}$ read 
\begin{align}\label{eq:app_gaussian_pac_bayes}
   \pi^{\textsc{gaus}}_{\mu, \sigma}(a | x) = \E{\theta \sim \cN(\mu, \sigma^2 I_{dK})}{\mathbb{I}_{\{ h_{\theta}(x) = a \}}}\,.
\end{align}
To see why this can be beneficial \citep{sakhi2022pac}, let $\pi_*$ be the optimal policy. Given $x \in \cX$, $\pi_*(\cdot | x)$ should be deterministic; it chooses the best action for context $x$ with probability 1. That is, there exists $\mu_* \in \real^{dK}$ such that $\pi_* = \mathbb{I}_{\{ h_{\mu_*}(x) = a \}}$. When $\mu \rightarrow \mu_*$ and $\sigma \rightarrow 0$, the Gaussian policy in \eqref{eq:app_gaussian_pac_bayes} approaches $\pi_*$. In contrast, the mixed-logit policy in \eqref{eq:app_logit_pac_bayes} approaches $\pi^{\textsc{sof}}_{\mu_*}$. However, $\pi^{\textsc{sof}}_{\mu_*}$ is not deterministic due to the additional randomness in $\gamma$ and is equal to $\pi_*$ only if $\phi(x)^\top \mu_{*,a_*(x)} \rightarrow \infty$. This explains the choice of removing the Gumbel noise. First, \citet{sakhi2022pac} showed that \eqref{eq:app_gaussian_pac_bayes} can be written as 
\begin{align*}
 \pi^{\textsc{gaus}}_{\mu, \sigma}(a | x)  & =\mathbb{E}_{\epsilon \sim \cN(0, 1)}\Big[\prod_{a^{\prime} \neq a} \Phi\big(\epsilon+\frac{\phi(x)^\top\left(\mu_a-\mu_{a^{\prime}}\right)}{\sigma\|\phi(x)\|}\big)\Big]\,,
\end{align*}
where $\Phi$ is the cumulative distribution function of a standard normal variable.  But $\|\phi(x)\|=1$ in all our experiments. Thus
\begin{align*}
 \pi^{\textsc{gaus}}_{\mu, \sigma}(a | x)  & =\mathbb{E}_{\epsilon \sim \cN(0, 1)}\Big[\prod_{a^{\prime} \neq a} \Phi\big(\epsilon+\frac{\phi(x)^\top\left(\mu_a-\mu_{a^{\prime}}\right)}{\sigma}\big)\Big]\,.
\end{align*}
Then similarly to mixed-logit policies, the gradient reads 
\begin{align*}
    \nabla_{\mu, \sigma} \pi^{\textsc{gaus}}_{\mu, \sigma}(a | x) =  \mathbb{E}_{\epsilon \sim \cN(0, 1)}\Big[\nabla_{\mu, \sigma}\prod_{a^{\prime} \neq a} \Phi\big(\epsilon+\frac{\phi(x)^\top\left(\mu_a-\mu_{a^{\prime}}\right)}{\sigma}\big)\Big]\,.
\end{align*}
Moreover, the propensities are approximated as 
\begin{align}
      &\pi^{\textsc{gaus}}_{\mu, \sigma}(a | x)  \approx \frac{1}{S} \sum_{i \in [S]}{\prod_{a^{\prime} \neq a} \Phi\big(\epsilon_i+\frac{\phi(x)^\top\left(\mu_a-\mu_{a^{\prime}}\right)}{\sigma}\big)}\,, & 
\epsilon_i  \sim \cN(0, 1)\,, \forall i \in [S]\,.
\end{align}
In all our experiments, we set $S=32$.

\subsection{Baselines}\label{app:baselines}
Here we present all the methods that we use in our experiments. For each method, we state the result that holds for any learning policy $\pi$. After that, we derive the corresponding closed-form bounds for Gaussian and mixed-logit policies that we presented previously. All the baselines require computing the KL divergence between the prior $\mathbb{P}$ and the posterior $\mathbb{Q}$. Thus before presenting them, we state the following lemma that allows bounding the KL divergence between the prior $\mathbb{P}$ and the posterior $\mathbb{Q}$ in the cases of mixed-logit or Gaussian policies.

\begin{lemma}[KL divergence for Gaussian distributions with Gumbel noise]\label{lemma:kl_gaussian}
For distributions $\mathbb{P} = \mathcal{N}\left(\mu_0, \sigma_0^2 I_{dK}\right) \times \operatorname{G}(0,1)^K$ and $\mathbb{Q} = \mathcal{N}\left(\mu, \sigma^2 I_{dK}\right) \times \operatorname{G}(0,1)^K$, with $\mu_0, \mu \in \mathbb{R}^{dK}$ and $0<\sigma^2 \leq \sigma_0^2<\infty$,
$$
D_{\mathrm{KL}}(\mathbb{Q} \| \mathbb{P}) \leq \frac{\left\|\mu-\mu_0\right\|^2}{2 \sigma_0^2}+\frac{dK}{2} \log \frac{\sigma_0^2}{\sigma^2}\,.
$$
Moreover, this result holds when the Gumbel noise is removed. That is when $\mathbb{P} = \mathcal{N}\left(\mu_0, \sigma_0^2 I_{dK}\right)$ and $\mathbb{Q} = \mathcal{N}\left(\mu, \sigma^2 I_{dK}\right)$.
\end{lemma}

We borrow this lemma from \citet{london2019bayesian}. In particular, \cref{lemma:kl_gaussian} shows that the KL terms for both policies can be bounded by the same quantity. As a result, the corresponding bounds will be the same; the only difference is the space of learning policies on which we optimize. For completeness, however, we write these bounds for both types of policies although they are similar. Since existing approaches are not named, we name them as \textbf{(Author, Policy)} where \textbf{Author $\in$ \{Ours, London et al., Sakhi et al. 1, Sakhi et al. 2\} } and \textbf{Policy $\in$ \{Gaussian, Mixed-Logit\} }. Here \textbf{Ours}, \textbf{London et al.}, \textbf{Sakhi et al. 1} and\textbf{ Sakhi et al. 2} correspond to \cref{thm:main_result},  \citet[Theorem 1]{london2019bayesian}, \citet[Proposition 1]{sakhi2022pac}, \citet[Proposition 3]{sakhi2022pac}, respectively. For example, \citet[Theorem 1]{london2019bayesian} leads to two baselines \textbf{(London et al., Gaussian)} and \textbf{(London et al., Mixed-Logit)}. In all our experiments, the learning policies are trained using Adam \citep{kingma2014adam} with a learning rate of $0.1$ for $20$ epochs.

\subsubsection{Ours, \cref{thm:main_result}}

\textbf{(Ours, Gaussian)} Here we use the Gaussian policies in \eqref{eq:app_gaussian_pac_bayes}. Thus we only replace the term, $D_{\mathrm{KL}}(\mathbb{Q} \| \mathbb{P})$, with its closed-form bound in \cref{lemma:kl_gaussian}. This leads to the following objective. 
\begin{align*}
  \min_{\mu \in \real^{dK}, \sigma >0} \Big( \hat{R}_n^\alpha\left(\pi^{\textsc{gaus}}_{\mu, \sigma}\right) + \sqrt{ \frac{\frac{\left\|\mu-\mu_0\right\|^2}{2} - \frac{dK}{2} \log \sigma^2+\log \frac{4\sqrt{n}}{\delta}}{2n} } + B_n^\alpha(\pi^{\textsc{gaus}}_{\mu, \sigma})  +
\frac{\frac{\left\|\mu-\mu_0\right\|^2}{2} - \frac{dK}{2} \log \sigma^2+\log \frac{4}{\delta}}{n \lambda } \\ + \frac{\lambda}{2}\bar{V}_n^\alpha(\pi^{\textsc{gaus}}_{\mu, \sigma})\Big)\,, 
\end{align*}
where we used that $\sigma_0=1$ since our prior is $\mathbb{P}=\cN(\eta_0 \mu_0, I_{dK})$ for Gaussian policies. Moreover, we set $\lambda=\sqrt{2\frac{\frac{\left\|\mu-\mu_0\right\|^2}{2} - \frac{dK}{2} \log \sigma^2+\log \frac{4}{\delta}}{n \bar{V}_n^\alpha(\pi^{\textsc{gaus}}_{\mu, \sigma}) }}$.

\textbf{(Ours, Mixed-Logit)} Here we use the mixed-logit policies in \eqref{eq:app_logit_pac_bayes}. Thus we only replace the terms, $D_{\mathrm{KL}}(\mathbb{Q} \| \mathbb{P})$, with their closed-form bound in \cref{lemma:kl_gaussian}. This leads to the following objective. 
\begin{align*}
  \min_{\mu \in \real^{dK}, \sigma >0} \Big( \hat{R}_n^\alpha\left(\pi^{\textsc{mixL}}_{\mu, \sigma}\right) + \sqrt{ \frac{\frac{\left\|\mu-\mu_0\right\|^2}{2} - \frac{dK}{2} \log \sigma^2+\log \frac{4\sqrt{n}}{\delta}}{2n} } + B_n^\alpha(\pi^{\textsc{mixL}}_{\mu, \sigma})  +
\frac{\frac{\left\|\mu-\mu_0\right\|^2}{2} - \frac{dK}{2} \log \sigma^2+\log \frac{4}{\delta}}{n \lambda } \\
+ \frac{\lambda}{2}\bar{V}_n^\alpha(\pi^{\textsc{mixL}}_{\mu, \sigma})\Big)\,, 
\end{align*}
where we used that $\sigma_0=1$ since our prior is $\mathbb{P} = \cN(\eta_0 \mu_0, I_{dK}) \times {\rm G}(0, 1)^K$ for mixed-logit policies. Moreover, we set $\lambda=\sqrt{2\frac{\frac{\left\|\mu-\mu_0\right\|^2}{2} - \frac{dK}{2} \log \sigma^2+\log \frac{4}{\delta}}{n \bar{V}_n^\alpha(\pi^{\textsc{mixL}}_{\mu, \sigma}) }}$.

\subsubsection{\citet[Theorem 1]{london2019bayesian}}

\begin{proposition}\label{prop:london}
Let $\tau \in (0, 1)$, $n \geq 1$, $\delta \in(0,1)$ and let $\mathbb{P}$ be a fixed prior on $\mathcal{H}$, then with probability at least $1-\delta$ over draws $\cD_n \sim \mu_{\pi_0}^n$, the following holds simultaneously for all posteriors, $\mathbb{Q}$, on $\mathcal{H}$ that
\begin{align}
R\left(\pi_{\mathbb{Q}}\right) & \leq \hat{R}_n^\tau\left(\pi_{\mathbb{Q}}\right) +\sqrt{\frac{2\left(\hat{R}_n^\tau\left(\pi_{\mathbb{Q}}\right)+\frac{1}{\tau}\right)\left(D_{\mathrm{KL}}(\mathbb{Q} \| \mathbb{P})+\log \frac{n}{\delta}\right)}{\tau(n-1)}} +\frac{2\left(D_{\mathrm{KL}}(\mathbb{Q} \| \mathbb{P})+\log \frac{n}{\delta}\right)}{\tau(n-1)}\,.
\end{align}

\end{proposition}

\textbf{Baseline 1: (London et al., Gaussian)} Here we use the Gaussian policies in \eqref{eq:app_gaussian_pac_bayes}. Thus we only replace the terms, $D_{\mathrm{KL}}(\mathbb{Q} \| \mathbb{P})$, with their closed-form bound in \cref{lemma:kl_gaussian}. This leads to the following objective. 
\begin{align*}
  \min_{\mu \in \real^{dK}, \sigma >0} \Big( \hat{R}_n^\tau\left(\pi^{\textsc{gaus}}_{\mu, \sigma}\right) +\sqrt{\frac{2\left(\hat{R}_n^\tau\left(\pi^{\textsc{gaus}}_{\mu, \sigma}\right)+\frac{1}{\tau}\right)\left(\frac{\left\|\mu-\mu_0\right\|^2}{2} - \frac{dK}{2} \log \sigma^2+\log \frac{n}{\delta}\right)}{\tau(n-1)}}\\ +\frac{2\left(\frac{\left\|\mu-\mu_0\right\|^2}{2} - \frac{dK}{2} \log \sigma^2+\log \frac{n}{\delta}\right)}{\tau(n-1)}\Big)\,, 
\end{align*}
where we used that $\sigma_0=1$ since our prior is $\mathbb{P}=\cN(\eta_0 \mu_0, I_{dK})$ for Gaussian policies.

\textbf{Baseline 2: (London et al., Mixed-Logit)} Here we consider the mixed-logit policies in \eqref{eq:app_logit_pac_bayes}. Since the additional Gumbel noise does not affect the KL divergence (\cref{lemma:kl_gaussian}), we have the same objective as in the Gaussian case. That is
\begin{align*}
  \min_{\mu \in \real^{dK}, \sigma >0} \Big( \hat{R}_n^\tau\left(\pi^{\textsc{mixL}}_{\mu, \sigma}\right) +\sqrt{\frac{2\left(\hat{R}_n^\tau\left(\pi^{\textsc{mixL}}_{\mu, \sigma}\right)+\frac{1}{\tau}\right)\left(\frac{\left\|\mu-\mu_0\right\|^2}{2} - \frac{dK}{2} \log \sigma^2+\log \frac{n}{\delta}\right)}{\tau(n-1)}}\\+\frac{2\left(\frac{\left\|\mu-\mu_0\right\|^2}{2} - \frac{dK}{2} \log \sigma^2+\log \frac{n}{\delta}\right)}{\tau(n-1)}\Big)\,,
\end{align*}
where we used that $\sigma_0=1$ since our prior is $\mathbb{P} = \cN(\eta_0 \mu_0, I_{dK}) \times {\rm G}(0, 1)^K$ for mixed-logit policies.

\subsubsection{\citet[Proposition 1]{sakhi2022pac}}

\begin{proposition}\label{prop:sakhi1}
Let $\tau \in (0, 1)$, $n \geq 1$, $\delta \in(0,1)$ and let $\mathbb{P}$ be a fixed prior on $\mathcal{H}$, then with probability at least $1-\delta$ over draws $\cD_n \sim \mu_{\pi_0}^n$, the following holds simultaneously for all posteriors, $\mathbb{Q}$, on $\mathcal{H}$ that
\begin{align}
R\left(\pi_{\mathbb{Q}}\right) & \leq \min_{\lambda >0} \frac{1}{\tau\left(e^\lambda-1\right)}\Big(1-e^{-\tau \lambda \hat{R}_n^\tau\left(\pi_{\mathbb{Q}}\right)+\frac{D_{\mathrm{KL}}(\mathbb{Q} \| \mathbb{P})+\log \frac{2 \sqrt{n}}{\delta}}{n}}\Big)\,.
\end{align}
\end{proposition}

\textbf{Baseline 3: (Sakhi et al. 1, Gaussian)} Here we use the Gaussian policies in \eqref{eq:app_gaussian_pac_bayes}.
\begin{align}
\min_{\mu \in \real^{dK}, \sigma >0, \lambda>0} \Big(\frac{1}{\tau\left(e^\lambda-1\right)}\Big(1-e^{-\tau \lambda \hat{R}_n^\tau\left(\pi^{\textsc{gaus}}_{\mu, \sigma}\right)+\frac{\frac{\left\|\mu-\mu_0\right\|^2}{2} - \frac{dK}{2} \log \sigma^2+\log \frac{2 \sqrt{n}}{\delta}}{n}}\Big)\Big)\,,
\end{align}
where we used that $\sigma_0=1$ since our prior is $\mathbb{P} = \cN(\eta_0 \mu_0, I_{dK})$ for Gaussian policies.

\textbf{Baseline 4: (Sakhi et al. 1, Mixed-Logit)} Here we consider the mixed-logit policies in \eqref{eq:app_logit_pac_bayes}.
\begin{align}
\min_{\mu \in \real^{dK}, \sigma >0, \lambda>0} \Big(\frac{1}{\tau\left(e^\lambda-1\right)}\Big(1-e^{-\tau \lambda \hat{R}_n^\tau\left(\pi^{\textsc{mixL}}_{\mu, \sigma}\right)+\frac{\frac{\left\|\mu-\mu_0\right\|^2}{2} - \frac{dK}{2} \log \sigma^2+\log \frac{2 \sqrt{n}}{\delta}}{n}}\Big)\Big)\,.
\end{align}
where we used that $\sigma_0=1$ since our prior is $\mathbb{P} = \cN(\eta_0 \mu_0, I_{dK}) \times {\rm G}(0, 1)^K$ for mixed-logit policies. 

\subsubsection{\citet[Proposition 3]{sakhi2022pac}}
\begin{proposition}\label{prop:sakhi2}
Let $\tau \in (0, 1)$, $n \geq 1$, $\delta \in(0,1)$, let $\mathbb{P}$ be a fixed prior on $\mathcal{H}$, and let $\Lambda = \set{\lambda_i}_{i \in [n_\lambda]}$ a set of $n_\lambda$ positive scalars. Then with probability at least $1-\delta$ over draws $\cD_n \sim \mu_{\pi_0}^n$, the following holds simultaneously for all posteriors, $\mathbb{Q}$, on $\mathcal{H}$ and any $\lambda_i \in \Lambda$,
\begin{align}
R\left(\pi_{\mathbb{Q}}\right) & \leq \hat{R}_n^\tau\left(\pi_{\mathbb{Q}}\right) +\sqrt{\frac{D_{\mathrm{KL}}(\mathbb{Q} \| \mathbb{P})+\log \frac{4 \sqrt{n}}{\delta}}{2 n}}+\frac{D_{\mathrm{KL}}(\mathbb{Q} \| \mathbb{P})+\log \frac{2 n_\lambda}{\delta}}{\lambda} +\frac{\lambda}{n} g\left(\frac{\lambda}{ \tau n}\right) \mathcal{V}_{n}^\tau\left(\pi_{\mathbb{Q}}\right)\,,
\end{align}
where $g: u \rightarrow \frac{\exp (u)-1-u}{u^2}$ and $\mathcal{V}_{n}^\tau(\pi_{\mathbb{Q}})=\frac{1}{n} \sum_{i=1}^{n} \E{a \sim \pi_{\mathbb{Q}}\left(\cdot | x_i\right)}{\frac{\pi_0\left(a | x_i\right)}{\max \left(\tau, \pi_0\left(a | x_i\right)\right)^2}},.$
\end{proposition}

\textbf{Baseline 5: (Sakhi et al. 2, Gaussian)} Here we consider the Gaussian policies in \eqref{eq:app_gaussian_pac_bayes}.

\begin{align}
\min_{\mu \in \real^{dK}, \sigma >0, \lambda \in \Lambda}\Big(\hat{R}_n^\tau\left(\pi^{\textsc{gaus}}_{\mu, \sigma}\right) +\sqrt{\frac{\frac{\left\|\mu-\mu_0\right\|^2}{2} - \frac{dK}{2} \log \sigma^2+\log \frac{4 \sqrt{n}}{\delta}}{2 n}}+\frac{\frac{\left\|\mu-\mu_0\right\|^2}{2} - \frac{dK}{2} \log \sigma^2+\log \frac{2 n_\lambda}{\delta}}{\lambda} \nonumber\\ +\frac{\lambda}{n} g\left(\frac{\lambda}{ \tau n}\right) \mathcal{V}_{n}^\tau\left(\pi^{\textsc{gaus}}_{\mu, \sigma}\right)\Big)\,,
\end{align}
where we used that $\sigma_0=1$ since our prior is $\mathbb{P} = \cN(\eta_0 \mu_0, I_{dK})$ for Gaussian policies. 

\textbf{Baseline 6: (Sakhi et al. 2, Mixed-Logit)} Here we consider the mixed-logit policies in \eqref{eq:app_logit_pac_bayes}.

\begin{align}
\min_{\mu \in \real^{dK}, \sigma >0, \lambda \in \Lambda}\Big(\hat{R}_n^\tau\left(\pi^{\textsc{mixL}}_{\mu, \sigma}\right) +\sqrt{\frac{\frac{\left\|\mu-\mu_0\right\|^2}{2} - \frac{dK}{2} \log \sigma^2+\log \frac{4 \sqrt{n}}{\delta}}{2 n}}+\frac{\frac{\left\|\mu-\mu_0\right\|^2}{2} - \frac{dK}{2} \log \sigma^2+\log \frac{2 n_\lambda}{\delta}}{\lambda} \nonumber \\ +\frac{\lambda}{n} g\left(\frac{\lambda}{ \tau n}\right) \mathcal{V}_{n}^\tau\left(\pi^{\textsc{mixL}}_{\mu, \sigma}\right)\Big)\,,
\end{align}
where we used that $\sigma_0=1$ since our prior is $\mathbb{P} = \cN(\eta_0 \mu_0, I_{dK}) \times {\rm G}(0, 1)^K$ for mixed-logit policies.

\subsection{Additional Results and Discussion} \label{app:add_results}

In \cref{fig:app_main_exp_results}, we report the reward of the learned policy using one of the considered methods. We make the following observations:
\begin{itemize}[topsep=0pt,itemsep=0pt]
\item \textbf{Choice of $\tau$ and $\alpha$:} in \cref{fig:app_main_exp_results}, we set $\tau= 1/\sqrt[\leftroot{-2}\uproot{2}4]{n} \approx 0.06$ and $\alpha = 1-1/\sqrt[\leftroot{-2}\uproot{2}4]{n} \approx 0.94$ so that when $n$ is large enough, both $\hat{R}_n^\tau(\pi)$ and $ \hat{R}_n^\alpha(\pi)$ approach $\hat{R}_n^{\textsc{ips}}(\pi)$ \citep{ionides2008truncated}. This is because standard IPS should be preferred when $n \rightarrow \infty$. For completeness, we also show in \cref{fig:app_varying_params} that the choice of $\alpha$ and $\tau$ does not affect the conclusions that we make here. We also include in \cref{fig:app_varying_params} the results with an adaptive and data-dependent $\alpha$ obtained using \eqref{eq:data_dependent_alpha} in \cref{subsec:data_dep_alpha}. The results in \cref{fig:app_varying_params} will be discussed in detail after we finish analyzing the results in \cref{fig:app_main_exp_results}.
\item \textbf{Overall performance:} our method outperforms the baselines for any class of learning policies (Gaussian or mixed-logit) and any choice of logging policies. The only exception is when the logging policy is uniform.
\item \textbf{Effect of the class of learning policies:} the class of policies, Gaussian or mixed-logit, affects the performance of all the baselines. In general, Gaussian policies behave better than mixed-logit policies. However, this is less significant for our method; the performance of both Gaussian and mixed-logit policies are comparable, and in both cases, our method outperforms the baselines with Gaussian policies. Therefore, in general, Gaussian policies should be preferred over mixed-logit policies. But in case engineering constraints impose the choice of mixed-logit or softmax policies, then the performance of our method is robust to this choice. 
\item \textbf{Effect of the logging policy:} our method reaches the maximum reward even when the logging policy is not performing well. In contrast, the baselines only reach their best reward when the logging policy is already well-performing ($\eta_0 \approx 1$), in which case minor to no improvements are made. Note that the baselines have a better reward than ours when the logging policy is uniform. But our method has better reward when the logging policy is not uniform, that is when $\eta_0>0$. This is more common in practice since the logging policy is deployed in production and thus it is expected to perform better than the uniform policy.
\end{itemize}

In \cref{fig:app_varying_params}, we compare our method to \textbf{(Sakhi et al. 2)} with Gaussian policies since this was the best-performing baseline in our experiments in \cref{fig:app_main_exp_results}. Note that we did not include \texttt{CIFAR100} in \cref{fig:app_main_exp_results} as it was computationally heavy to run these experiments with varying $\eta_0$, $\alpha$ and $\tau$ for a very high-dimensional dataset such as \texttt{CIFAR100}. We consider $20$ varying values of $\tau$ and $\alpha$ evenly spaced in $(0, 1)$. We also include the results using the adaptive tuning procedure of $\alpha$ described in \cref{subsec:data_dep_alpha} (green curve). We make the following observations:

\begin{itemize}[topsep=0pt,itemsep=0pt]
\item \textbf{Adaptive and data-dependent $\alpha$:} This procedure is reliable since the performance with an adaptive $\alpha$ (green curve) is comparable with the best possible choice of $\alpha$. This is consistent for the three datasets.
\item \textbf{Effect of the choice $\alpha$:} as we observed before, the only case where the choice of $\alpha$ may lead to bad-performing policies is when the logging policy is uniform. When the logging policy is not uniform, our method outperforms the best baseline with the best $\tau$ for a wide range of values of $\alpha$. Also, note that there is no very bad choice of $\alpha$, in contrast with $\tau \approx 0$ that led to a very bad performing policy that slightly improved upon the logging policy. This attests to the robustness of our method to the choice of $\alpha$. Moreover, our bound regularizes better $\alpha$; it contains a bias-variance trade-off term for $\alpha$. Also, the bound of \textbf{(Sakhi et al. 2)} has a $1/\tau$ making it vacuous for small values of $\tau$.
\item \textbf{Best choice of $\alpha$:} To see the effect of $\alpha$ for varying problems, we consider the following experiment. We split the logging policies into two groups. The first is \emph{modest logging} which corresponds to logging policies whose $\eta_0$ is between $0$ and $0.5$. This includes uniform logging policies and other average-performing logging policies. The second is \emph{good logging} which corresponds to logging policies whose $\eta_0$ is between $0.5$ and $1$. After that, for each $\alpha$, we compute the average reward of the learned policy across either the group of modest or good logging policies. For each dataset, this leads to the two red and green curves in the second row of \cref{fig:app_varying_params}. Overall, we observe that $\alpha \approx 0.7$ leads to the best performance for the \emph{modest logging} group. Thus when the performance of the logging policy is average, regularizing the importance weights can be critical. In contrast, when the performance of the logging policy is already good, regularization is less needed and we can set $\alpha \approx 1$. Fortunately, one of the main strengths of this work is that our bound also holds for standard IPS recovered for $\alpha=1$. The bounds in all prior works cannot provide good performance for standard IPS due to their dependency on $1/\tau$.
\end{itemize}

\begin{figure}[H]
  \centering  \includegraphics[width=\linewidth]{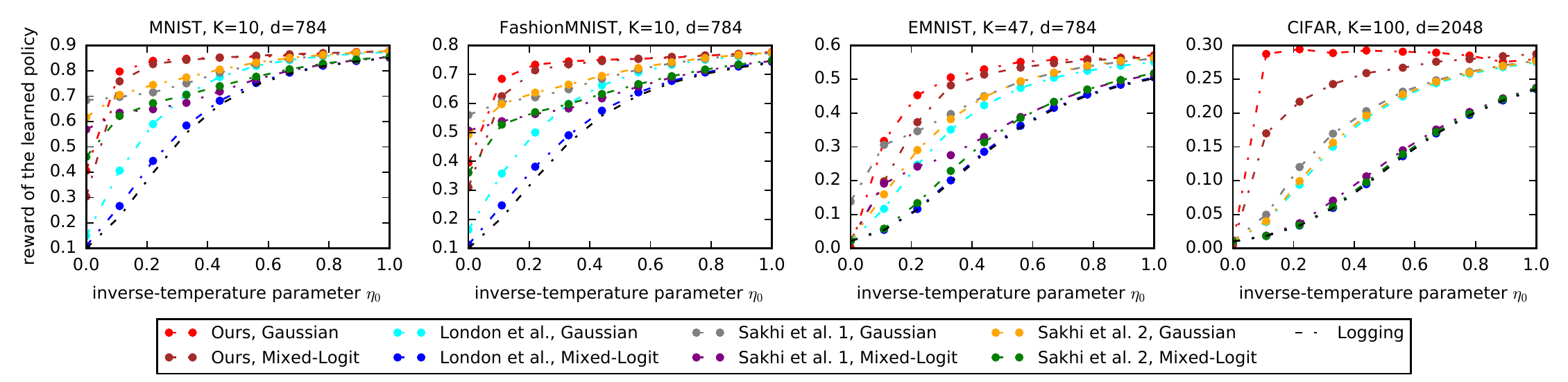}
  \caption{The reward of the learned policy for four datasets with varying quality of the logging policy $\eta_0 \in [0, 1]$.} 
  \label{fig:app_main_exp_results}
\end{figure}

\begin{figure}
  \centering  \includegraphics[width=0.7\linewidth]{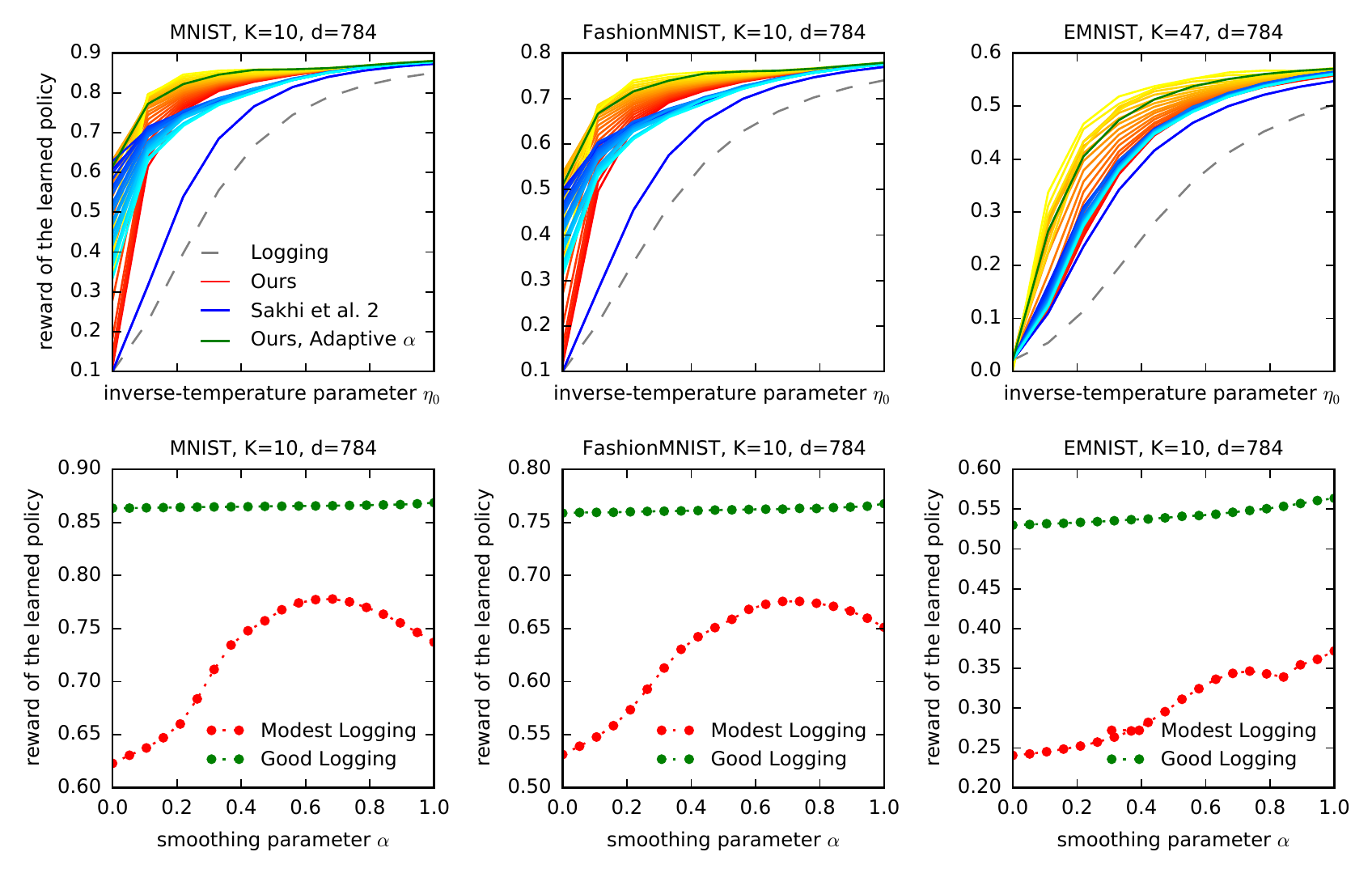}
  \caption{In the first row, we report the reward of the learned policy with 20 evenly space values of $\tau \in (0, 1)$ and $\alpha \in (0, 1)$ and varying $\eta_0 \in [0, 1]$, and for an adaptive and data-dependent $\alpha$ obtained using \eqref{eq:data_dependent_alpha} in \cref{subsec:data_dep_alpha}. The blue-to-cyan colors correspond to different values of $\tau$. The lighter the color, the higher the value of $\tau$. For instance, the cyan lines correspond to high values of $\tau$ while the blue ones correspond to very small values of $\tau$. Similarly, the red-to-yellow colors correspond to different values $\alpha$. The lighter the color, the higher the value of $\alpha$. For instance, the yellow lines correspond to high values of $\alpha$ while the red ones correspond to very small values of $\alpha$. Finally, the green curve corresponds to the reward of the learned policy using an adaptive and data-dependent $\alpha$ described in \eqref{eq:data_dependent_alpha} (\cref{subsec:data_dep_alpha}). In the second row, we report the \emph{average} reward of the learned policies using our method across the modest logging group ($\eta_0 \in [0, 0.5]$ in red) and the good logging group  ($\eta_0 \in [0.5, 1]$ in green).}
  \label{fig:app_varying_params}
\end{figure}

\subsection{Learning Principles} \label{app:add_discussion}

Here we compare our bound in \cref{thm:main_result} and our learning principle in \eqref{eq:learning_principle} to the one in \citet{london2019bayesian}. We do not include the learning principle in \citet{swaminathan2015batch} since the one in \citet{london2019bayesian} enjoys similar performance and is far more scalable. The learning principle of \citet{london2019bayesian} is defined as 
\begin{align}\label{eq:lp_london}
\min_{\mu} \hat{R}^\tau_n(\pi_\mu) + \lambda \norm{\mu - \mu_0}^2\,.
\end{align}
where $\lambda$ is a tunable hyper-parameters, $\pi_\mu$ is the softmax policy defined in \eqref{eq:app_softmax_pac_bayes} and $\mu \in \real^{dK}$ is its parameter vector. This learning principle is referred to as \textbf{(London et al., LP)}. In contrast, our learning principle is defined as 
\begin{align}\label{eq:our_learning_principle}
\hat{R}^\alpha_n(\pi_{\mu}) + \lambda_1 \norm{\mu - \mu_0}^2  + \lambda_2  \bar{V}_n^\alpha(\pi_{\mu}) + \lambda_3 B_n^\alpha(\pi_{\mu})\,,
\end{align}
where $\lambda_1, \lambda_2$ and $\lambda_3$ are tunable hyper-parameters and $\pi_{\mu}$ is the Gaussian policy in \eqref{eq:gaussian_pac_bayes} with a fixed $\sigma=1$. Our learning principle is referred to as \textbf{(Ours, LP)}. Finally, our bound in \cref{thm:main_result} with Gaussian policies is referred to as \textbf{(Ours, Bound)}. Similarly to the previous experiments, we set $\tau= 1/\sqrt[\leftroot{-2}\uproot{2}4]{n} \approx 0.06$ and $\alpha = 1-1/\sqrt[\leftroot{-2}\uproot{2}4]{n} \approx 0.94$ so that when $n$ is large enough, both $\hat{R}_n^\tau(\pi)$ and $ \hat{R}_n^\alpha(\pi)$ approach $\hat{R}_n^{\textsc{ips}}(\pi)$ \citep{ionides2008truncated}. For the learning principles, we tried multiple values of hyper-parameters $\lambda, \lambda_1, \lambda_2$ and $\lambda_3$, all between $10^{-5}$ and $ 10^{-1}$. For instance, we found that the best hyper-parameter for \citet{london2019bayesian} is $\lambda=10^{-5}$ which matches the value they found in their \texttt{FashionMNIST} experiments. For our learning principle, the best hyper-parameters were $\lambda_1=10^{-5}, \lambda_2=10^{-5}$ and $\lambda_3=10^{-5}$. In contrast, our bound does not require hyper-parameter tuning. We report in \cref{fig:app_lp_compare} the reward of the learned policy on the \texttt{FashionMNIST} for all these methods with varying values of hyper-parameters. To reduce clutter, we only report the reward for good choices of hyper-parameters $\lambda, \lambda_1, \lambda_2$ and $\lambda_3$. We observe that for a wide range of hyper-parameters, our learning principle outperforms the one in \citet{london2019bayesian}. However, both learning principles are sensitive to the choice of hyper-parameters. In contrast, our bound does not require the tuning of any additional hyper-parameter and it achieves the best performance except for the uniform logging policy. In addition to being more theoretically grounded, this approach also enjoys favorable empirical performance without additional hyper-parameter tuning, an important practical consideration.

\begin{figure}[H]
  \centering  \includegraphics[width=0.35\linewidth]{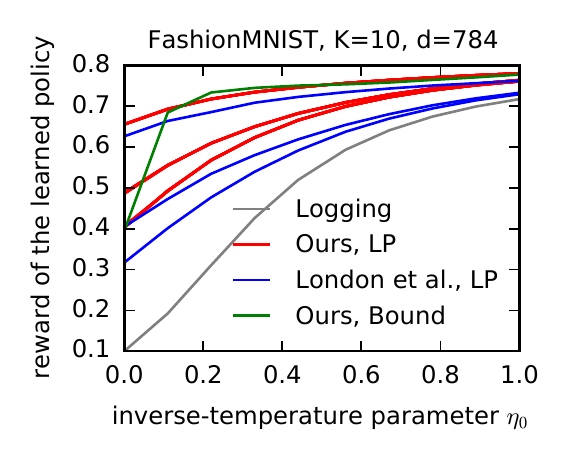}
  \caption{The reward of the learned policy using either our bound in \cref{thm:main_result} (referred to as \textbf{(Ours, Bound)} in green), our learning principle in \eqref{eq:learning_principle} (referred to as \textbf{(Ours, LP)} in red for multiple values of hyper-parameters) or the learning principle in \citet{london2019bayesian} (referred to as \textbf{(London et al., LP)} in blue) for multiple values of hyper-parameters).
} 
  \label{fig:app_lp_compare}
\end{figure}

\subsection{Other Importance Weight Corrections} \label{app:other_corrections}

\citet{su2020doubly,metelli2021subgaussian} also proposed corrections that are different from hard clipping (a detailed comparison is given in \cref{sec:ope}). However, they were not included in our main experiments since they do not provide generalization guarantees; they focus on OPE and only propose a heuristic for OPL in their Appendix B.2 and Section 6.1.2, respectively. Those heuristics are not based on theory, in contrast with ours which is directly derived from our generalization bound. However, for completeness, we also compare our regularization of importance weights to theirs. To make such a comparison, we use the hyper-parameters and tuning procedures provided in Section 6 and Appendix B.2 for \citet{metelli2021subgaussian} and Sections 5 and 6.1.2 for \citet{su2020doubly}. Overall, we observe in \cref{fig:app_other_corrections} that our method outperforms these baselines in OPL and the gap is more significant when the logging policy is not performing well.

\begin{figure}[H]
  \centering  \includegraphics[width=0.8\linewidth]{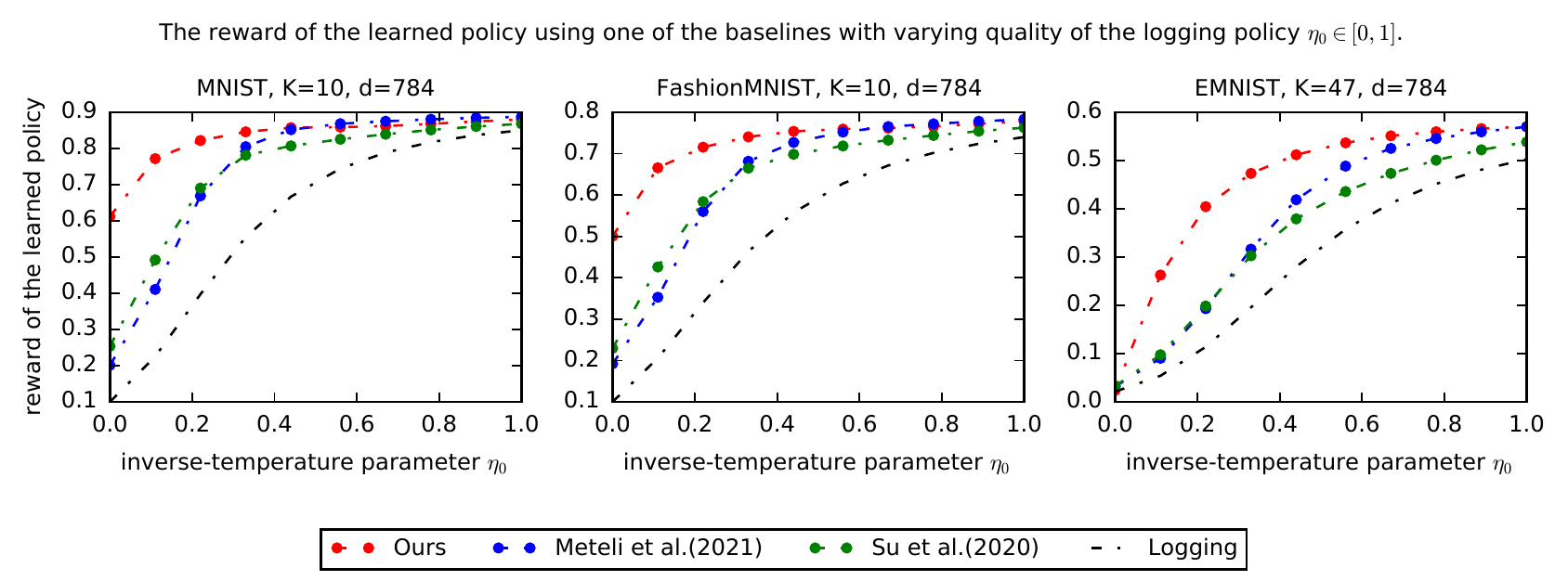}
  \caption{The reward of the learned policy with varying quality of the logging policy $\eta_0 \in [0, 1]$ using either our regularization ($\alpha$-\texttt{IPS}) or the ones in \citet{su2020doubly,metelli2021subgaussian}.} 
  \label{fig:app_other_corrections}
\end{figure}

\end{document}